\documentclass[final]{siamonline171218}
\usepackage{booktabs}
\usepackage{wrapfig}
\usepackage{enumitem}
\usepackage{colortbl}
\usepackage{lipsum}
\usepackage{amsmath}
\usepackage{amssymb}
\usepackage{amsfonts}
\usepackage{graphicx}
\usepackage{epstopdf}
\usepackage{algorithm}
\usepackage{algorithmicx}
\usepackage{algpseudocode}
\usepackage{amsopn}
\usepackage{caption}
\usepackage{subcaption}
\usepackage{pgfplots}
\usepackage{pgfplotstable}
\pgfplotsset{compat=newest}
\usetikzlibrary{spy}
\usepackage{multirow}

\newcommand{\normm}[1]{ \left\|#1\right\| }
\newcommand{\bbR}{\mathbb{R}}

\newcommand{\bA}{\mathbf{A}}

\newcommand{\bu}{\mathbf{u}}
\newcommand{\bv}{\mathbf{v}}
\newcommand{\bx}{\mathbf{x}}
\newcommand{\by}{\mathbf{y}}

\newtheorem{remark}{Remark}

\newtheorem{counterexample}{Counterexample}
\makeatletter
\def\BState{\State\hskip-\ALG@thistlm}
\makeatother
\headers{}{Shuchang Zhang, Hui Zhang, Hongxia Wang}
\title{HPPP: Halpern-type Preconditioned Proximal Point Algorithms and Applications to Image Restoration\thanks{\textbf{\color{header1}Funding:} This work was supported by the National Key Research and Development Program (2020YFA0713504) and the National Natural Science Foundation of China (grants 12471300 and 12471401).}}
\author{Shuchang  Zhang\thanks{Department of Mathematics, National University of Defense Technology, Changsha, 410073, China (\email{zhangshuchang19@nudt.edu.cn}).}
\and Hui Zhang\thanks{Department of Mathematics, National University of Defense Technology, Changsha, 410073, China (\email{h.zhang1984@163.com}).}
\and Hongxia Wang\thanks{Corresponding author. Department of Mathematics, National University of Defense Technology, Changsha, 410073, China (\email{wanghongxia@nudt.edu.cn}).}}
\ifpdf
\hypersetup{
  pdftitle={HPPP: Halpern-type Preconditioned Proximal Point Algorithms and Applications to Image Restoration},
  pdfauthor={Shuchang Zhang, Hui Zhang, Hongxia Wang}
}
\fi

\begin{document}

\maketitle
\begin{abstract}
	Recently, the degenerate preconditioned proximal point (PPP) method provided a unified and flexible framework for designing and analyzing operator-splitting algorithms, such as the Douglas-Rachford (DR) splitting. However, the degenerate PPP method exhibits weak convergence in the infinite-dimensional Hilbert space and lacks accelerated variants. To address these issues, we propose a Halpern-type PPP (HPPP) algorithm, which leverages the strong convergence and acceleration properties of Halpern's iteration method. Moreover, we propose a novel algorithm for image restoration by combining HPPP with denoiser priors, such as the Plug-and-Play (PnP) prior, which can be viewed as an accelerated PnP method. Finally, numerical experiments, including several toy examples and image restoration, validate the effectiveness of our proposed algorithms.
\end{abstract}

\section{Introduction}
Image restoration (IR) problems, including image deblurring, super-resolution, and inpainting, can be formulated as the following optimization problem~\cite{Chambolle2011,He2012}:
\begin{equation}\label{problem: optimization problem for image restoration}
\min_{\bx\in\mathcal{X}} \lambda f(\bx)+g(\mathbf{K}\bx),
\end{equation}
where $ f: \mathcal{X} \to \mathbb{R} \cup \{+\infty\}$ and $ g: \mathcal{Y} \to \mathbb{R} \cup \{+\infty\}$ are convex, lower semicontinuous functions, $ \mathbf{K}: \mathcal{X} \to \mathcal{Y}$ is a bounded linear operator, and $ \lambda > 0 $ is a balance parameter. Both $ \mathcal{X} $ and $ \mathcal{Y} $ are real Hilbert spaces. The first term $ f $ represents the data fidelity, while the second term $ g $ serves as a regularization, such as TV (total variation)~\cite{RUDIN1992259}, to mitigate the ill-posedness of IR problems. 

By the first-order optimality condition, the convex optimization problem~\eqref{problem: optimization problem for image restoration} is equivalent to the following inclusion problem:
\begin{equation}\label{original inclusion problem}
\text{find } \bx \in \mathcal{X} \text{ such that } \mathbf{0} \in \lambda \partial f(\bx) + \mathbf{K}^* \partial g(\mathbf{K}\bx),
\end{equation}
where $ \partial f (\bx) $ and $ \partial g (\bx) $ are the subdifferentials of $ f $ and $ g $ at $ \bx $, respectively~\cite[Chapter 3]{beck2017first}. Following~\cite{He2012,Bredies2022}, by introducing an auxiliary variable $\by \in \partial g(\mathbf{K}\bx) $, we reformulate~\eqref{original inclusion problem} as
\begin{equation}\label{problem: inclusion}
\text{find } \bu \in \mathcal{H} \text{ such that } \mathbf{0} \in \mathcal{A} \bu,
\end{equation}
where $\mathcal{A} = \begin{pmatrix}
\lambda \partial f & \mathbf{K}^* \\
- \mathbf{K} & (\partial g)^{-1}
\end{pmatrix}$, $\bu = (\bx, \by)$, and $\mathcal{H} = \mathcal{X} \times \mathcal{Y}$. The problem~\eqref{problem: inclusion} is common in modern optimization and variational analysis~\cite{BCombettes,Bauschke2023}.
When $\mathcal{A}$ is maximal monotone, the resolvent $J_{\mathcal{A}} = (I + \mathcal{A})^{-1}$ is nonexpansive with a full domain, as established by the Minty surjectivity theorem~\cite{Minty1962MonotoneO}. The proximal point iteration $\mathbf{u}^{k+1} = (I + \mathcal{A})^{-1}\mathbf{u}^k$ is used to solve~\eqref{problem: inclusion} and converges weakly~\cite{Rockafellar1976}. However, it is difficult to compute the operator $(I + \mathcal{A})^{-1}$, and hence splitting methods have been developed to address this issue. The well-known Douglas-Rachford splitting (DRS)~\cite{Eckstein1992} decomposes $\mathcal{A}$ into the sum of two maximal monotone operators $\mathcal{A}_1$ and $\mathcal{A}_2$, for which $J_{\mathcal{A}_1}$ and $J_{\mathcal{A}_2}$ are easier to obtain. Another way to solve~\eqref{problem: optimization problem for image restoration} is to transform it into a saddle-point problem \cite{Chambolle2011,Pock2011}, i.e.,
\begin{equation}\label{problem: saddle-point problem}
\min_{\bx \in \mathcal{X}} \max_{\by \in \mathcal{Y}} \left\langle \mathbf{K}\bx, \by \right\rangle + \lambda f(\bx) - g^*(\by),
\end{equation}
where $g^*: \mathcal{Y} \to \mathbb{R} \cup \{+\infty\}$ is the conjugate of $g$. In this direction, many primal-dual methods, including the well-known Chambolle-Pock (CP) primal-dual method~\cite{Chambolle2011,Pock2011} and the primal-dual hybrid gradient (PDHG) method~\cite{He2012}, are designed and studied extensively. In particular, He and Yuan~\cite{He2012} analyzed the PDHG method from a PPP standpoint with a positive definite preconditioner $\mathcal{M}: \mathcal{H}\to\mathcal{H}$. Recently, Bredies et al.~\cite{Bredies2015, Bredies2017, Bredies2022} developed a unified degenerate PPP algorithmic framework with a positive semidefinite $\mathcal{M}$. By choosing appropriate preconditioners, this framework could cover the DR and CP algorithms~\cite{Bredies2022}.

It is crucial to study algorithms that exhibit both strong convergence and acceleration~\cite{2002Iterative,BOT2021,Izuchukwu2023,Bauschke2023,Sun2024,chen2024hprlpimplementationhprmethod,lu2024restartedhalpernpdhglinear}, such as strongly convergent proximal point methods~\cite{2002Iterative} and forward-reflected-backward splitting algorithms~\cite{Izuchukwu2023}. For example, Bauschke et al. demonstrated strong convergence of degenerate PPP under the special case where $\mathcal{A}$ is linear~\cite{Bauschke2023}. In such cases, the weak limit of the PPP sequence corresponds to the $\mathcal{M}$-projection of the initial point. Moreover, Sun et al.~\cite{Sun2024} proposed an accelerated preconditioned alternating direction method of multipliers (pADMM) by leveraging the degenerate PPP method~\cite{Bredies2022} and the fast Krasnosel'ski\u{ı}-Mann (KM) iteration~\cite{fastKM2023}. Similarly, Chen et al. introduced an accelerated HPR-LP solver, which implements a Halpern Peaceman-Rachford method enhanced with semiproximal terms for efficiently solving linear programming (LP) problems~\cite{chen2024hprlpimplementationhprmethod}, and the restarted Halpern PDHG (rHPDHG) achieves an accelerated linear convergence rate~\cite{lu2024restartedhalpernpdhglinear}. A shortcoming of forward-backward and DRS algorithms employing KM iteration is that their iterates converge only in the weak topology~\cite{BOT2021}. Consequently, the degenerate PPP method based on KM iteration typically exhibits weak convergence in Hilbert spaces. Meanwhile, it is unclear whether the degenerate PPP method can be accelerated via simple modification.

The classic Halpern iteration~\cite{halpern1967fixed} offers the advantage of strong convergence over the KM iteration in infinite-dimensional Hilbert spaces, with the limit identified as the metric projection of the anchor onto the fixed point set~\cite{he2024convergence}. Then the Halpern iteration is also known as an implicitly regularized method~\cite{Diakonikolas2020}. Due to this implicit regularity, the degenerate PPP method incorporating Halpern iteration can obtain a unique solution, yielding stable recovery results for IR problems. Achieving stable reconstruction for ill-posed inverse problems is important~\cite{engl1996regularization}. Moreover, beyond the implicit regularity, Halpern iteration can also accelerate convergence rate in terms of operator residual norm~\cite{TranDinh2024}, which is widely utilized in machine learning~\cite{Diakonikolas2020,Park2022,he2024convergence}. Notably, the work~\cite{Park2022} has demonstrated that PDHG with restarted Halpern iteration achieves a faster convergence rate for function values in CT image reconstruction. Based on the Halpern iteration, the following Halpern-type preconditioned proximal point algorithm (called HPPP)~\eqref{eq: HPPP} is thus proposed to overcome the two limitations of the degenerate PPP method:
\begin{equation}\label{eq: HPPP}
\bu^{k+1}=\mu_{k+1}\mathbf{a}+(1-\mu_{k+1})\mathcal{T}\bu^k,
\end{equation}
where $\mathcal{T}= (\mathcal{A}+\mathcal{M})^{-1}\mathcal{M}$, and $\mathcal{A},\mathcal{M}$ are detailed in subsection~\ref{subsec:GraREDHP3}, and $ \mathbf{a}, \bu^0 \in \mathcal{H}$ are the anchor point and the initial point, respectively, and $ \{\mu_k \}_{k\in \mathbb{N}}$ is a sequence in $[0,1]$ such that $ \sum_{k\in \mathbb{N}} \mu_k = +\infty, \lim_{k\to \infty}\mu_k =  0$. 

PnP (Plug-and-Play) methods that combine splitting algorithms with denoiser priors have been widely applied in practical problems~\cite{Venkatakrishnan2013,PnP2015,PnP2020forMRI,PnP-CI2023,There-Operator-Splitting2024} and have achieved state-of-the-art performance in inverse imaging tasks~\cite{zhang2021plug,hurault2022proximal,Tan2024}. Buzzard et al. provided a consensus equilibrium interpretation on denoiser priors~\cite{Buzzard2018}. Romano , Elad, and Milanfar introduced RED (regularization by denoising)~\cite{Romano2017}, whose gradient exactly corresponds to the denoising residual, thereby yielding a clear objective function~\cite{cohen2021regularization} that can be exploited in first-order optimization. Powerful denoisers such as denoising convolutional neural network (DnCNN)~\cite{Zhang2017} typically do not meet the conditions of RED~\cite{REDClarifications2019}. To address this limitation, Reehorst and Schniter introduced the score-matching by denoising (SMD) perspective to interpret RED~\cite{REDClarifications2019}. Based on the fixed-point projection, Cohen et al. proposed the RED-PRO model~\cite{cohen2021regularization} as a bridge between RED and PnP, while the hybrid steepest descent (HSD) method~\cite{yamada2001hybrid,Yamada2005} was employed to solve the resulting model. Meanwhile, Ryu et al. proved the convergence of PnP-FBS (forward-backward splitting) and PnP-ADMM using the Banach contraction principle under the assumption that the data term $f$ is strongly convex and that the residual is nonexpansive~\cite{ryu2019plug}. However, enforcing strong convexity on $f$ precludes many IR tasks~\cite{hurault2022proximal}. Furthermore, Cohen et al.  parameterized denoisers via the gradients of smooth potential functions that satisfy a symmetric Jacobian property~\cite{NEURIPS2021_97108695}, and Hurault, Leclaire, and Papadakis developed the gradient step (GS) denoiser~\cite{hurault2022gradient,hurault2022proximal}, which can be interpreted as a proximal operator of an implicit regularizer~\cite{hurault2022proximal}. Although the convergence of PnP-ADMM is commonly analyzed from the perspective of Douglas-Rachford splitting (DRS) based on the equivalence between ADMM and DRS~\cite{ryu2019plug,hurault2022proximal}, theoretically achieving fast fixed-point residual decay in PnP-DRS remains unknown.

Based on HPPP and PnP priors, we propose the gradient regularization by denoising via HPPP called GraRED-HP$^3$ (Algorithm~\ref{alg: GraRED-HP3})~\footnote{Since the RED gradient exactly equals the residual, we adopt the \texttt{GraRED} notation—derived from Moreau decomposition—to represent the proximal operator of the implicit regularizer associated with the residual.}. Our main contributions are as follows:
\begin{enumerate}
  \item \textbf{Theoretical contributions.} The sequence $\{\mathbf{u}^k\}_{k\in \mathbb{N}}$ generated by HPPP with the positive semidefinite preconditioner $\mathcal{M}$ converges strongly to a unique solution  $ \mathbf{u}^* = \arg\min_{\mathbf{u}\in \mathrm{Fix}(\mathcal{T})}\|\mathbf{u}-\mathbf{a} \|_\mathcal{M}^2$, as stated in Theorem~\ref{thm:bigthm} and Proposition~\ref{cor:bigthm cor}, while the original PPP method can only converge weakly to some uncertain solution. Moreover, compared with the asymptotic regularity result of PPP~\cite[Lemma~2.8]{Bredies2022} (i.e., $\lim_{k\to\infty}\|\mathcal{T}\mathbf{u}^k-\mathbf{u}^k\|_\mathcal{M} = 0$), we establish a convergence rate of $\mathcal{O}(1/k)$ for both $\|\mathcal{T}\mathbf{u}^k-\mathbf{u}^k\|$ and $\|\mathcal{T}\mathbf{u}^k-\mathbf{u}^k\|_\mathcal{M}$ (see Propositions~\ref{cor: convergence rate of HPPP} and~\ref{cor: accelerated convergence rate of HPPP}).
  
  \item \textbf{Algorithmic development.} We integrate HPPP with denoiser priors to propose the GraRED-HP$^3$ algorithm for IR problems and shed new theoretical insights on denoiser priors from a PPP standpoint. Noting that the PnP-ADMM algorithm can be reformulated in an equivalent DRS form~\cite{ryu2019plug,hurault2022proximal}, and given that the worst-case convergence rate for DRS is $\mathcal{O}(1/\sqrt{k})$~\cite{He2015}, GraRED-HP$^3$ improved a sublinear convergence rate (as per Proposition~\ref{cor: accelerated convergence rate of HPPP}) and thereby serves as an accelerated PnP method.
  
  \item \textbf{Experimental validation.} We numerically validate the advantages of HPPP, including its implicit regularity, acceleration, and efficiency, with several toy examples. In addition, we demonstrate the state-of-the-art performance of GraRED-HP$^3$ and its restarted variant with advanced denoisers in IR tasks.
\end{enumerate}

The rest of this paper is organized as follows. In section~\ref{section: Preliminaries}, we review some preliminaries for convergence analysis. In section~\ref{section: Main results}, we analyze the convergence of HPPP. Combining HPPP and denoiser priors, we propose GraRED-HP$^3$ for IR problems. In section~\ref{sec:experiments}, we verify the advantages of HPPP with several toy examples. Furthermore, we validate the performance of GraRED-HP$^3$ through IR experiments. Finally, conclusions are presented in section~\ref{sec:conclusions}.
\section{Preliminaries}
\label{section: Preliminaries}
In this section, we provide some fundamental concepts related to the degenerate PPP method and denoiser priors. Let $ \mathcal{H} $  be a real Hilbert space with inner product $ \langle \cdot,\cdot\rangle $ and corresponding induced norm $ \normm{\cdot} $, and let $ \mathcal{A}:  \mathcal{H} \to 2^{\mathcal{H}}$ be a (maybe multivalued) operator (that through the rest of this paper we often identify with its graph in $\mathcal{H}\times \mathcal{H}$). 
\subsection{Preconditioned proximal point}
Bredies et al.~\cite{Bredies2022} introduced the degenerate PPP framework with the positive semidefinite preconditioner $ \mathcal{M}: \mathcal{H}\to \mathcal{H} $. The proper preconditioner $ \mathcal{M}$ can make $ \mathcal{A}+\mathcal{M} $ have a lower triangular structure, which conveniently calculates the inverse $ (\mathcal{A}+\mathcal{M})^{-1}$. 
\begin{definition}
An admissible preconditioner for the operator $\mathcal{A}: \mathcal{H}\to 2^\mathcal{H}$ is a bounded, linear, self-adjoint, and positive semidefinite operator $\mathcal{M}:\mathcal{H}\to \mathcal{H}$ such that
\begin{displaymath}
\mathcal{T}=(\mathcal{M}+\mathcal{A})^{-1}\mathcal{M}
\end{displaymath}
is single-valued and has full domain.
\end{definition}

Therefore, the PPP iteration is written as
\begin{equation}\label{PPP iteration}
\bu^0\in \mathcal{H},\bu^{k+1} = \mathcal{T}\bu^k = (\mathcal{M}+\mathcal{A})^{-1}\mathcal{M}\bu^k.
\end{equation}
If $ \mathcal{M} =I $, then $ \mathcal{T} $ is a firmly nonexpansive (FNE) operator~\cite[Proposition 23.8]{BCombettes} and~\eqref{PPP iteration} becomes the standard proximal point iteration. If $ \mathcal{M} $  is positive semidefinite, the operator $ \mathcal{T} $ is related with the degenerate $\mathcal{M}$-firmly nonexpansive ($\mathcal{M}$-FNE) operator~\cite{Bredies2015,BUI2020124315,Xue2022,Bredies2022}, which is associated with seminorm $ \normm{\bu}_\mathcal{M} = \sqrt{\langle \mathcal{M}\bu, \bu\rangle} $ and semi inner-product $ \langle \bu,\bv\rangle_\mathcal{M} = \langle \mathcal{M}\bu,\bv\rangle $. The following notion extends monotone characteristics, i.e., $ \mathcal{M}$-monotonicity.
\begin{definition}[$\mathcal{M}$-monotonicity]
Let $\mathcal{M}:\mathcal{H}\to \mathcal{H}$ be a bounded linear positive semidefinite operator; then $\mathcal{B}:\mathcal{H}\to 2^\mathcal{H}$ is $\mathcal{M}$-monotone if we have
\[\langle \bv-\bv', \bu-\bu'\rangle _\mathcal{M}\geq 0\quad \forall (\bu, \bv), (\bu',\bv')\in \mathcal{B}.\]
\end{definition}

The following lemma demonstrates that the firmly nonexpansive notation can be generalized to the degenerate case~\cite{Bredies2015,BUI2020124315}. If $\mathcal{M}$ is an admissible preconditioner and $\mathcal{M}^{-1}\mathcal{A}$ is $\mathcal{M}$-monotone, then the operator $\mathcal{T}$ is $\mathcal{M}$-FNE.
\begin{lemma}[\cite{Bredies2022}]\label{lemma: MFNE}
	Let $ \mathcal{A}:\mathcal{H}\to 2^\mathcal{H}$ be an operator with $\mathrm{zer} \mathcal{A} \neq \emptyset$, and let $\mathcal{M}$ be an admissible preconditioner such that $\mathcal{M}^{-1}\mathcal{A}$ is $\mathcal{M}$-monotone. Then $\mathcal{T}$ is $\mathcal{M}$-FNE, i.e.,
	\begin{equation}\label{eq: MFNE}
		\left\|\mathcal{T}\bu-\mathcal{T}\bv\right\|_{\mathcal{M}}^2+\left\|(I-\mathcal{T})\bu-(I-\mathcal{T})\bv\right\|^2_{\mathcal{M}}\leq \left\|\bu-\bv\right\|^2_{\mathcal{M}}.
	\end{equation}
\end{lemma}
 
By the KM iteration, the degenerate PPP method is written as
\begin{equation}\label{eq: PPP}
\bu^{k+1}=(1-\lambda_k)\bu^k + \lambda_k \mathcal{T}\bu^k,
\end{equation}
where $\lambda_k$ is a sequence in $[0,2]$ such that $\sum_{k\in \mathbb{N}}\lambda_k (2-\lambda_k) = +\infty$.
\subsection{PnP prior and RED}
\subsubsection{PnP prior}
Venkatakrishnan, Bouman, and Wohlberg proposed the first PnP method based on ADMM~\cite{Venkatakrishnan2013}. PnP methods replace the proximal mapping of the implicit regularizer defined by~\eqref{def:prox} with a denoiser $D_\sigma:\mathbb{R}^n\to\mathbb{R}^n$. PnP-ADMM (Algorithm~\ref{alg:pnp-admm}) is well known for its fast empirical convergence and efficiency in computational imaging~\cite{PnP-CI2023,wu2024principled}. Sreehari et al. established theoretical conditions for PnP-ADMM, requiring that the Jacobian $\nabla D_\sigma$ be a doubly stochastic and symmetric matrix with all real eigenvalues in the range $(0,1]$~\cite{PnP2015}. When the denoiser meets the necessary and sufficient conditions of Proposition~\ref{prop: characterization of proximal operators}, it is a proximal mapping of the implicit regularizer $\phi:\mathbb{R}^n\to\mathbb{R} \cup \{+\infty\}$. Subsequently, a nonlocal means denoiser satisfying these conditions was constructed~\cite{PnP2015}.
\begin{algorithm}[ht]
	\caption{PnP-ADMM }
	\label{alg:pnp-admm}
	\begin{algorithmic}[1]
		\State \textbf{Input:} Given $\mathbf{u}^0 = \mathbf{0},\mathbf{x}^0,\mathbf{z}^0$ and total iterations $N>0$ 
		\For{$k = 0, 1, 2, \ldots, N-1$}
		\State $\mathbf{x}^{k+1} = D_\sigma(\mathbf{z}^{k} -\mathbf{u}^{k})$
		\State $\mathbf{z}^{k+1} = \text{prox}_{\lambda f}(\mathbf{x}^{k+1} + \mathbf{u}^{k})$
		\State $\mathbf{u}^{k+1} = \mathbf{u}^{k} + (\mathbf{x}^{k+1} - \mathbf{z}^{k+1})$
		\EndFor
		\State \textbf{Output:} $\bx^N$
	\end{algorithmic}
\end{algorithm}
\begin{definition}[\cite{beck2017first}]\label{def: proximal operator}
	The proximal operator of $  \phi:\bbR^n\to \mathbb{R} \cup \{+\infty\} $ is defined by
	\begin{equation}\label{def:prox}
	\mathrm{prox}_{\phi} (\mathbf{x})= \arg\min_{\mathbf{u}\in\bbR^n} \{\frac 12\normm{\mathbf{u}-\mathbf{x}}^2+\phi(\mathbf{u})\}.
	\end{equation}
\end{definition}
\begin{proposition}[\cite{moreau1965proximite,Gribonval2020}]\label{prop: characterization of proximal operators}
	A function $ h:\bbR^n\to \bbR^n$ defined everywhere is the proximal operator of a proper convex lower semicontinuous function $ \phi:\bbR^n\to \mathbb{R} \cup \{+\infty\} $  if and only if the following conditions hold jointly:
	\begin{enumerate}[label={(\roman*)},leftmargin=1cm]
	\item[(a)] There exists a convex lower semicontinuous function $ \psi $ such that for each $ \mathbf{x}\in \bbR^n, h(\mathbf{x}) = \nabla \psi(\mathbf{x})$;
	\item[(b)]$ h $ is nonexpansive, i.e.,
	\[\normm{h(\mathbf{x})-h(\mathbf{y})}\leq \normm{\mathbf{x}-\mathbf{y}}\quad \forall \mathbf{x},\mathbf{y}\in \bbR^n.\]
	\end{enumerate}
	\end{proposition}
\subsubsection{RED prior}
RED~\cite{Romano2017} is defined by
\begin{equation}\label{RED}
g_{\text{red}}(\mathbf{x}) = \frac 12 \langle \mathbf{x},\mathbf{x}-D_\sigma(\mathbf{x})\rangle,
\end{equation}
where $ D_\sigma:\bbR^n\to\bbR^n$ is a denoiser which is assumed to obey the following assumptions:
\begin{enumerate}[label={\bf (C\arabic*)}, ref={\bf C\arabic*}]
    \item \label{C1} Local homogeneity: For all $ \mathbf{x} \in \mathbb{R}^n $, $ D_\sigma \left( (1+\epsilon)x \right) = (1+\epsilon)D_\sigma(\mathbf{x}) $ for sufficiently small $ \epsilon > 0 $.
    \item \label{C2} Differentiability: The denoiser $ D_\sigma(\cdot) $ is differentiable.
    \item \label{C3} Jacobian symmetry~\cite{REDClarifications2019}: $ [\nabla D_\sigma(\mathbf{x})]^{\mathrm{T}} = \nabla D_\sigma(\mathbf{x})$ for all $ \mathbf{x} \in \mathbb{R}^n $.
    \item \label{C4} Strong passivity: The spectral radius of the Jacobian satisfies $ \eta \left( \nabla D_\sigma(\mathbf{x}) \right) \leq 1 $.
\end{enumerate}

If $ D_\sigma(\mathbf{x}) $ satisfies~\ref{C1},~\ref{C2}, and~\ref{C3}, then $ \nabla g_{\text{red}}(\mathbf{x}) = \mathbf{x}-D_\sigma(\mathbf{x}) = R(\mathbf{x}) $. Moreover, if the denoiser satisfies~\ref{C4}, then RED is convex and $D_\sigma$ is nonexpansive. 
\subsubsection{Implicit Gradient RED}
\label{subsection:GraRED}
The gradient of RED is exactly the denoising residual. We show that under ideal conditions and assuming the residual is nonexpansive, both the denoiser and its residual function are proximal operators of implicit regularizers. Even if all RED conditions are hard to satisfy, learning FNE operators can still achieve this.

The nonexpansive assumption has theoretically played a crucial role in PnP prior and RED. Ryu et al. proposed the following assumption to guarantee the convergence of PnP methods~\cite{ryu2019plug}: 
\begin{equation}\label{assumption: nonexpansive residual}
	\left\| (I-D_\sigma)(\mathbf{u})- (I-D_\sigma)(\mathbf{v}) \right\|\leq \varepsilon \left\| \mathbf{u}-\mathbf{v} \right\|;  \tag{A}
\end{equation}
for any $\mathbf{u},\mathbf{v}\in \mathbb{R}^n$ and $\varepsilon\leq 1$, the real spectral normalization is used to obtain nonexpansive residual. 
Under the conditions~\ref{C1},~\ref{C2},~\ref{C3}, and~\ref{C4} of RED and assumption~\eqref{assumption: nonexpansive residual}, we can further discuss the relationship between $\nabla g_{\text{red}}$ and the proximal operator, and we prove that there exists an implicit regularizer $ \phi $ such that $ D_\sigma(\mathbf{x}) = \mathrm{prox}_{\phi}(\mathbf{x}) $ and $R(\mathbf{x}) = \nabla g_{\text{red}}(\mathbf{x}) = \mathrm{prox}_{\phi^*}(\mathbf{x})$, where $\phi^*$ is the conjugate of $\phi$.

\begin{lemma}\label{lemma: gradient of RED}
Assume that a denoiser $ D_\sigma:\bbR^n\to\bbR^n $ satisfies conditions~\ref{C1},~\ref{C2},~\ref{C3},~\ref{C4} and assumption~\eqref{assumption: nonexpansive residual}; then there exists an implicit regularizer $\phi:\bbR^n\to \mathbb{R} \cup \{+\infty\} $ such that 
\begin{equation}
D_\sigma(\mathbf{x}) =\mathrm{prox}_{\phi}(\mathbf{x}),R(\mathbf{x}) = \mathrm{prox}_{\phi^*}(\mathbf{x}).
\end{equation}
\end{lemma}
\begin{proof}
By assumption conditions, the RED term $g_{\text{red}}$ is convex, and $\nabla g_{\text{red}} = I-D_\sigma = R$. Let $\psi(\mathbf{x}) = \frac{1}{2}\left\| \mathbf{x} \right\|^2-g_{\text{red}}$, then $D_\sigma = I- \nabla g_{\text{red}} = \nabla\psi$, which is a GS denoiser. Since $R$ is nonexpansive and $D_\sigma$ is differentiable, then $\eta(\nabla^2\psi) =\eta(\nabla D_\sigma) = \eta (I-\nabla R) \in [0,1]$, it follows that $\psi$ is convex. Applying Proposition~\ref{prop: characterization of proximal operators}, together with $h(\mathbf{x}) = D_\sigma(\mathbf{x}),\psi = \frac{1}{2}\left\| \mathbf{x} \right\|^2-g_{\text{red}}$ and Moreau decomposition~\cite[Theorem 6.44]{beck2017first}, i.e., $\mathrm{prox}_{\phi}(\mathbf{x}) +\mathrm{prox}_{\phi^*}(\mathbf{x}) =\mathbf{x}$, there exists a function $ \phi $ and its conjugate $\phi^*$ such that $ D_\sigma(\mathbf{x}) = \nabla \psi(\mathbf{x}) = \mathrm{prox}_{\phi}(\mathbf{x}) $, and $ R(\mathbf{x}) =\mathrm{prox}_{\phi^*}(\mathbf{x}) $, which completes the proof.
\end{proof}

Satisfying all RED conditions is challenging; in practice, only the nonexpansiveness assumption is required. From the perspective of monotone operator theory, Lemma~\ref{lemma: resolvent condtions}~\cite[Corollary 23.9]{BCombettes} provides a necessary and sufficient condition for a denoiser to be a proximal operator of an implicit regularizer.
\begin{lemma}\label{lemma: resolvent condtions}
	The mapping $T: \mathcal{X} \to \mathcal{X}$ is the resolvent of a maximal monotone operator if and only if $T$ is FNE.
\end{lemma}
Moreover, by~\cite[Proposition 4.4]{BCombettes}, the operator $T$ is FNE if and only if there exists a nonexpansive operator $S:\mathcal{X} \to \mathcal{X}$ such that $T =\frac{S+I}{2}$. Therefore, the core issue in explaining denoiser prior as a proximal operator of a convex regularizer is training a nonexpansive operator. Recent developments, such as spectral normalization~\cite{ryu2019plug,pmlr-v202-delattre23a}, learning maximally monotone operators~\cite{MMO2021}, learning FNE denoisers~\cite{building_FNE,bredies2024learningfirmlynonexpansiveoperators}, and learning pseudocontractive denoisers~\cite{Wei2024}, have made the nonexpansive assumption more realistic in practice. Constructing the FNE denoiser $D_\sigma$ is feasible, similarly, the residual $R = I-D_\sigma$ can be FNE in~\eqref{eq: MFNE} with $\mathcal{M}=I$. Both can be proximal operators of implicit regularizers.
\section{Halpern-type preconditioned proximal point (HPPP)}
\label{section: Main results}
Compared with KM iteration, the classic Halpern iteration offers the advantage of strong convergence in infinite-dimensional Hilbert spaces~\cite{he2024convergence}. The Halpern iteration $\bu^{k+1}=\lambda_{k+1}\bu^0+(1-\lambda_{k+1})T\bu^k$~\cite{halpern1967fixed} is an implicitly regularized method for finding a particular fixed point~\cite{Diakonikolas2020}, and the sequence $ \{\bu^k\}_{k\in\mathbb{N}} $ generated by the Halpern iteration with suitable $ \{\lambda_k\}_{k\in\mathbb{N}} $ strongly converges to the projection $P_{\mathrm{Fix}(T)}(\bu^0)$~\cite{halpern1967fixed,lions1977approximation,Wittmann1992,2002Iterative,xu2004viscosity, Qi2021,he2024convergence}, where $ P_{\mathrm{Fix}(T)}(\bu^0) = \arg\min_{\mathrm{Fix}(T)} \normm{\bu-\bu^0}^2 $. When PPP is used to solve the inclusion problem~\eqref{problem: inclusion}, the sequence $ \{\bu^k\}_{k\in\mathbb{N}} $ generated by PPP can only converge weakly to some uncertain fixed point of $ \mathcal{T} $. Based on the Halpern iteration, we propose the Halpern-type preconditioned proximal point (HPPP) algorithms~\eqref{eq: HPPP}. 
\subsection{Convergence analysis}
First, we analyze the convergence of the sequence $\{\bu^k\}_{k\in \mathbb{N}}$ generated by HPPP~\eqref{eq: HPPP}. If $ \mathcal{T} $ satisfies the mild condition $\left\| \mathcal{T}\bu-\mathcal{T}\bv \right\|\leq C\left\| \bu-\bv \right\|_\mathcal{M}(C>0)$~\cite{Bredies2022}, then there exists a unique solution $ \bu^* = \arg\min_{\bu\in\mathrm{Fix}(\mathcal{T})}\normm{\bu-\mathbf{a}}_\mathcal{M}^2$, which corresponds to the $\mathcal{M}$-projection of $\mathbf{a}$ onto $\mathrm{Fix}(\mathcal{T})$ from Lemma~\ref{ineq: generalized VI}. In the subsequent analysis, the notation $\bu^*$ will always denote this $\mathcal{M}$-projection. We use "$ \to $" for \textit{strong convergence} and "$\rightharpoonup $" for \textit{weak convergence}.
\begin{theorem}\label{thm:bigthm}
	Let $ \mathcal{A}:\mathcal{H}\to 2^\mathcal{H}$ be an operator with $\mathrm{zer} \mathcal{A} \neq \emptyset$, and let $\mathcal{M}$ be an admissible preconditioner such that $\mathcal{M}^{-1}\mathcal{A}$ is $\mathcal{M}$-monotone and $(\mathcal{M}+\mathcal{A})^{-1}$ is $L$-Lipschitz. Let $\{ \bu^k \}$ be the sequence generated by HPPP~\eqref{eq: HPPP}. Assume that every weak cluster point of $ \{\bu^k\}_{k\in\mathbb{N}} $ lies in $ \mathrm{Fix}(\mathcal{T}) $, and that $ \{\mu_k\}_{k\in\mathbb{N}} $ satisfies 
	\begin{enumerate}[label={(\roman*)},leftmargin=1cm]
		\item $ \lim_{k\to \infty}\mu_k = 0 $,
		\item $ \sum_{k\in \mathbb{N}} \mu_k = +\infty $,
		\item $ \lim_{k\to \infty}\frac{\mu_{k+1}-\mu_k}{\mu_k} = 0\text{ or }\sum_{k\in \mathbb{N}}\left\vert \mu_{k+1}-\mu_k \right\vert<\infty $.
	\end{enumerate}
	Then $\{\bu^k\}_{k\in\mathbb{N}}$ converges strongly to $ \bu^*$.
\end{theorem}
\begin{proof}
	First, we show $\lim_{k\to \infty}\left\| \bu^k-\bu^* \right\|_\mathcal{M}= 0$. 
	Based on the definition of the upper limit, we take a subsequence $\{ \mathbf{u}^{k_n} \}_{n\in \mathbb{N}}$ of $\{ \mathbf{u}^{k} \}_{k\in \mathbb{N}}$ such that 
	\begin{equation}\label{ineq: VI of M projection}
	\limsup_{k\to \infty}\left \langle \mathbf{a}-\bu^*, \mathcal{T}\bu^k-\bu^*\right\rangle_ \mathcal{M} = \lim_{n\to\infty} \left \langle \mathbf{a}-\bu^*, \mathcal{T} \bu^{k_n}-\bu^*\right\rangle_ \mathcal{M}.
	\end{equation}
	By Lemma~\ref{lemma: lemma for regularized fixed point}(i), the sequence $\{ \mathbf{u}^{k_n} \}_{n\in \mathbb{N}}$ is bounded, there exists a weakly convergent subsequence $\{ \mathbf{u}^{k_{n_j}} \}_{j\in \mathbb{N}}$, and we may assume $\mathbf{u}^{k_{n_j}}\rightharpoonup \tilde{\mathbf{u}}$. According to the conditions of the theorem, we have $ \tilde{\mathbf{u}}\in \mathrm{Fix}(\mathcal{T}) $.

	Applying Lemma~\ref{lemma: lemma for regularized fixed point}, together with conditions (i)-(iii), we obtain $\mathcal{T}\mathbf{u}^k-\mathbf{u}^k\to \mathbf{0}$; then $\langle \mathbf{a}-\mathbf{u}^*,\mathcal{T}\mathbf{u}^k-\mathbf{u}^k\rangle_{\mathcal{M}} \to 0$. Hence, we have
	\begin{align*}
	\lim_{n\to\infty} \left \langle \mathbf{a}-\bu^*, \mathcal{T} \mathbf{u}^{k_n}-\bu^*\right\rangle_ \mathcal{M} & = \lim_{j\to \infty}\left \langle \mathbf{a}-\bu^*, (\mathcal{T}\mathbf{u}^{k_{n_j}}-\mathbf{u}^{k_{n_j}})+(\mathbf{u}^{k_{n_j}}-\bu^*)\right\rangle_ \mathcal{M} \\
		& = \lim_{j\to \infty} \left \langle \mathbf{a}-\bu^*, \mathbf{u}^{k_{n_j}}-\bu^*\right\rangle_ \mathcal{M}  \\
		& = \left \langle \mathbf{a}-\bu^*, \tilde{\bu}-\bu^*\right\rangle_ \mathcal{M}.
	\end{align*}

	By Lemma~\ref{lemma: generalized VI}, we have 
	\begin{equation}
		\limsup_{k\to \infty}\left \langle \mathbf{a}-\bu^*, \mathcal{T}\bu^k-\bu^*\right\rangle_ \mathcal{M} = \left \langle \mathbf{a}-\bu^*, \tilde{\bu}-\bu^*\right\rangle_ \mathcal{M}\leq 0.
	\end{equation}

	Denoting $\delta_{k+1} = \mu_{k+1}\left\| \mathbf{a}-\bu^* \right\|_ \mathcal{M}^2+2(1-\mu_{k+1})\left \langle \mathbf{a}-\bu^*, \mathcal{T}\bu^{k}-\bu^*\right\rangle_ \mathcal{M}$, noticing that
	\begin{align*}
	\left\| \bu^{k+1}-\bu^* \right\|_ \mathcal{M}^2 & = \left\| \mu_{k+1}(\mathbf{a}-\bu^*)+(1-\mu_{k+1})(\mathcal{T}\bu^k-\bu^*) \right\|_\mathcal{M}^2 \\
	& = \mu_{k+1}^2\left\| \mathbf{a}-\bu^* \right\|_\mathcal{M}^2+(1-\mu_{k+1})^2\left\| \mathcal{T}\bu^k-\bu^* \right\|_ \mathcal{M}^2 \\
	&+2\mu_{k+1}(1-\mu_{k+1})\left \langle \mathbf{a}-\bu^*, \mathcal{T}\bu^k-\bu^*\right\rangle_ \mathcal{M} \\
	& \leq (1-\mu_{k+1})\left\| \bu^k-\bu^* \right\|_\mathcal{M}^2+\mu_{k+1}\delta_{k+1},
	\end{align*}
	and applying Lemma~\ref{lemma for convergence}, together with conditions (i)-(ii) and $\limsup_{k\to \infty }\delta_k \leq 0$, we have $\lim_{k\to \infty}\left\| \bu^k-\bu^* \right\|_ \mathcal{M} = 0.$

	By the definition of $\mathbf{u}^{k+1}$ we immediately obtain
		\begin{align}
			\left\| \bu^{k+1}-\bu^* \right\| & = \left\| \mu_{k+1}(\mathbf{a}-\bu^*) +(1-\mu_{k+1})(\mathcal{T}\bu^k-\bu^*)\right\| \notag\\
			&\leq \mu_{k+1}\left\| \mathbf{a}-\bu^* \right\|+(1-\mu_{k+1})\left\| \mathcal{T}\bu^k - \bu^* \right\| \notag\\
			&\leq \mu_{k+1}\left\| \mathbf{a}-\bu^* \right\|+C(1-\mu_{k+1})\left\| \bu^k - \bu^* \right\|_\mathcal{M},\label{ineq: u_k+1-u*}
			\end{align}
	where the last inequality follows from~\eqref{ineq: boundedness by M norm} and $C$ is defined as in~\eqref{constant C}.
	Since $\lim_{k\to\infty}\mu_k = 0$ and $\lim_{k\to \infty}\left\| \bu^k-\bu^*\right\|_\mathcal{M}= 0$, combined with~\eqref{ineq: u_k+1-u*} they yield that 
	\begin{equation*}
		\lim_{k\to \infty} \left\| \mathbf{u}^k-\mathbf{u}^* \right\|  = 0,
	\end{equation*}
	which ends the proof of Theorem~\ref{thm:bigthm}.
\end{proof}

It is easy to verify that $ \mu_k = \frac{1}{k^\alpha}(0<\alpha\leq 1) $ satisfies conditions (i)-(iii) of Theorem~\ref{thm:bigthm}. Compared with~\cite{Bauschke2023}, we have provided an alternative method to obtain the limit point as the $ \mathcal{M}$-projection of the initial point. The work~\cite{Bauschke2023} requires that $\mathcal{A}$ be linear, whereas we consider a more general case for $\mathcal{A}$. By introducing the maximal monotonicity assumption on $\mathcal{A}$, we can prove that every weak cluster point $\mathbf{u}$ of $\{\mathbf{u}^k\}_{k\in\mathbb{N}}$ lies in $\mathrm{Fix}(\mathcal{T})$. According to Theorem~\ref{thm:bigthm}, we obtain the following proposition.
\begin{proposition}\label{cor:bigthm cor}
	If the conditions of Theorem~\ref{thm:bigthm} hold, and $ \mathcal{A}:\mathcal{H}\to 2^\mathcal{H}$ is a maximal monotone operator, then $\{ \bu^k \}_{k\in \mathbb{N}}$ converges strongly to $\bu^*$.
\end{proposition}
	\begin{proof}
	Assume that $\mathbf{u}^{k_n}\rightharpoonup \mathbf{u}$. By Lemma~\ref{lemma: lemma for regularized fixed point}(iii), we have $ \mathcal{T}\bu^{k_n}-\bu^{k_n}\to \mathbf{0} $, and it follows that
	\begin{displaymath}
	\mathcal{T}\bu^{k_n} = (\mathcal{T}\bu^{k_n}-\bu^{k_n})+\bu^{k_n}\rightharpoonup \mathbf{0}+\bu = \bu,
	\end{displaymath}
	i.e., $ \mathcal{T} \bu^{k_n}\rightharpoonup \bu $. Using $ \mathcal{T}\bu^{k_n}-\bu^{k_n}\to \mathbf{0} $, we have
	\begin{equation*}
	\mathcal{A}\mathcal{T}\bu^{k_n} \owns \mathcal{M}(\bu^{k_n}-\mathcal{T}\bu^{k_n})\to \mathbf{0}.
	\end{equation*}
	Since $ \mathcal{A} $ is maximal, we have that $ \mathcal{A} $ is closed in $ \mathcal{H}_{weak}\times  \mathcal{H}_{strong}$~\cite[Proposition 20.38]{BCombettes}, and hence $ \mathbf{0}\in \mathcal{A}\bu $, i.e., $ \bu \in \mathrm{zer}\mathcal{A} = \mathrm{Fix}(\mathcal{T})$. Thus, we prove that every weak cluster point of $ \{\bu^k\}_{k\in\mathbb{N}}$ lies in $ \mathrm{Fix}(\mathcal{T})$. This completes the proof of Proposition~\ref{cor:bigthm cor}. 
\end{proof}
\begin{remark}
	We compare Theorem~\ref{thm:bigthm}, Proposition~\ref{cor:bigthm cor} with~\cite[Theorem 2.9]{Bredies2022} and~\cite[Corollary 2.10]{Bredies2022}, where the sequence $ \{\bu^k\}_{k\in\mathbb{N}} $ generated by HPPP converges strongly to a particular fixed point of $ \mathcal{T} $. All conditions are the same, except for the additional assumption about $ \{\mu_k\}_{k\in\mathbb{N}} $. The Lipschitz regularity of $ (\mathcal{M}+\mathcal{A})^{-1} $ is a mild assumption, especially in applications to splitting algorithms, and is used to prove the uniqueness of $ \mathcal{M}$-projection and guarantee the boundedness of  $ \{\bu^k\}_{k\in\mathbb{N}},\{\mathcal{T}\bu^k\}_{k\in\mathbb{N}} $.
\end{remark}

Hundal constructed an example in which the alternating projection $(P_UP_V)^k \mathbf{x}^0$ converges weakly but not strongly~\cite{HUNDAL200435}, where $U$ is a closed convex cone and $V$ is a closed hyperplane. Building on Hundal's counterexample,~\cite{DR_weak_convergence2020} showed that the weak convergence of the DRS algorithm cannot be improved to strong convergence. Similarly, using the same counterexample, we can demonstrate that PPP also fails to converge strongly.
\begin{counterexample}\label{counterexample}
	Suppose that $\mathcal{H}$ is infinite-dimensional and
	separable. Let $V$ and $U$ be Hundal's hyperplane and Hundal's cone, respectively, and set
	\begin{equation*}
		\mathcal{A}_1: \mathbf{x}\mapsto \begin{cases}
			V^{\perp } & \text{ if } \mathbf{x}\in V, \\
			\emptyset  & \text{ if } \mathbf{x} \notin V,
		\end{cases}
		\text{ and } \mathcal{A}_2 = (P_V\circ P_U\circ P_V)^{-1}-I.
	\end{equation*}
	Then the sequence $\{ \mathbf{u}^k \}_{k\in \mathbb{N}}$ generated by PPP~\eqref{eq: PPP} with $\lambda_k=1$ converges weakly, but not strongly, to a zero of $\mathcal{A}_1+\mathcal{A}_2$. 
\end{counterexample}
\begin{proof}
	As proved in~\cite{DR_weak_convergence2020}, the operators $\mathcal{A}_1$ and $\mathcal{A}_2$ are maximally monotone, and $\mathbf{0}\in \mathrm{zer}(\mathcal{A}_1+\mathcal{A}_2)$, and the sequence $\{ \mathbf{w}^k \}_{k\in \mathbb{N}}$ 
	\begin{equation}\label{DRS iteration}
		\mathbf{w}^{k+1} = \mathbf{w}^{k} + J_{\mathcal{A}_2} (2J_{\mathcal{A}_1}\mathbf{w}^{k}-\mathbf{w}^{k})- J_{\mathcal{A}_1}\mathbf{w}^{k},
	\end{equation}
	generated by the DRS algorithm can weakly converge to $\mathbf{0}$ and $ \mathbf{w}^k \not\rightarrow \mathbf{0}$.

	The sequence~\eqref{DRS iteration} can be viewed as a special PPP with
	\begin{equation*}
		\mathcal{A} = \begin{pmatrix}
			\mathcal{A}_1 & I \\
			-I & \mathcal{A}_2
		\end{pmatrix},\quad \mathcal{M} =\begin{pmatrix}
			I & -I \\ 
			-I & I
		\end{pmatrix},
	\end{equation*}
	and $\mathbf{w}^k = \mathcal{C}^* \mathbf{u}^k =\mathbf{x}^k-\mathbf{y}^k$, where $\mathcal{C}^* = (I,-I)$. By~\cite[Corollary 2.15]{Bredies2022}, $\mathbf{u}^k\rightharpoonup \mathbf{0}$. Indeed, if by contraction, $\mathbf{u}^k\to \mathbf{0}$, then
	\begin{equation*}
		\left\| \mathbf{w}^k -\mathbf{0} \right\| = \left\|  \mathcal{C}^* \mathbf{u}^k-\mathbf{0}\right\|\leq \left\| \mathcal{C} \right\| \cdot \left\| \mathbf{u}^k-\mathbf{0} \right\| \to 0,
	\end{equation*} 
	which contradicts the fact that $\mathbf{w}^k \not\rightarrow \mathbf{0}$. Thus $ \{ \mathbf{u}^k \}_{k\in \mathbb{N}}$ can only converge weakly to $\mathbf{0}$.
\end{proof}
\begin{remark}
	From Counterexample~\ref{counterexample}, HPPP has a strong convergence advantage over PPP in infinite-dimensional spaces, which has important theoretical value.
\end{remark}

\begin{remark}
Regarding the choice of the anchor point, we can use the degraded observed image, the denoised image, or the restarted technique (see Algorithm~\ref{alg: Restart HPPP}). The restarted technique is an effective method in machine learning~\cite{Park2022} and LP~\cite{chen2024hprlpimplementationhprmethod,lu2024restartedhalpernpdhglinear}. 
\end{remark}
\begin{algorithm}
	\caption{Restarted HPPP}
	\label{alg: Restart HPPP}
	\begin{algorithmic}[1]
		\Require Initialization $ \bu^{0,0}  = \mathbf{a}^0$, total iteration $N>0$, period $q>0$, and epoch $N_e = \lfloor \frac{N}{q} \rfloor$
		\For{$ n=1,2,\ldots,N_e$}
			\For{$ k=0,1,\ldots,q-1 $}
				\State $\bu^{n,k+1} =\mathrm{ HPPP} (\bu^{n,k},\mathbf{a}^n)$
			\EndFor
		\State $\mathbf{a}^{n+1} = \bu^{n,q}$
		\EndFor
		\Ensure $ \bu^{N_e,q}$
	\end{algorithmic}
\end{algorithm}
\subsection{Convergence rate}
\label{appendix: convergence rate}
Based on the following technical lemma, Sabach and Shtern~\cite{Sabach2017} first showed that, for a general form of Halpern iteration, the fixed-point residual converges at a rate of $\mathcal{O}(1/k)$.
\begin{lemma}[\cite{Sabach2017}]\label{lemma: convergence rate}
	Let $M_1>0$. Assume that $\{ a_k \}_{k\in \mathbb{N}}$ is a sequence of nonnegative real numbers such that $a_1<M_1$ and
	\[a_{k+1}\leq (1-\gamma b_{k+1})a_k +(b_k-b_{k+1})c_k, k\geq 1,\]
	where $\gamma\in (0,1]$, the sequence $\{ b_k \}_{k\in \mathbb{N}}$ is defined as in $b_k = \min \{ \frac 2{\gamma k},1 \}$, and $\{c_k\}_{k\in \mathbb{N}}$ is a sequence of real numbers such that $c_k\leq M_1<\infty $. Then the sequence $\{ a_k \}_{k\in \mathbb{N}}$ satisfies
	\[a_k\leq \frac{M_1J}{\gamma k}, k\geq 1,\]
	where $J =\lfloor \frac 2\gamma \rfloor$.
\end{lemma}

By applying Lemma~\ref{lemma: convergence rate}, we will establish a sublinear convergence rate for $\left\| \mathbf{u}^{k}-\mathbf{u}^{k-1} \right\|$ and fixed-point residual $\left\| \mathcal{T}\mathbf{u}^k -\mathbf{u}^k \right\| $ or $\left\| \mathcal{T}\mathbf{u}^k -\mathbf{u}^k \right\|_\mathcal{M} $.
\begin{proposition}\label{cor: convergence rate of HPPP}
	Let $ \mathcal{A}:\mathcal{H}\to 2^\mathcal{H}$ be a maximal operator with $\mathrm{zer} \mathcal{A} \neq \emptyset$, and let $\mathcal{M}$ be an admissible preconditioner such that $\mathcal{M}^{-1}\mathcal{A}$ is $\mathcal{M}$-monotone and $(\mathcal{M}+\mathcal{A})^{-1}$ is $L$-Lipschitz. Let $\{ u^k \}$ be the sequence generated by~\eqref{eq: HPPP}. If $\mu_k = \min \{ \frac 2{ k},1 \} $, then
	\begin{enumerate}[label={(\roman*)},leftmargin=1cm]
		\item $\left\| \bu^{k}-\bu^{k-1} \right\|_\mathcal{M}\leq \frac{2M}{k}$, i.e., $\left\| \bu^{k}-\bu^{k-1} \right\|_\mathcal{M} = \mathcal{O}\left(\frac 1k\right)$;
		\item $\left\| \bu^{k+1}-\bu^k \right\| =  \mathcal{O}\left(\frac 1k\right)$, $\left\| \bu^k-\mathcal{T}\bu^k \right\| = \mathcal{O}\left(\frac 1k\right)$.
	\end{enumerate}
\end{proposition}
\begin{proof}
	(i) According to~\eqref{ineq: convergence of M uk+1-uk}, we have
	\begin{displaymath}
	\left\| \bu^{k+1}-\bu^k \right\|_\mathcal{M} \leq (1-\mu_{k+1})\left\|\bu^k-\bu^{k-1} \right\|_\mathcal{M}+ M\left\vert \mu_{k+1}-\mu_k \right\vert,
	\end{displaymath}
	where $M = \left\| \mathbf{a}-\bu^* \right\|_\mathcal{M}+C_1, \mathbf{u}^*\in \mathrm{Fix}(\mathcal{T})$, and $C_1$ is defined as in~\eqref{constant C1}. Applying Lemma~\ref{lemma: convergence rate} together with $c_k = M_1 = M, J=2,\gamma = 1, a_k = \left\|\bu^k-\bu^{k-1} \right\|_\mathcal{M}$, and $\mu_k = b_k$ to the above inequality, we have $\left\| \bu^{k}-\bu^{k-1} \right\|_\mathcal{M}\leq \frac{2M}{k}$.

	(ii) By~\eqref{ineq: convergence of uk+1-uk}, when $k\geq 2$ we have
	\begin{align*}
	\left\| \bu^{k+1}-\bu^k \right\|&\leq C(1-\mu_{k+1})\left\| \bu^k-\bu^{k-1} \right\|_\mathcal{M} +M'\left\vert \mu_{k+1}-\mu_k\right\vert  \\
	& \leq \frac{2MC}{k}+\frac{2M'}{k(k+1)} = \mathcal{O}\left(\frac 1k\right),
	\end{align*}
	where the second inequality follows from (i) and the fact that $\mu_k -\mu_{k+1}=\frac{2}{k(k+1)}(k\geq 2)$, and $M' = \left\| \mathbf{a}-\bu^* \right\| +CC_1$, and $C, C_1$ are defined in~\eqref{constant C} and~\eqref{constant C1}, respectively.
	
	As for the second result, by the triangle inequality and $\mathbf{u}^{k+1}-\mathcal{T}\mathbf{u}^k =\mu_{k+1}(\mathbf{a}-\mathcal{T} \mathbf{u}^k)$, for $k\geq 2$,
	\begin{align*}
		\left\| \bu^k-\mathcal{T}\bu^k \right\|
		&\leq \left\| \bu^k-\bu^{k+1} \right\|+\left\| \bu^{k+1}-\mathcal{T}\bu^k \right\|\\
		&\leq \left\| \bu^k-\bu^{k+1} \right\|+\mu_{k+1}\left\| \mathbf{a}-\mathcal{T}\bu^k \right\| \\
		&\leq \frac{2MC}{k}+\frac{2M'}{k(k+1)}+\frac{2M'}{k+1} = \mathcal{O}\left(\frac 1k\right).
	\end{align*} 
\end{proof}

Let $\mathcal{M} = \mathcal{C}\mathcal{C}^*$ be a decomposition of $\mathcal{M}$ according to~\cite[Proposition 2.3]{Bredies2022}. By~\cite[Theorem 2.13]{Bredies2022}, the operator $\tilde{\mathcal{T}} = \mathcal{C}^*(\mathcal{M}+\mathcal{A})^{-1}\mathcal{C}$ is FNE. Let $\mathbf{w}^k=\mathcal{C}^* \mathbf{u}^k$ and $\mathbf{a} =\mathbf{u}^0$, and then HPPP is equivalent to the following reduced algorithm:
\begin{equation}\label{eq: HPPP reduced}
	\mathbf{w}^0 = \mathcal{C}^* \mathbf{u}^0,\mathbf{w}^{k+1} = \mu_{k+1}\mathbf{w}^0+(1-\mu_{k+1})\tilde{\mathcal{T}}\mathbf{w}^k.
\end{equation}
By~\cite[Theorem 2.1]{lieder2021convergence} and~\cite[Proposition 2.9]{Sun2024}, we  further give the following (tight) optimal convergence rate for HPPP.
\begin{proposition}\label{cor: accelerated convergence rate of HPPP}
	Let $ \mathcal{A}:\mathcal{H}\to 2^\mathcal{H}$ be a maximal operator with $\mathrm{zer} \mathcal{A} \neq \emptyset$, and let $\mathcal{M}$ be an admissible preconditioner such that $\mathcal{M}^{-1}\mathcal{A}$ is $\mathcal{M}$-monotone and $(\mathcal{M}+\mathcal{A})^{-1}$ is $L$-Lipschitz. Let $\{\mathbf{u}^k\}_{k\in \mathbb{N}},\{ \mathbf{w}^k\}_{k\in \mathbb{N}}$ be the sequences generated by~\eqref{eq: HPPP} and~\eqref{eq: HPPP reduced}. If $\mu_k = \frac{1}{k+1}$ and $\mathbf{a} = \mathbf{u}^0$, then
	\begin{enumerate}[label={(\roman*)},leftmargin=1cm]
		\item $\left\| \mathbf{w}^k-\tilde{\mathcal{T}}\mathbf{w}^k \right\|\leq \frac{2\left\| \mathbf{w}^0-\mathbf{w}^*\right\| }{k+1} $ for $k\geq 0$ and $\mathbf{w}^*\in \mathrm{Fix}(\tilde{\mathcal{T}})$, 
		\item $\left\| \mathbf{u}^k-{\mathcal{T}}\mathbf{u}^k \right\|_{\mathcal{M}}\leq \frac{2\left\| \mathbf{u}^0-\mathbf{u}^*\right\|_{\mathcal{M}}}{k+1} $ for $k\geq 0$ and $\mathbf{u}^*\in \mathrm{Fix}({\mathcal{T}})$. 
	\end{enumerate}
\end{proposition}
\subsection{GraRED-HP$^3$} 
\label{subsec:GraREDHP3}
The primal-dual algorithm for solving~\eqref{problem: saddle-point problem} is viewed as the fixed-point iteration $\mathbf{u}^{k+1} = \mathcal{T} \mathbf{u}^k = (\mathcal{A}+\mathcal{M})^{-1} \mathcal{M}\mathbf{u}^k$ with $\mathcal{A} = \begin{pmatrix}
	\lambda \partial f & \mathbf{K}^* \\
	- \mathbf{K} & \partial g^*
\end{pmatrix}, \mathcal{M}= \begin{pmatrix} \frac{1}{\tau}I & -\mathbf{K}^* \\ -\mathbf{K} & \frac{1}{s}I \end{pmatrix}$. Under the degenerate case $ \tau s\normm{\mathbf{K}}^2=1 $, the HPPP iteration is given by
\begin{equation}\label{eq: HPPP for saddle-point problem}
\left\{
\begin{array}{ll}
\mathbf{d}^k & =(I+\tau \lambda \partial f)^{-1}(\bx^k-\tau \mathbf{K}^*\by^k),\\
\bx^{k+1} & = \mu_{k+1} \bx_a+(1-\mu_{k+1})\mathbf{d}^k,\\
\by^{k+1} & =\mu_{k+1} \by_a+(1-\mu_{k+1})(I+ s \partial g^*)^{-1}\left(\by^k+s\mathbf{K}(2 \mathbf{d}^k-\mathbf{x}^k)\right),
\end{array}
\right.
\end{equation}
where $ \mathbf{a} = (\bx_a, \by_a), (\bx_0, \by_0)\in\mathcal{X}\times \mathcal{Y} $ are the anchor and initial points. Under the ideal conditions of RED and with assumption~\eqref{assumption: nonexpansive residual}, according to Lemma~\ref{lemma: gradient of RED}, both the denoiser $D_\sigma$ and the residual $R$ can be expressed as the proximal operators of implicit regularizers, i.e., $ D_\sigma (\mathbf{x})= \mathrm{prox}_{\phi}(\bx), R(\bx) = \mathrm{prox}_{\phi^*}(\bx) $. Let $g(\mathbf{x}) =\frac{\phi (s \mathbf{x})}{s}$; then $\phi^* = sg^*$~\cite[Theorem 4.14]{beck2017first}, and we can replace $(I+ s \partial g^*)^{-1}$ with the residual $R$ in~\eqref{eq: HPPP for saddle-point problem}. Set $\mathbf{K}=I$; the implicit gradient RED via HPPP called GraRED-HP$^3$ is proposed in Algorithm~\ref{alg: GraRED-HP3}. Even if the conditions of RED are difficult to meet~\cite{REDClarifications2019}, recent developments~\cite{ryu2019plug,building_FNE,MMO2021,pmlr-v202-delattre23a,BUI2020124315,bredies2024learningfirmlynonexpansiveoperators,Wei2024} can learn FNE denoisers. According to Lemma~\ref{lemma: resolvent condtions}, both the denoiser and the residual still act as the proximal operators of implicit regularizers. 
\subsubsection{Acceleration for PnP methods}
PnP-ADMM (see Algorithm~\ref{alg:pnp-admm}) is a well-known method for solving~\eqref{problem: optimization problem for image restoration}, which can be written into an equivalent DRS form~\cite{ryu2019plug,hurault2022proximal}, i.e., $\mathbf{w}^{k+1}=\tilde{\mathcal{T}}\mathbf{w}^k = \mathbf{w}^{k} + D_\sigma (2\mathrm{prox}_{\lambda f}(\mathbf{w}^{k})-\mathbf{w}^{k})- \mathrm{prox}_{\lambda f}(\mathbf{w}^{k})$. When $ \mathbf{w}^k = \bx^k-\by^k,\tau=s=1, \mu_ k =0 $ in GraRED-HP$^3$ results exactly in the DRS iteration. Thus, PnP-DRS is a special case of GraRED-HP$^3$, which can be obtained from the perspective of PPP~\cite{Bredies2015,Bredies2022}. The related GraRED-P$^3$ algorithm is shown in Algorithm~\ref{alg: GraRED-P3}. Furthermore, Proposition~\ref{cor: accelerated convergence rate of HPPP} demonstrates that GraRED-HP$^3$ is an accelerated PnP method achieving an $\mathcal{O}(1/k)$ convergence rate for the fixed-point residual.
 \begin{remark}
	In~\cite{He2015}, a worst-case convergence rate of $\mathcal{O}\left(\frac{1}{\sqrt{k}}\right)$ for $\left\| \mathbf{w}^k- \tilde{\mathcal{T}}\mathbf{w}^k\right\|$ was established for DRS, which also applies to PnP-ADMM. In contrast, GraRED-HP$^3$ obtains an accelerated convergence rate of $\mathcal{O}\left(\frac{1}{k}\right)$ for $\left\| \mathbf{w}^k- \tilde{\mathcal{T}}\mathbf{w}^k\right\|$.
 \end{remark}
\begin{algorithm}[H]
	\caption{GraRED-HP$ ^3 $}
	\label{alg: GraRED-HP3}
	\begin{algorithmic}[1]
		\Require Initialization $ \bu^0 =(\bx^0,\by^0)\in\bbR^{n}\times \bbR^{n}$, anchor point $ \mathbf{a} = (\bx_a,\by_a) \in\bbR^{n}\times \bbR^{n}$, iteration number $ N>0 $, and $ R = I-D_\sigma $ is the residual.
		\For{$ k=0, 1,2,\ldots, N-1 $}
		\State $  \mathbf{d}^k = \mathrm{prox}_{\lambda f}(\bx^k-\tau \by^k)$
		\State $ \bx^{k+1} = \mu_{k+1}\bx_a+(1-\mu_{k+1})\mathbf{d}^k$
		\State $ \by^{k+1} = \mu_{k+1}\by_a+(1-\mu_{k+1})R\left(\by^k+s(2\mathbf{d}^k-\bx^k)\right)$
		\EndFor
		\Ensure $ \bx^N $.
	\end{algorithmic}
\end{algorithm}
\vspace{-1.2em}
\begin{algorithm}[H]
	\caption{GraRED-P$ ^3 $}
	\label{alg: GraRED-P3}
	\begin{algorithmic}[1]
		\Require Initialization $ \bu^0 =(\bx^0,\by^0)\in\bbR^{n}\times \bbR^{n} $, iteration number $ N>0 $, relaxing parameter $ \lambda_k\in [0,2] $, and $ R = I-D_\sigma $ is the residual.
		\For{$ k=0, 1,2,\ldots, N-1 $}
		\State $ \mathbf{d}^{k} = \mathrm{prox}_{\lambda f}(\bx^{k}-\tau\by^{k})$
		\State $ \bx^{k+1} = \lambda_k \mathbf{d}^{k}+(1-\lambda_k)\mathbf{x}^k$
		\State $ \by^{k+1} = \lambda_k R(\by^k+s(2\mathbf{d}^k-\bx^k)) + (1-\lambda_k) \by^k$
		\EndFor
		\Ensure $ \bx^N $.
	\end{algorithmic}
\end{algorithm}
\begin{algorithm}
    \caption{Relaxed RED-PRO via SD}
    \label{alg:relaxed-red-pro}
    \begin{algorithmic}[1]
        \State \textbf{Input:} $\mathbf{x}^0 \in \mathbb{R}^n$, $\{t_k\}_{k \in \mathbb{N}}$, $\alpha \in (0, 1)$, $\lambda$, $\mu$, $J$, $N,\delta > 0$, and the denoiser $T$.
        \For{$k = 0, 1, 2, \ldots, N - 1$}
            \State $\mathbf{x}^{k,0} = \mathbf{x}^k$
            \For{$j = 0, 1, 2, \ldots, J - 1$}
                \State $\mathbf{x}^{k,j+1} = T_\alpha \left( t_j \mathbf{x}^k + (1 - t_j) \mathbf{x}^{k,j} \right)$
			\EndFor
			\State $\mathbf{v}^k = \frac{\delta}{\|\mathbf{x}^k - \mathbf{x}^{k,J}\|} \mathbf{x}^k + \left( 1 - \frac{\delta}{\|\mathbf{x}^k - \mathbf{x}^{k,J}\|} \right) \mathbf{x}^{k,J}$
			\State $\mathbf{x}^{k+1} = \mathbf{x}^k - \mu \left( \nabla f(\mathbf{x}^k) + \lambda (\mathbf{x}^k - \mathbf{v}^k) \right)$
        \EndFor
        \State \textbf{Output:} $\mathbf{x}^{N}$
    \end{algorithmic}
\end{algorithm}
\subsubsection{Compared with RED-PRO}
RED-PRO~\cite{cohen2021regularization} is defined by
 \begin{equation}
	\min_{\mathbf{x}\in \mathrm{Fix}(T)} \frac{1}{2\sigma^2} \left\| \mathbf{A}\mathbf{x}-\mathbf{y} \right\|^2.
 \end{equation}
RED-PRO sheds new insights on the denoiser prior, which theoretically bridges the gap between RED and the PnP prior. However, practical denoisers often have narrow fixed-point sets, leading to suboptimal recovery solutions. To address this issue, they relax the hard constraint of $\mathrm{Fix}(T)$ and replace it with a dilated fixed-point set, defined for some $\delta >0$ as
\begin{equation*}
	B_\delta (T) = \left\{ \mathbf{x}\in \bbR^n: \left\| \mathbf{x}-P_{\mathrm{Fix}(T)}(\mathbf{x}) \right\|\leq \delta \right\},
\end{equation*}
where $P_{\mathrm{Fix}(T)}(\mathbf{x})$ is the fixed-point projection of $\mathbf{x}$ onto $\mathrm{Fix}(T)$. The relaxed RED-PRO (RRP) problem is formulated as
\begin{equation}\label{eq: RRP}
	\min_{\mathbf{x}\in \mathbb{R}^n} \frac{1}{2\sigma^2} \left\| \mathbf{A}\mathbf{x}-\mathbf{y} \right\|^2+\frac{\lambda}{2} \left\| \mathbf{x}-P_{B_\delta (T)}(\mathbf{x}) \right\|^2.
\end{equation}
Algorithm~\ref{alg:relaxed-red-pro} outlines RRP, where $T_\alpha = \alpha T+(1-\alpha) I$ with the demicontractive denoiser $T$. When the steepest descent method solves~\eqref{eq: RRP}, it requires an inner loop to calculate the fixed-point projection $P_{\mathrm{Fix}(T)}(\mathbf{x}^k)$, and line 5 in RRP is equivalent to the Halpern iteration.

Comparing Algorithm~\ref{alg: GraRED-HP3} with the inner loop of Algorithm~\ref{alg:relaxed-red-pro}, we see that both methods share the same algorithmic form of Halpern iteration and utilize identical anchor coefficients. In both cases, finding $P_{\mathrm{Fix}(T)}(\mathbf{x}^k)$ or $P_{\mathrm{Fix}(\mathcal{T})}^\mathcal{M}(\mathbf{a})$ corresponds to a bilevel optimization problem that finds the point in the fixed-point set closest to the anchor point $\mathbf{x}^k$ or $\mathbf{a}$.

The key difference is that HPPP extends Halpern iteration to a degenerate form by employing a positive semidefinite metric $\mathcal{M}$, whereas the denoiser operator $T$ in RED-PRO is defined under the standard metric $\mathcal{M} = I$. From an optimization perspective, the RRP algorithm is based on gradient descent and uses the classic Halpern iteration in its inner loop. In contrast, $\mathcal{T}$ in GraRED-HP$^3$ is an algorithmic operator and has no inner loop. From a PPP standpoint, we integrate HPPP with denoiser priors to propose Algorithm~\ref{alg: GraRED-HP3}.
\section{Experiments}
\label{sec:experiments}
In this section, we show the numerical experiments of the algorithms discussed in section 3. First, we will verify the implicit regularity of HPPP by an easy 1D example, and validate the accelerated advantage of HPPP and restarted HPPP by other toy examples. Then, we compare CP, PPP~\eqref{eq: PPP},  GraRED-HP$^3$, and GraRED-P$^3$ for image deblurring and inpainting under the same setting and verify the efficiency of the proposed algorithms. The source code is available at \href{https://github.com/zsc15/HPPP}{https://github.com/zsc15/HPPP}.
\subsection{A toy example}
We consider the optimization problem in $ \mathbb{R} $, i.e.,
\begin{equation}
\min_{x\in \mathbb{R}} f(x)+g(x),
\end{equation}
where $ f(x) =\max\{-x, 0\} $ and $ g(x) = \max\{1-x, 0\} $, and we plot $f(x)+g(x)$ in Figure~\ref{fig: f+g}. 
\begin{figure}[htbp]
\centering
\begin{tikzpicture}
	\begin{axis}[
		axis lines = middle,
		xlabel = $x$,
		ylabel = {$y$},
		samples = 100,
		domain = -2:2,
		ymin = 0, ymax = 2,
		xmin = -2, xmax = 2,
		width=8cm, height=4.944cm,
		grid = both,
		grid style = {line width=.1pt, draw=gray!10},
		major grid style = {line width=.2pt,draw=gray!50},
		minor tick num = 5,
		enlargelimits={abs=0.5},
	]
		\addplot[blue, thick] {max(-x,0) + max(1-x,0)};
		\node at (axis cs:0.2,1.5) [anchor=north west] {$f(x) + g(x)$};
	\end{axis}
  \end{tikzpicture}
\caption{The image of $ f(x)+g(x) $.}
\label{fig: f+g}
\end{figure}
The corresponding saddle-point problem is $\min_{x\in \mathbb{R}}\max_{y\in \mathbb{R}}\{xy+f(x)-g^*(y)\}$, where $ g^*(y) =\max _x[y x-\max \{1-x, 0\}] = y+\delta_{[-1,0]}(y)$ and $\delta_{[-1,0]}$ is the indicator function of the interval $[-1,0]$. We denote the optimal set $ X^* = [1,+\infty) = \arg\min_{x\mathbb{\in \mathbb{R}}} f(x)+g(x) $ and the primal-dual objective function $F(x,y) = xy+f(x)-g^*(y)$. Let us solve the saddle-point set $\{ (x^*,y^*):F(x^*, y)\leq F(x^*, y^*)\leq F(x, y^*)\text{ for all }(x, y)\in \mathbb{R}^2 \}$ of $F(x,y)$. Fix $x^*\geq 1$; then
\begin{align*}
\max_{-1\leq y\leq 0}\{ x^*y-g^*(y) \} = \max_{-1\leq y\leq 0} \{ x^*y-y \}
=\left\{\begin{array}{ll}
0, &x^*>1, y=0,\\
0, &x^*=1, y\in [-1,0].
\end{array}\right.
\end{align*}
If $x^* = 1$, assume that $-1\leq y^*< 0$, then $F(1,y^*)=0 $, while $F(x,y^*) = y^*(x-1)+\max\{ -x,0 \}$, and there exists $x=2$ such that $F(2,y^*)=y^*<F(1,y^*)$, which leads to a contradiction. Therefore, the saddle-point set is $ \Omega = \{ (x^*, y^*):x^*\geq 1,y^*=0 \}$.

When $\mathcal{M} = \begin{pmatrix} 1 & -1 \\ -1 & 1 \end{pmatrix}$, as shown in Figure~\ref{fig: Projection onto the saddle point set}, the minimization $\arg\min_{\bu \in \mathrm{Fix}(\mathcal{T})} \|\bu-\mathbf{a}\|_\mathcal{M}^2 = \arg\min_{(x,y)\in \Omega}(x-x_a-(y-y_a))^2$ can be interpreted as finding the $\mathcal{M}$-projection, namely, the projection of the anchor point $\mathbf{a}=(x_a, y_a)$ onto the line $x-y=x_a-y_a$. Therefore, the $\mathcal{M}$-projection of $\mathbf{a}$ onto $\Omega$ is uniquely determined,
\begin{equation}\label{eq: projection onto the saddle point set}
P_{\Omega}^\mathcal{M} (\mathbf{a}) =\begin{cases}
(1,0), x_a-y_a-1\leq 0,\\
(x_a-y_a,0), x_a-y_a-1> 0.
\end{cases}
\end{equation}
\begin{figure}[htbp]
\centering
\begin{tikzpicture}[scale=.8]
	\fill[blue!20] (-0.8,-1.6) -- (5,4) -- (-0.8,4) -- cycle;
	\draw[blue, thick] (-0.8,-1.6) -- (5,4) node[pos=0.25, above right, xshift=-15, yshift= 50] {$x - y - 1 \leq 0$};
    \draw[->] (-0.6,0) -- (4.6,0) node[right] {$x$};
    \draw[->] (0,-1.5) -- (0,4.5) node[above] {$y$};

    

    \fill[black] (0,0) circle (1pt) node[below left] {$O$};
    \fill[red] (2,0) circle (1pt) node[below] {$S(2,0)$};
    \fill[blue] (3,1) circle (1pt) node[right] {$R(3,1)$};

	\draw[gray, thick] (2,0) -- (3,1);
\end{tikzpicture}
\caption{$ \mathcal{M}$-projection onto the saddle-point set.}
\label{fig: Projection onto the saddle point set}
\end{figure}

We set the total iteration number $N=1000$. As shown in Figure~\ref{subfig: a}, with the anchor point fixed at $\mathbf{a} = (12,9)$, the sequence $\mathbf{u}^k = (x^k,y^k)$ generated by HPPP converges to the same point $\mathbf{u}^* = (3,0)$, regardless of the initial point, which verifies Proposition~\ref{cor:bigthm cor} and~\eqref{eq: projection onto the saddle point set}. HPPP is also an implicitly regularized method. In contrast, Figure~\ref{subfig: d} shows that the sequence generated by PPP oscillates around the limit $(1.8,0)$. However, a quantitative mathematical characterization of the limit of the PPP sequence remains elusive.
\begin{figure}[htbp]
\centering
\begin{subfigure}[b]{.4\linewidth}
\centering
\includegraphics[height=5cm]{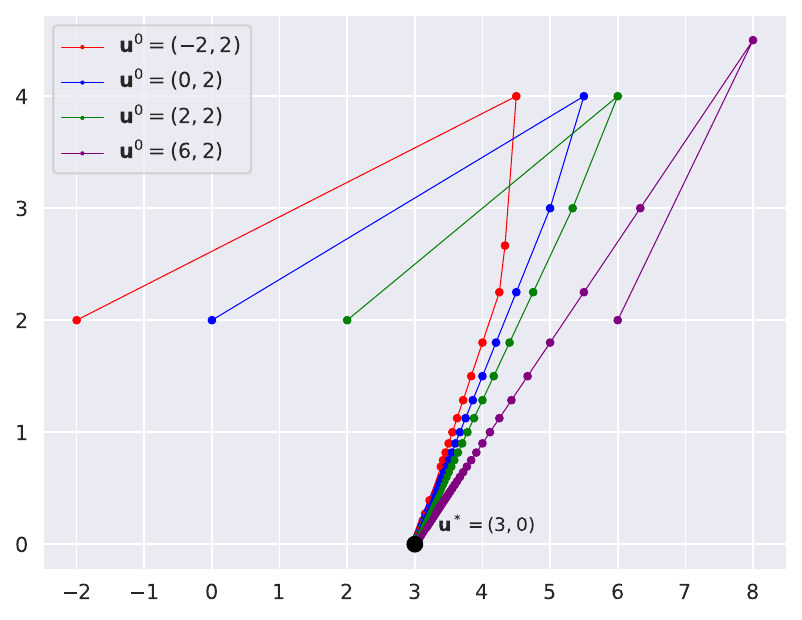}
\caption{}
\label{subfig: a}
\end{subfigure}
\begin{subfigure}[b]{.4\linewidth}
	\centering
	\includegraphics[height=5cm]{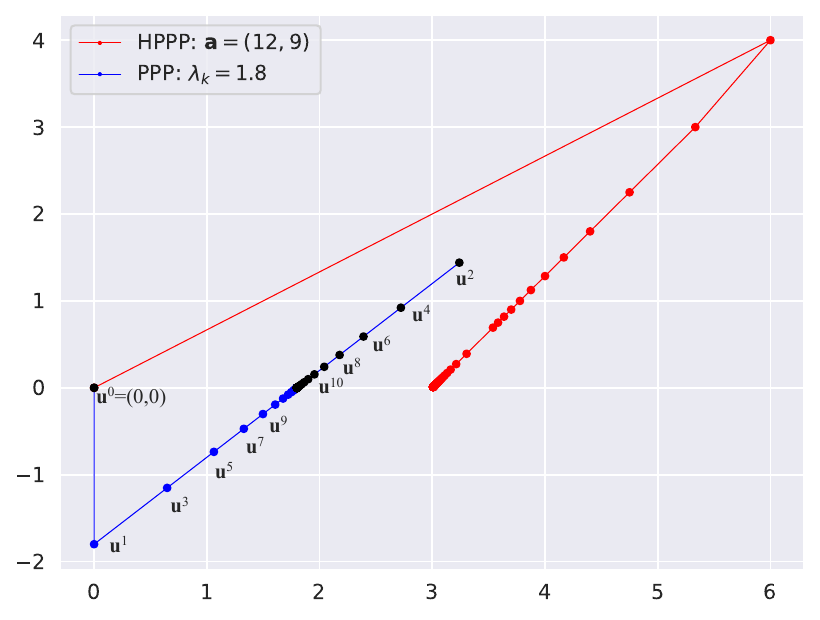}
	\caption{}
	\label{subfig: d}
\end{subfigure}
\caption{Trajectory of $ \bu^k = (x^k,y^k) $ generated by HPPP or PPP.}
\label{fig: regularity}
\end{figure}
\subsection{A 2D toy example}
We consider a $\frac{1}{\gamma}$-contractive operator $T_\theta: \mathbb{R}^2\to \mathbb{R}^2$ from~\cite{Park2022} for any $\mathbf{x} = (x_1,x_2)^{\mathrm{T}}\in \mathbb{R}^2$, 
\begin{equation*}
	T_\theta \mathbf{x} = \frac{1}{\gamma }\begin{pmatrix}
		\cos\theta & -\sin\theta \\
		\sin\theta & \cos\theta
	\end{pmatrix} \mathbf{x}.
\end{equation*}  
We compare the Picard iteration $\mathbf{x}^{k+1} = T_\theta \mathbf{x}^k$, PPP ($\lambda_k =0.1$), HPPP ($\mu_k =\frac{1}{k+1}$), and restarted HPPP ($q=50$). Set $\mathcal{A} =\begin{pmatrix}
	-T_\theta & I \\
	-I & I
\end{pmatrix}$ and $\mathcal{M} = \begin{pmatrix}
	\frac{1}{\tau}I & -I \\
	-I & \frac{1}{s}I
\end{pmatrix}
$ with $\tau = 0.8, s =1.25$, and initial points $\mathbf{x}^0 = \mathbf{y}^0 = \mathbf{x}_a = \mathbf{y}_a = (1,0)^{\mathrm{T}}$, where $\mathbf{a} = \mathbf{u}^0$. Figure~\ref{fig: 2d example} shows that HPPP and restarted HPPP indeed provide accelerated residual decay.
\begin{figure}[htbp]
	\centering
	\begin{subfigure}[b]{.45\linewidth}
	\centering
	\includegraphics[height=5.5cm]{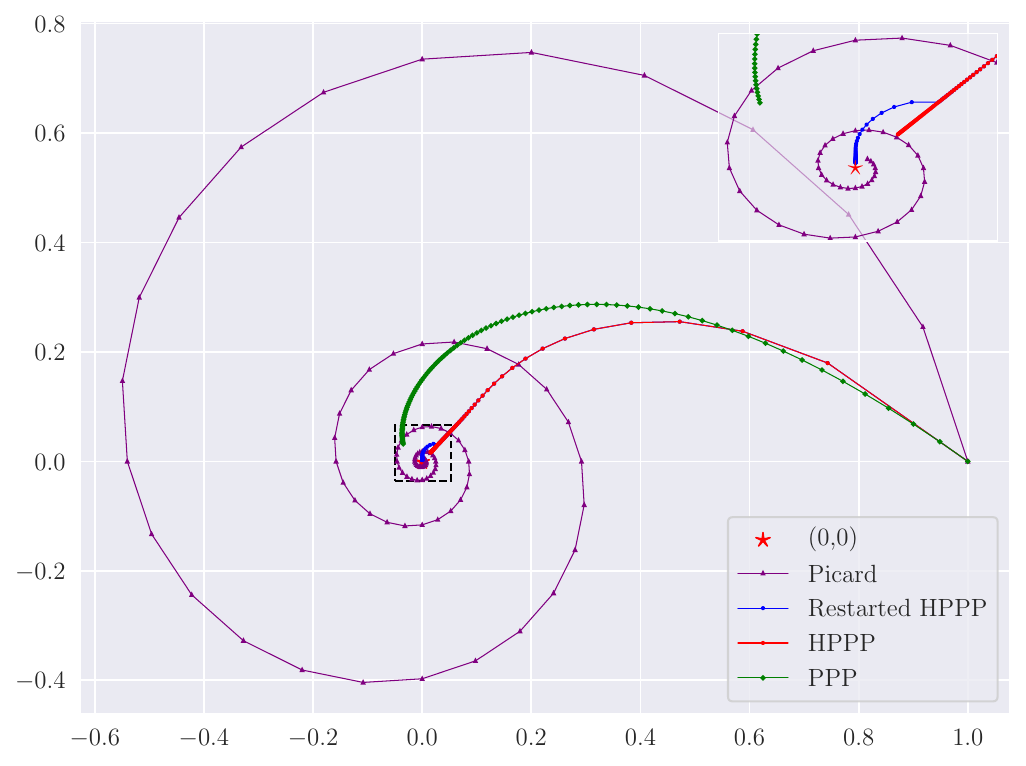}
	\caption{}
	\label{subfig: 2d_Trajectories}
	\end{subfigure}
	\begin{subfigure}[b]{.45\linewidth}
		\centering
		\includegraphics[height=5.5cm]{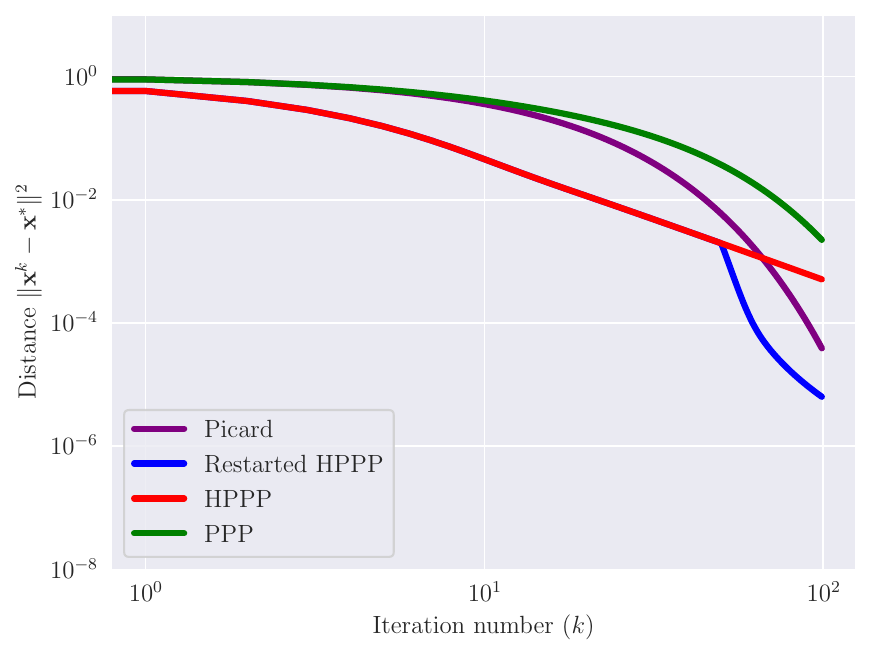}
		\caption{}
		\label{subfig: 2ddists}
	\end{subfigure}
	\begin{subfigure}[b]{.45\linewidth}
		\centering
		\includegraphics[height=5.5cm]{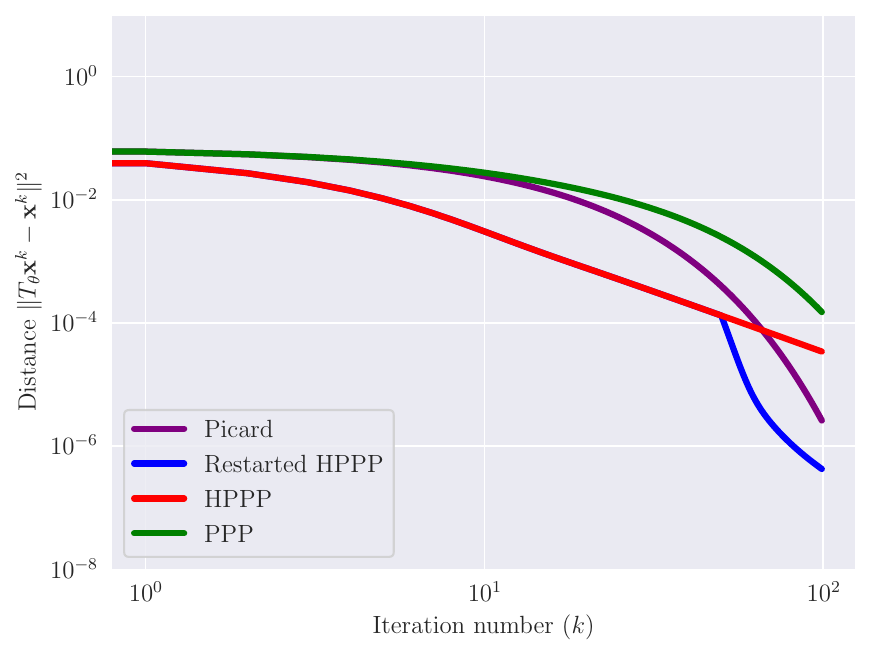}
		\caption{}
		\label{subfig: 2dres}
	\end{subfigure}
	\begin{subfigure}[b]{.45\linewidth}
		\centering
		\includegraphics[height=5.5cm]{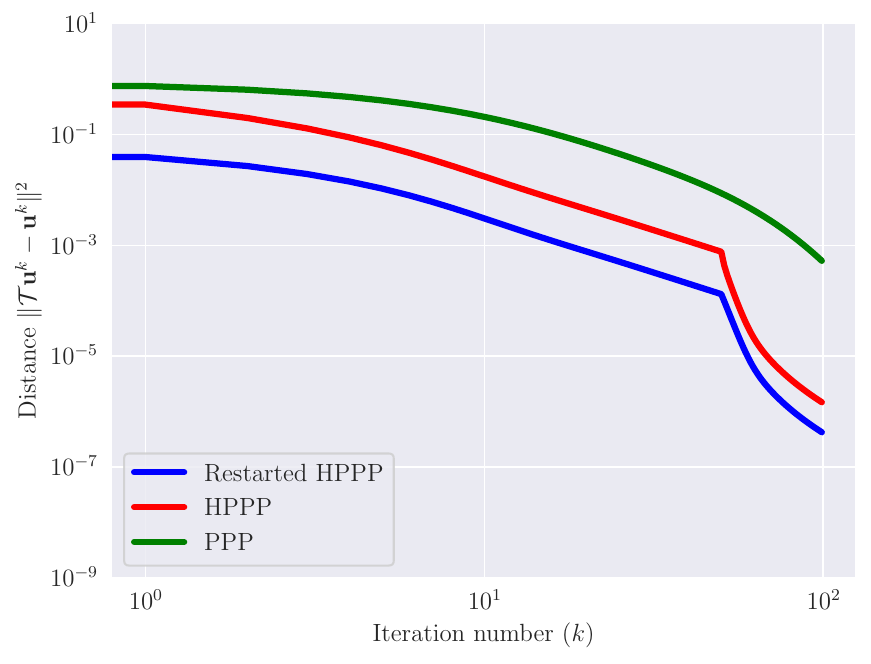}
		\caption{}
		\label{subfig: 2dres1}
	\end{subfigure}
	\caption{Trajectories, distance to solution, and two residuals of iterates for the 2D toy example. Here $\gamma = 1/0.95, \theta =15^\circ, N=100$.}
	\label{fig: 2d example}
	\end{figure}
\subsection{An infinite-dimensional example}
We consider an infinite-dimensional problem on $\ell^2$ from~\cite{bc24}, the space of square summable sequences. For any $\mathbf{x} =(x_0,x_1,x_2,\ldots)\in \ell^2$, the norm is  $\left\| \mathbf{x} \right\| =\left(\sum_{k\in \mathbb{N}} \left| x_i \right|^2  \right)^\frac{1}{2} $. We consider the following operator inclusion problem:
\begin{equation}\label{Infinite Dimensional Example}
	\mathbf{0} \in \mathcal{A}_1 \mathbf{x}+\mathcal{A}_2 \mathbf{x},  \mathcal{A}_1 = -\mathcal{S}-\mathbf{b}, \mathcal{A}_2 =I,
\end{equation}
where $\mathcal{S}(\mathbf{x}) = (0,x_0,x_1,\ldots)$ is the right-shift operator and $\mathbf{b} =(1,-1,0,\ldots)$. Since $\mathcal{A}_1,\mathcal{A}_2$ are nonexpansive, $\mathcal{A}_1,\mathcal{A}_2$ are maximal monotone~\cite[Example 20.29]{BCombettes}. Additionally, by solving the above inclusion problem~\eqref{Infinite Dimensional Example}, we obtain $\mathbf{x}^* = (1,0,\ldots)$. Using $\mathcal{A} =\begin{pmatrix}
	\mathcal{A}_1 & I \\
	-I & I
\end{pmatrix}$ and positive semidefinite preconditioner $\mathcal{M} = \begin{pmatrix}
	2I & -I \\
	-I & \frac{1}{2}I
\end{pmatrix}
$, we compare the PPP iteration with $\lambda_k = 0.5,1,1.8$. We set $N = 10^3,\tau =0.5,s =2$ and initialize $\mathbf{x}^ 0 =(0.9,1,0,\ldots)$ for HPPP and PPP, and the anchor point is the same as the initial point, i.e., $\mathbf{a} = (\mathbf{x}^0,\mathbf{x}^0)$. For this example, we present visual results about the residual and the distance to the solution. As illustrated in Figure~\ref{fig:Infinite Dimensional Example}, HPPP outperforms PPP in terms of both  $\left\| \mathbf{x}^k-\mathbf{x}^* \right\| $ and $\left\| \mathbf{u}^{k+1}-\mathbf{u}^k \right\| $.
\begin{figure}[htbp]
    \centering
    \begin{minipage}[b]{.49\linewidth}
        \centering
        \includegraphics[height=5.4cm]{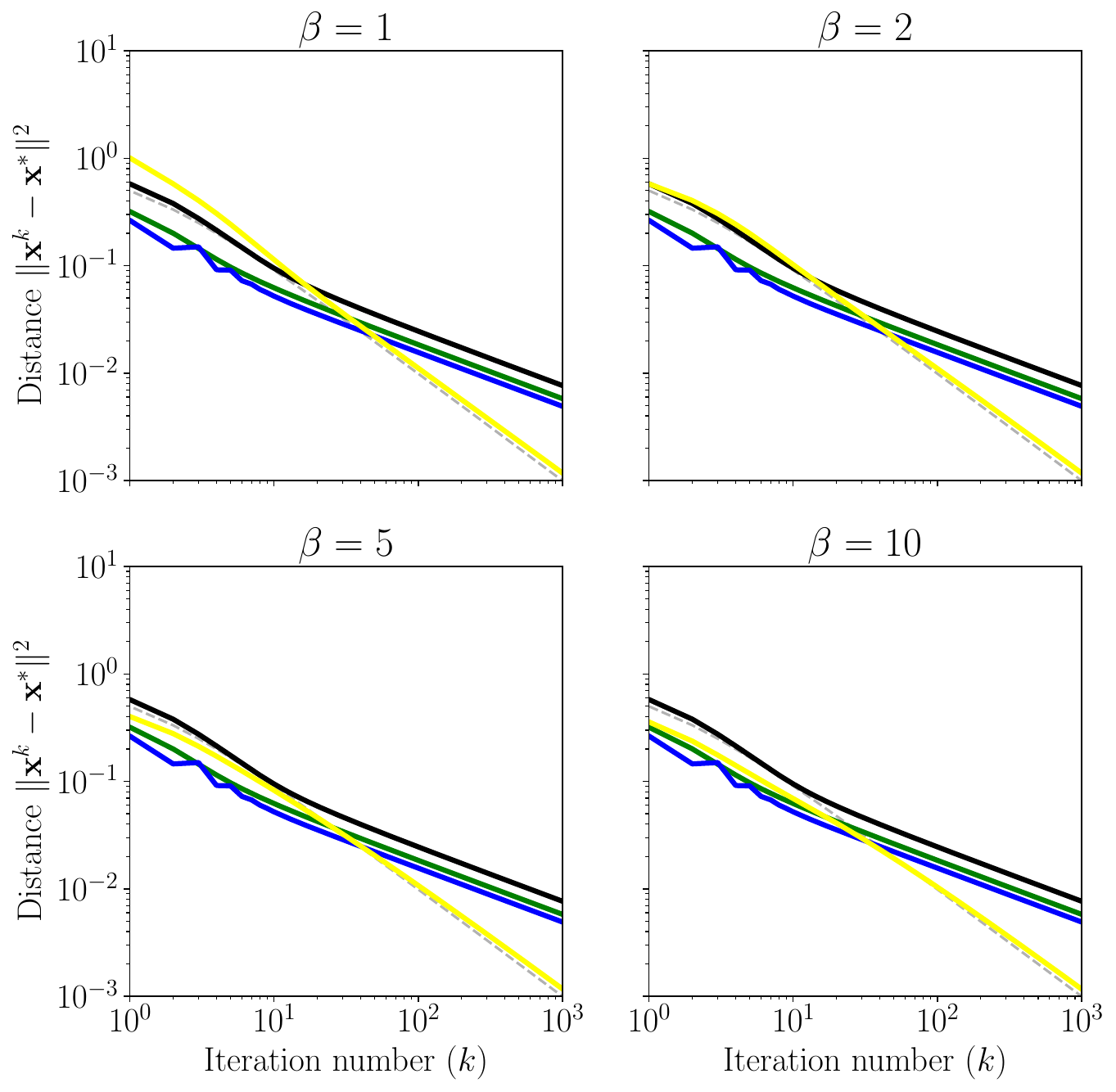}
        \subcaption{Distance to solution}
        \label{fig: experiment_HPPP_PPP_dists_tau_0.5}
    \end{minipage}
	\hspace{0.0001\linewidth} 
    \begin{minipage}[b]{.49\linewidth}
        \centering
        \includegraphics[height=5.4cm]{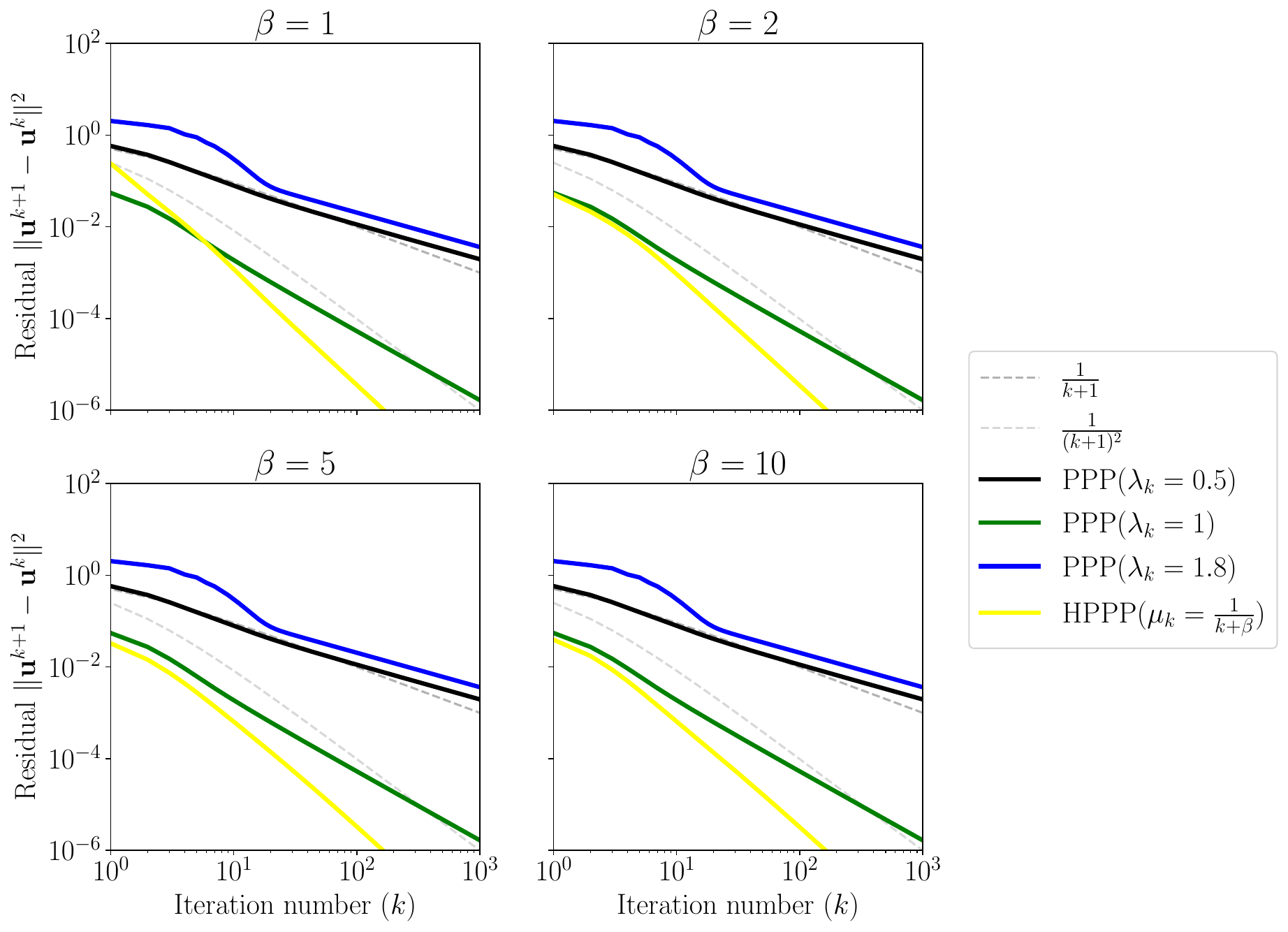}
        \subcaption{Residual comparison}
        \label{fig: experiment_HPPP_PPP_resid_tau_0.5}
    \end{minipage}
	\caption{Comparison with PPP to solve an infinite-dimensional problem in $\ell^2$ space.}
	\label{fig:Infinite Dimensional Example}
\end{figure}
\subsection{Image deblurring}
First, we compared CP~\cite{Chambolle2011}, PPP~\cite{Bredies2022}, and HPPP for TV-regularized IR problems. For TV-regularized IR problems, 
\begin{equation}\label{problem: TV deblurring}
\min_{\mathbf{x}\in \mathbb{R}^n} \frac{\lambda}{2} \left\| \bA \bx-\mathbf{y} \right\|_2^2+\beta \left\| \nabla \mathbf{x} \right\|_1,
\end{equation}
where $ \mathbf{y} $ is the degraded image, $ \bA $ is a linear operator, and $\lambda, \beta $ are balance coefficients.

We use a 2D Gaussian function with a standard deviation of 1.6  to convolve 10 test grayscale images, and finally obtain the degraded images with an additive white Gaussian noise  with the noise level $ \sigma =2.55 $.  Firstly, we compared three algorithms, CP, PPP, and the proposed HPPP. All the algorithms use the degraded images as initial points. We  calculate the norm $ K =\left\| \nabla \right\| = 1.75 $, and choose the total iteration $ N=400 $, balance coefficients $ \lambda =2,\beta = 5\times 10^{-4} $. Their parameters are given in Table~\ref{table: parameters of CP, PPP, and HPPP}. Both GraRED-P$ ^3 $  and GraRED-HP$ ^3 $ use DnCNN; other parameters are used as follows:
\begin{itemize}
	\item GraRED-P$ ^3 $ : $ \tau = 1, s =1, \lambda_k = 0.2, \lambda = 20$;
	\item GraRED-HP$ ^3 $ : $ \tau = 1, s =1, \mu_k = 1/(k+2), \lambda = 20, \bx_a = \by, \by_a = \mathbf{0}$.
\end{itemize}

Table~\ref{table: PSNR about CP, PPP and HPPP} shows PSNR (dB) of restoration results on CP, PPP, HPPP, GraRED-P$ ^3 $, and GraRED-HP$ ^3 $. The performance of these methods is evaluated using PSNR measure. The best recovery results are highlighted in bold. From Table~\ref{table: PSNR about CP, PPP and HPPP}, GraRED-P$ ^3 $ and GraRED-HP$ ^3 $ are better than classic algorithms with explicit TV regularization, which demonstrates that the implicit regularizer is more powerful for regularizing inverse imaging problems. We visualize the numerical comparison between GraRED-P$ ^3 $, GraRED-HP$ ^3 $, CP, PPP, and HPPP in Figure~\ref{fig: comparison with classic TV regularization}. We further compare the robustness of the initial points between the proposed HPPP, CP, and PPP with TV regularization.  As shown in Figure~\ref{fig: Chuzhi infulence}, we plot their respective evolutions of PSNR values for iterations for \textbf{Butterfly} and \textbf{Parrots} with 10 different random initial points, here we choose the degraded image $\mathbf{y}$ as the anchor point. Once the anchor point is chosen, the proposed HPPP algorithm is robust to random perturbations of the initial points, which achieves stable recovery.  
\begin{table}[htbp]
	\centering
	\caption{Parameters on CP, PPP, and HPP with TV regularization for image deblurring.}
	\label{table: parameters of CP, PPP, and HPPP}
\scalebox{0.8}{\begin{tabular}{|c|c|}
	\hline
	CP & $ \tau = s =1/K = 0.57  $ \\
	\hline
	PPP & $ \tau = s =1/K = 0.57, \lambda_k = 1.95 $ or $ \lambda_k =1.2 $ \\
	\hline
	HPPP & $ \tau = s =1/K = 0.57, \mu_{k}=\frac{1}{k+2}, \bx_a = \bA^\mathrm{T}\mathbf{y}, \by_a = 0\cdot \nabla \bx_a  $\\
	\hline
\end{tabular}}
\end{table}
\begin{table}[htbp]
	\centering
	\caption{Deblurring results of grayscale images compared with CP, PPP, and HPPP about the Gauss blurring kernel.}
	\label{table: PSNR about CP, PPP and HPPP}
	\scalebox{0.8}{
		\begin{tabular}{|c|c|c|c|c|c|c|c|c|c|}
			\hline 
			& Cameraman & House & Pepper & Starfish & Butterfly & Craft & Parrots & Barbara & Boat \\
			\hline
			CP & 26.04 & 30.86 & 25.99 & 27.53 & 27.85 & 25.51 & 26.88 &  24.46 & 28.99 \\
			\hline
			PPP & {26.05} & 30.85 & 25.99 & 27.53 & 27.85& 25.51 & {26.88}  & 24.46& 28.98\\
			\hline
			HPPP & 26.00 & {31.39} &{26.05} & {27.65} & {27.99} & {25.51} & 26.82  & {24.51} & {29.09}\\
			\hline
			GraRED-P$ ^3 $ & \textbf{26.87} & {32.51} & {28.34} & \textbf{29.16} & \textbf{29.70} & {26.90} & \textbf{27.95} & {24.66} & {29.68}\\
			\hline
			GraRED-HP$ ^3 $ & 26.85 & \textbf{32.51} & \textbf{28.80} & {29.15} & {29.67} & \textbf{27.03} & 27.89 &  \textbf{24.66} & \textbf{29.99}\\
			\hline
	\end{tabular}}
\end{table}
\begin{figure}[htbp]
	\centering
	\captionsetup[subfigure]{justification=centering}
	\begin{subfigure}[b]{.22\linewidth}
		\centering
		\begin{tikzpicture}[spy using outlines={rectangle,blue,magnification=5,size=1.2cm, connect spies}]
		\node {\includegraphics[height=3cm]{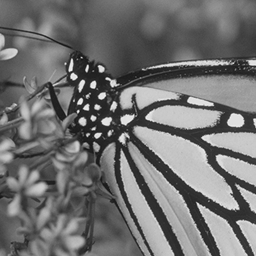}};
		\spy on (0.25,-0.4) in node [left] at (1.5,.9);
		\end{tikzpicture}
		\caption{Clean image  \\ ~}
	\end{subfigure}
	\begin{subfigure}[b]{.22\linewidth}
		\centering
		\begin{tikzpicture}[spy using outlines={rectangle,blue,magnification=5,size=1.2cm, connect spies}]
		\node {\includegraphics[height=3cm]{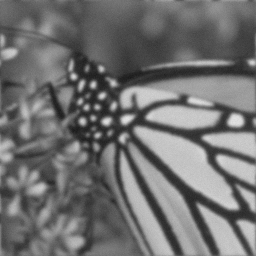}};
		\spy on (0.25,-0.4) in node [left] at (1.5,.9);
		\end{tikzpicture}
		\caption{Degraded image \\ ~}
	\end{subfigure}
	\begin{subfigure}[b]{.22\linewidth}
		\centering
		\begin{tikzpicture}[spy using outlines={rectangle,blue,magnification=5,size=1.2cm, connect spies}]
		\node {\includegraphics[height=3cm]{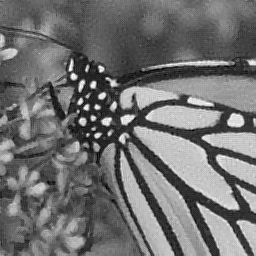}};
		\spy on (0.25,-0.4) in node [left] at (1.5,.9);
		\end{tikzpicture}
		\caption{CP ($27.85$dB) \\ ~}
	\end{subfigure}
	\begin{subfigure}[b]{.22\linewidth}
		\centering
		\begin{tikzpicture}[spy using outlines={rectangle,blue,magnification=5,size=1.2cm, connect spies}]
		\node {\includegraphics[height=3cm]{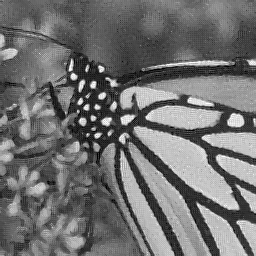}};
		\spy on (0.25,-0.4) in node [left] at  (1.5,.9);
		\end{tikzpicture}
		\caption{$ \lambda_k =1.95 $: PPP ($27.85$dB)}
	\end{subfigure}
	\begin{subfigure}[b]{.22\linewidth}
		\centering
		\begin{tikzpicture}[spy using outlines={rectangle,blue,magnification=5,size=1.2cm, connect spies}]
		\node {\includegraphics[height=3cm]{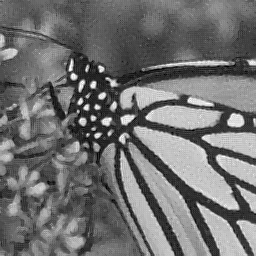}};
		\spy on (0.25,-0.4) in node [left] at (1.5,.9);
		\end{tikzpicture}
		\caption{$ \lambda_k =1.2 $: PPP ($27.85$dB)}
	\end{subfigure}
	\begin{subfigure}[b]{.22\linewidth}
		\centering
		\begin{tikzpicture}[spy using outlines={rectangle,blue,magnification=5,size=1.2cm, connect spies}]
		\node {\includegraphics[height=3cm]{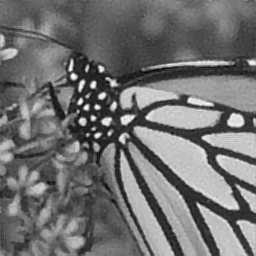}};
		\spy on (0.25,-0.4) in node [left] at  (1.5,.9);
		\end{tikzpicture}
		\caption{HPPP ($27.99$dB) \\ ~}
	\end{subfigure}
	\begin{subfigure}[b]{.22\linewidth}
		\centering
		\begin{tikzpicture}[spy using outlines={rectangle,blue,magnification=5,size=1.2cm, connect spies}]
		\node {\includegraphics[height=3cm]{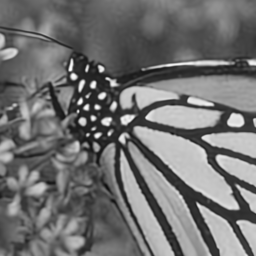}};
		\spy on (0.25,-0.4) in node [left] at  (1.5,.9);
		\end{tikzpicture}
		\caption{GraRED-P$ ^3 $ ($29.70$dB)}
	\end{subfigure}
	\begin{subfigure}[b]{.22\linewidth}
		\centering
		\begin{tikzpicture}[spy using outlines={rectangle,blue,magnification=5,size=1.2cm, connect spies}]
		\node {\includegraphics[height=3cm]{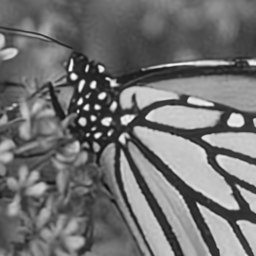}};
		\spy on (0.25,-0.4) in node [left] at  (1.5,.9);
		\end{tikzpicture}
		\caption{GraRED-HP$ ^3 $ ($29.67$dB)}
	\end{subfigure}
	\caption{Deblurring  results of \textbf{Butterfly} degraded by the Gaussian blur kernel with noise level $ \sigma = 2.55 $.}
	\label{fig: comparison with classic TV regularization}
\end{figure}
\begin{figure}[htbp]
    \centering
    \begin{subfigure}[b]{.45\linewidth}
        \centering
        \includegraphics[width=1\linewidth]{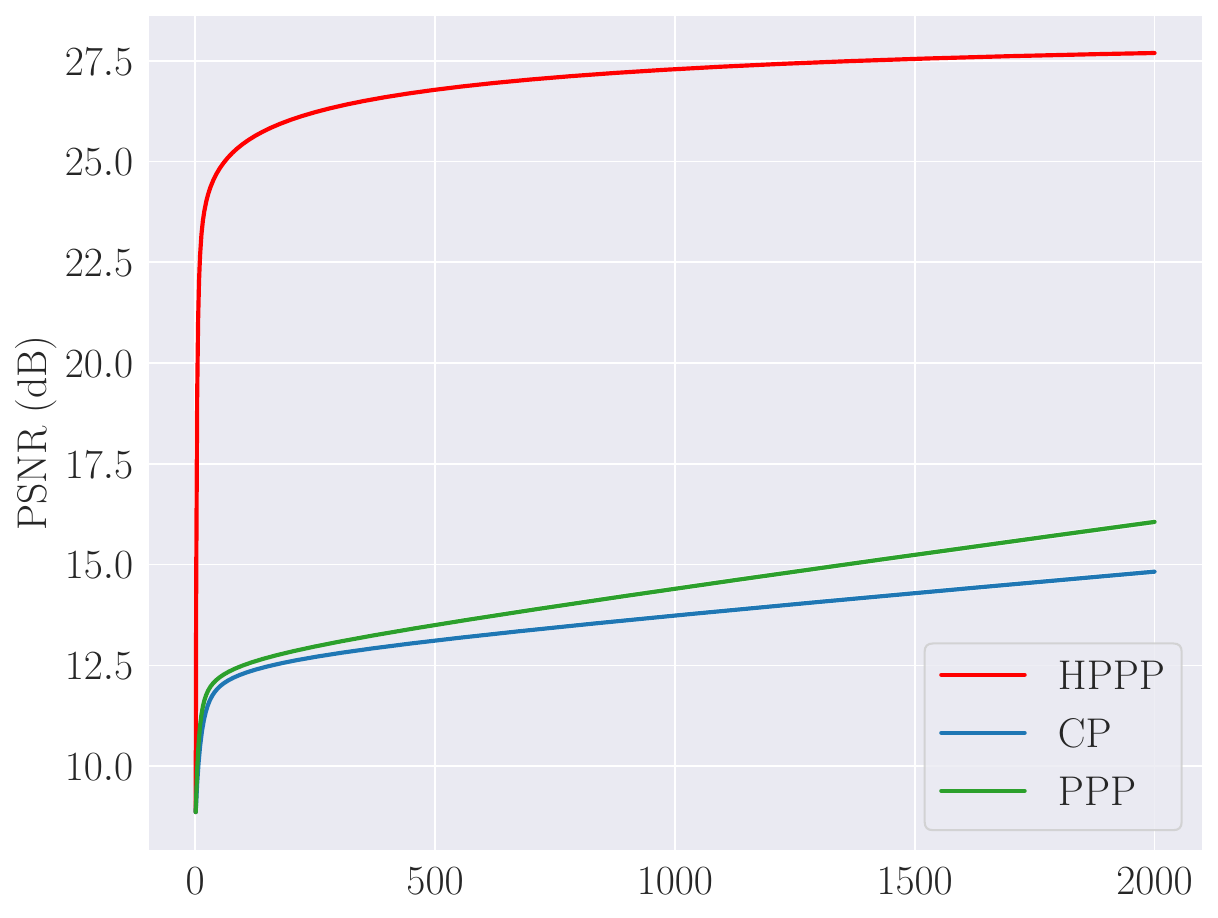}
        \caption{Butterfly}
        \label{fig: Butterfly_robustness_comparison}
    \end{subfigure}
    \begin{subfigure}[b]{.45\linewidth}
        \centering
        \includegraphics[width=1\linewidth]{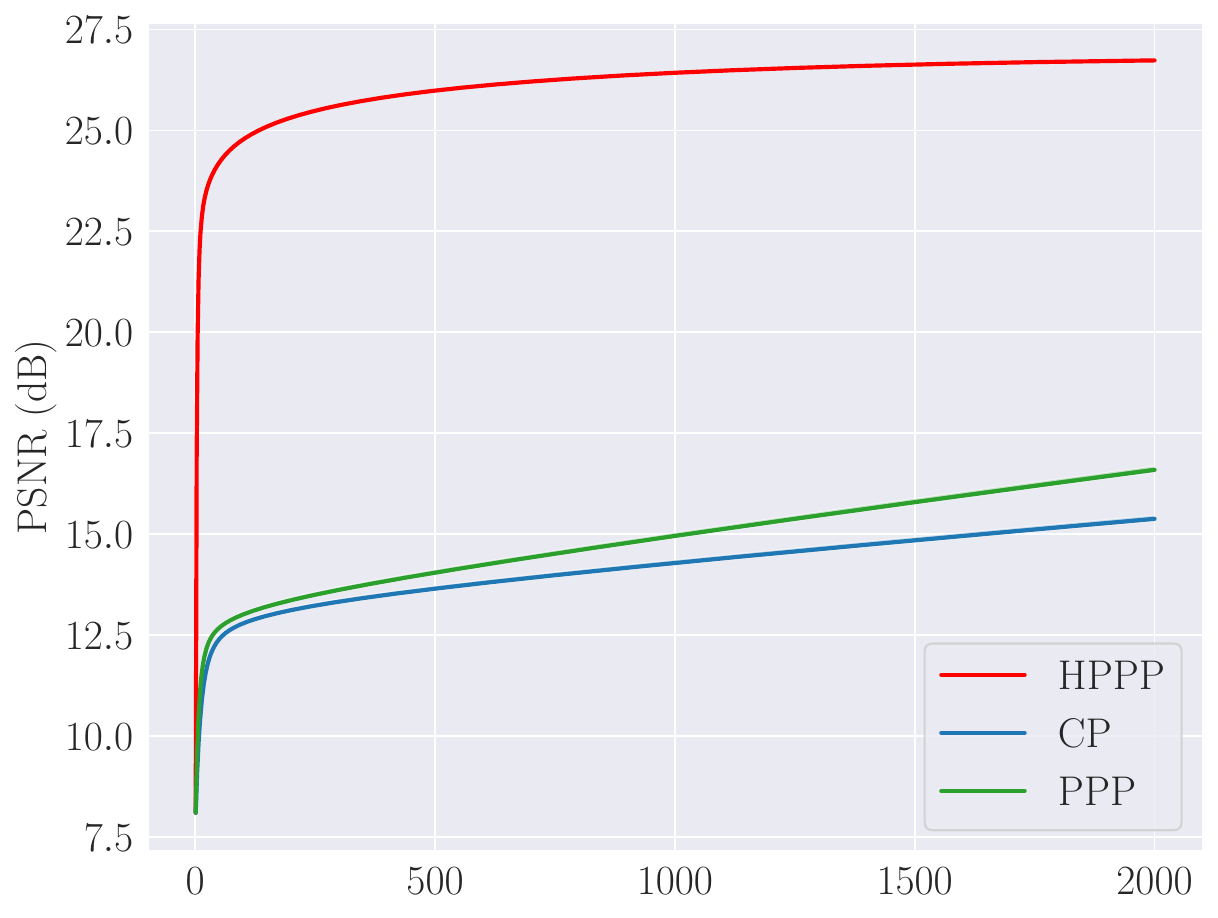}
        \caption{Parrots}
        \label{fig: Parrots_robustness_comparison}
    \end{subfigure}
    \caption{Robustness comparison with random initialization. Test images are degraded by the Gaussian blurring kernel with the noise level $ \sigma = 2.55 $.}
	\label{fig: Chuzhi infulence}
\end{figure}

Second, we compared RED and RED-PRO. Following~\cite{Romano2017,cohen2021regularization}, a $ 9\times 9 $ uniform point spread function (PSF) and a 2D Gaussian function
with a standard deviation of $ 1.6 $ are used to convolve test images. We finally obtained the degraded images with the noise level $ \sigma = \sqrt{2} $. The original RGB image is converted into the YCbCr image, all algorithms are applied to the luminance channel, and then the reconstruction image is returned to RGB space to obtain the final image. PSNR is measured on the luminance channel of the ground truth and the restored images. Table~\ref{table: RED, RED-PRO and ours} shows the PSNR values of the restoration results on RED, RED-PRO, RRP, GraRED-P$^3$, and GraRED-HP$^3$. From the deblurring experiment of Table~\ref{table: RED, RED-PRO and ours} and Figure~\ref{fig: comparison with RED and REDPRO}, GraRED-P$ ^3 $ and GraRED-HP$ ^3 $ achieve better performance than RED, RED-PRO, and RRP, which illustrates that KM or Halpern iteration used in PPP methods is effective. We further compare the differences among GraRED-P$^3$, GraRED-HP$^3$, and restarted HPPP under six different settings with $\lambda=40,\tau=s=1,\lambda_k=0.2,\mu_k=\frac{1}{k+1},N=500,q=100$. In Table~\ref{table:different tasks}, we report the average PSNR values of 15 test images using six different blur kernels. When the anchor point and the initial point coincide, GraRED-HP$^3$ achieves the best performance on some IR tasks. Moreover, by dynamically updating the anchor points, restarted HPPP further boosts the performance of GraRED-HP$^3$.
\begin{figure}[htbp]
    \centering
    \begin{subfigure}[b]{.45\linewidth}
        \centering
        \includegraphics[width=1\linewidth]{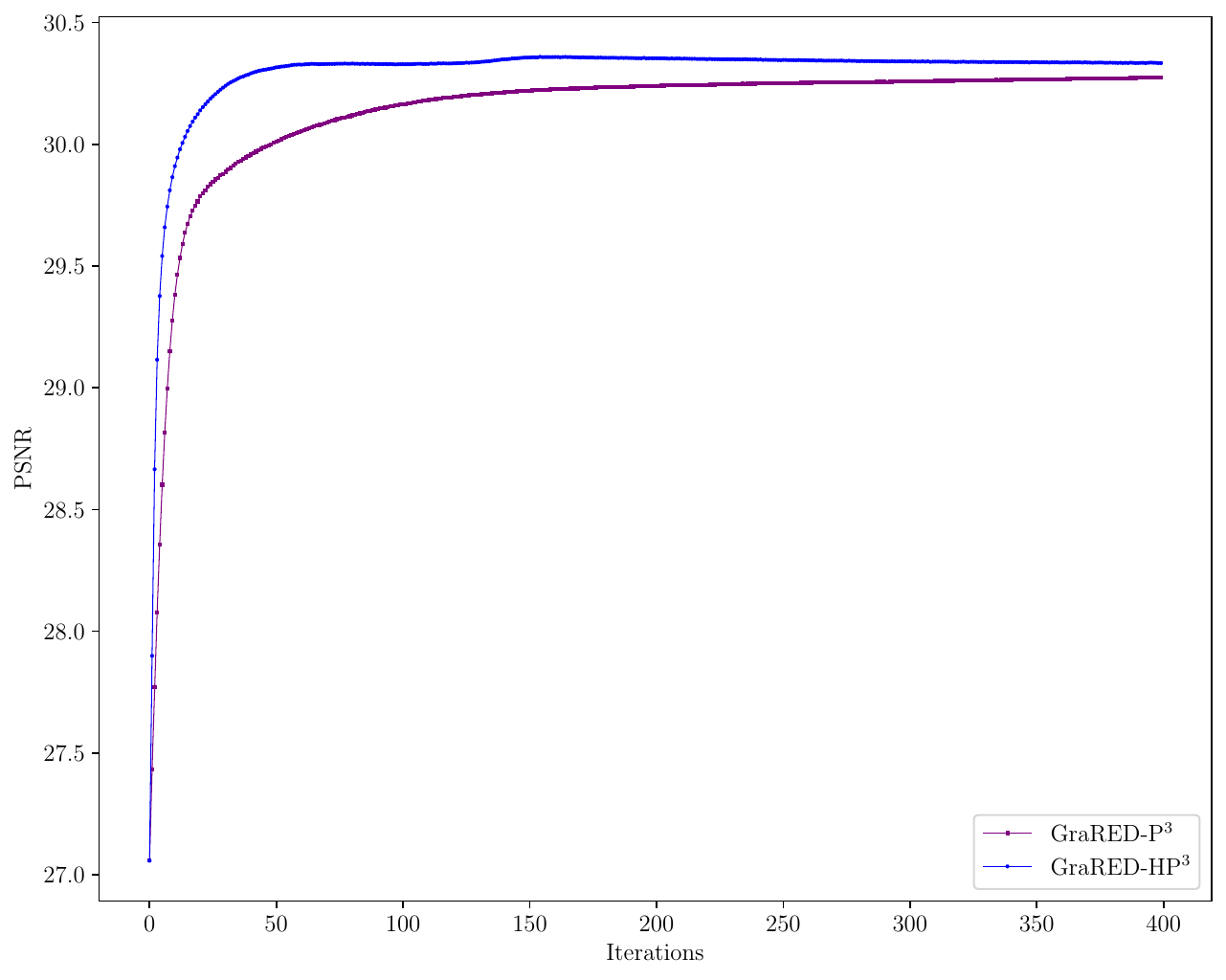}
        \caption{Recovery performance}
        \label{fig: PSNR_comparison}
    \end{subfigure}
    \begin{subfigure}[b]{.45\linewidth}
        \centering
        \includegraphics[width=1\linewidth]{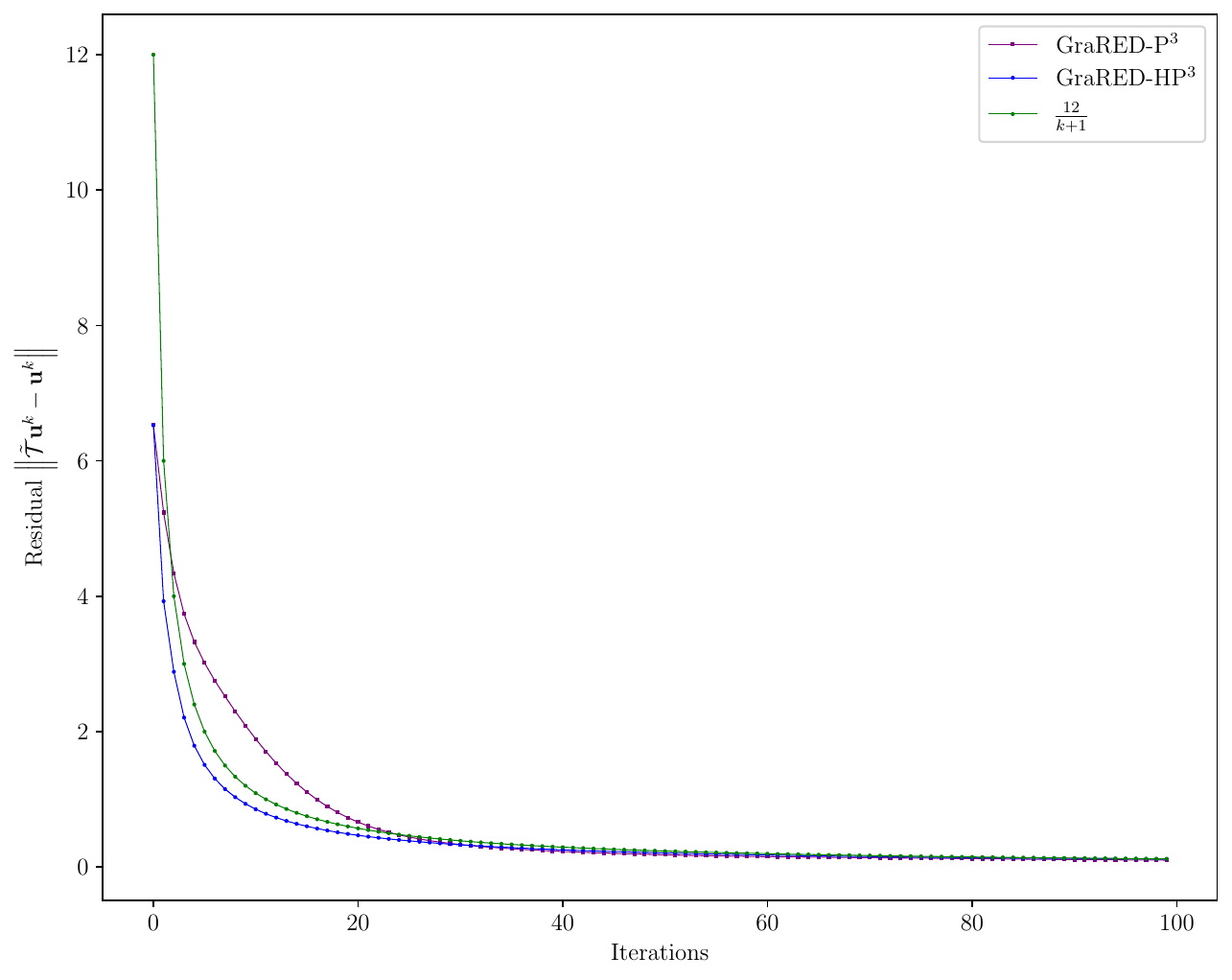}
        \caption{Fixed point residual}
        \label{fig: residual_comparison}
    \end{subfigure}
    \caption{Recovery performance and fixed-point residual comparison for the House image degraded by the Gaussian blurring kernel with the noise level $\sigma =5$. Here $D_1$ is used.}
	\label{fig: PSNR_and_residual_comparison}
\end{figure}

Additionally, we trained a FNE-DnCNN ($D_1$) using real spectral normalization~\cite{ryu2019plug}, which is employed in RED and RED-PRO. We also compared our results with state-of-the-art methods using advanced denoisers, such as DRUNet ($D_2$)~\cite{zhang2021plug} and GS ($D_3$)~\cite{hurault2022gradient,hurault2022proximal}. The PnP-DYS method~\cite{There-Operator-Splitting2024} uses the Davis-Yin splitting (DYS) algorithm. For GraRED-HP$^3$, the anchor point is the same as the initial point. The default parameters are presented in Table~\ref{table: parameters of GraRED-HPPP and restartHP3}. As shown in Table~\ref{table: comarison with state-of-the-art methods}, Restarted HPPP with advanced denoisers achieves competitive or even superior performance compared to state-of-the-art methods. In Figure~\ref{fig: PSNR_and_residual_comparison}, we compare  GraRED-P$^3$ ($\lambda_k = 0.1$) and GraRED-HP$^3$ ($\mathbf{a} = \mathbf{u}^0,\mu_k = \frac{1}{k+1}$). The results demonstrate that GraRED-HP$^3$ achieves better recovery performance and exhibits a faster fixed-point convergence rate than GraRED-P$^3$.
\begin{table}[!ht]
	\centering
	\caption{Recovery results are obtained by different algorithms with trainable nonlinear reaction-diffusion (TNRD) denoiser~\cite{Chen2017}. The best two results are highlighted in {\color{red}red} and {\color{blue}blue}, respectively.}
	\label{table: RED, RED-PRO and ours}
	\scalebox{0.8}{\begin{tabular}{|l|c|c|c|c|c|c|c|c|c|}
		\hline 
		\multirow{2}{*}{Algorithms}
		& \multicolumn{4}{c|}{Uniform kernel} & \multicolumn{4}{c|}{Gaussian kernel($ \sigma_k = 1.6 $)} \\
		\cline{2-9}
		& Bike & Butterfly & Flower & Hat  & Bike & Butterfly & Flower & Hat  \\
		\hline
		RED~\cite{Romano2017} & 26.10 &	{30.41} & 30.18 & 32.16	& 27.90 & {31.66} & 32.05 & 33.30 \\
		\hline
		RED-PRO(HSD)~\cite{cohen2021regularization}& 24.95 & 27.24 & 29.38 & 31.55 & 27.36 & 30.55 & 31.81 & 33.07  \\
		\hline 
		RRP~\cite{cohen2021regularization} & {26.48} & {30.64} & {30.46} & {32.25} & {28.02}& {{31.66}} &{32.08} & {33.26}  \\
		\hline
		GraRED-P$ ^3 $ & {\color{blue}{26.55}} & {\color{blue}30.72}	& {\color{blue}{30.67}} &{\color{red}{32.43}}  &{\color{red}{28.13}} &{\color{red}31.81}&{\color{red}{32.42}}& {\color{red}{33.51}} \\
		\hline
		GraRED-HP$ ^3 $& {\color{red}26.80} & {\color{red}30.88} & {\color{red}30.74} & {\color{blue}32.42} & {\color{blue}28.06} & {\color{blue}31.80} & {\color{blue}32.27} & {\color{blue}33.50}\\
		\hline
	\end{tabular}}
\end{table}
\begin{table}[htbp]
	\centering
	\caption{Parameters on GraRED-HP$^3$ and restarted HPPP with advanced denoisers.}
	\label{table: parameters of GraRED-HPPP and restartHP3}
\scalebox{0.8}{	\begin{tabular}{|c|c|}
	\hline
	GraRED-HP$^3(D_1)$ & $ \tau = 3,s =1/3,N=500 $ \\
	\hline
	Restarted HPPP$(D_2)$ & $ \tau = 1.2,s =1/1.2, N=30, q=10 $ \\
	\hline
	GraRED-HP$^3(D_3)$ & $ \tau = 1.8,s =1/1.8,N=500$ \\
	\hline
	Restarted HPPP$^3(D_3)$ & $ \tau = 1.8,s =1/1.8,N=500,q = 40$\\
	\hline
\end{tabular}}
\end{table}
\begin{table}[htbp]
    \centering
    \caption{Comparison with GraRED-P$^3$, GraRED-HP$^3$, and restarted HPPP with TNRD for different tasks.}
    \label{table:different tasks}
    \scalebox{0.8}{\begin{tabular}{|c|c|c|c|c|c|c|}
        \hline
       Kernels & \includegraphics[width=0.5cm]{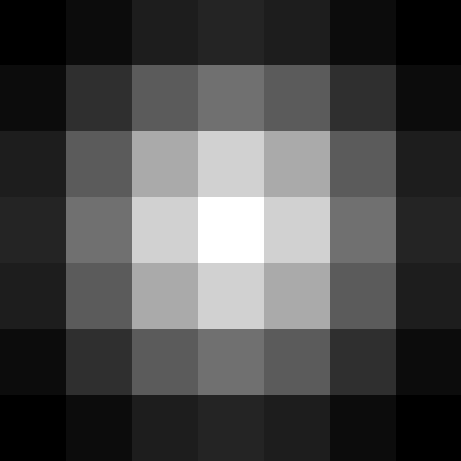} & \includegraphics[width=0.5cm]{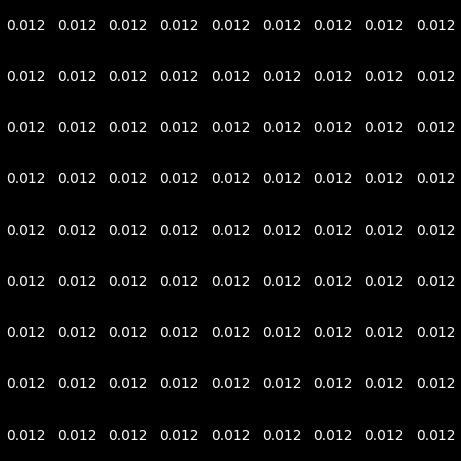} & \includegraphics[width=0.5cm]{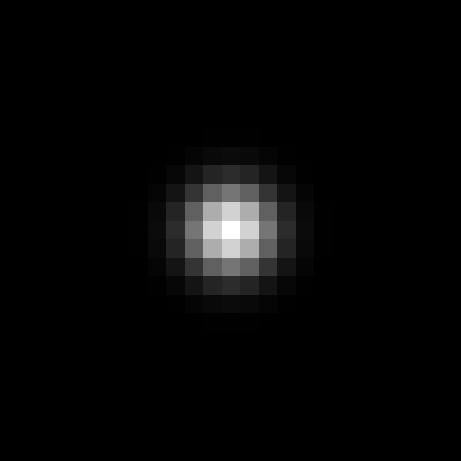} & \includegraphics[width=0.5cm]{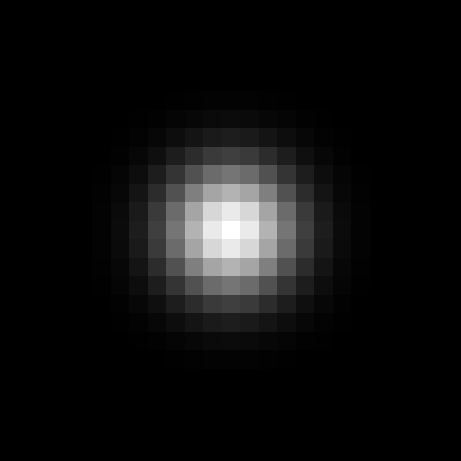} & \includegraphics[width=0.5cm]{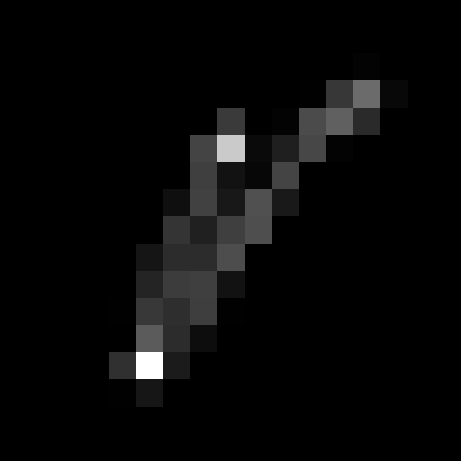} & \includegraphics[width=0.5cm]{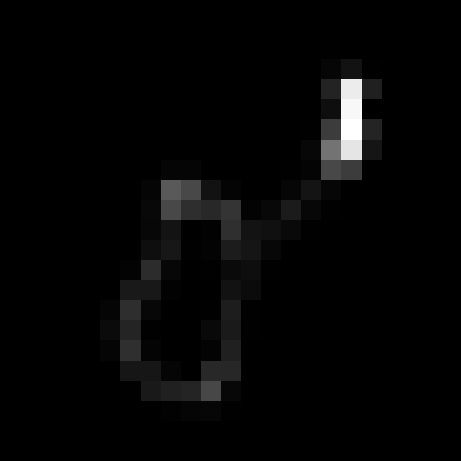} \\ 
        \hline
		noise level, scale & $\sigma=5,\times 3$ & $\sigma=2.55,\times 1$ & $\sigma=2.55,\times 1$ & $\sigma=2.55,\times 1$ & $\sigma=2.55,\times 1$ & $\sigma=2.55,\times 1$ \\
		\hline
        GraRED-P$^3$ & 27.49 & 29.25 & 30.92 & 28.19 & 33.53 & 33.73 \\
        \hline
        GraRED-HP$^3$ & 27.61 & 29.33 & 30.99 & 28.32 & 33.53 & 33.71 \\
        \hline
		Restarted HPPP & \textbf{27.66} & \textbf{29.36} & \textbf{31.00} & \textbf{28.34} & \textbf{33.53} & \textbf{33.73} \\
        \hline
    \end{tabular}}
\end{table}
\begin{table}[htbp]
	\centering
	\caption{Comparison with state-of-the-art methods. Test grayscale images are degraded by the Gaussian kernel with the noise level $\sigma = 2.55$. The best two results are highlighted in {\color{red}red} and {\color{blue}blue}, respectively.}
	\label{table: comarison with state-of-the-art methods}
	\scalebox{0.8}{
		\begin{tabular}{|c|c|c|c|c|c|c|c|c|c|}
			\hline 
			& Pepper  & Craft & Cameraman  & Couple & Man &  House & Starfish & Butterfly &  Boat \\
			\hline
			RED~\cite{Romano2017} & {27.19} & 26.00 & 26.46 & 29.21 & 30.16 & 32.41 & {28.46}  & 28.36 & 29.58\\
			\hline
			RED-PRO~\cite{cohen2021regularization} & 27.43 & {26.05} &{26.39} & {28.99} & {29.98} & {32.41} & 28.25  & {28.21} & {29.38}\\
			\hline
			DPIR~\cite{zhang2021plug} & 28.47 & 27.02 & {\color{red}27.63} & {\color{red}30.24} &
			{\color{red}30.89} & {\color{red}33.56} & {\color{blue}29.62} & {30.43} & {\color{red}30.54} \\
			\hline
			PnP-DRS~\cite{hurault2022proximal} & {\color{blue}29.79} & 26.92 & 27.46 & 29.70 & 29.59 & {\color{blue}33.25} & 29.94 & {\color{red}30.66} & 29.46 \\
			\hline 
			PnP-PGD~\cite{hurault2022gradient} & 27.18 & 26.28 & 26.73 & 29.17 & 29.14 & 32.84 & 29.14 & 29.87 & 30.37 \\ 
			\hline
			PnP-DYS~\cite{There-Operator-Splitting2024} & 27.22 & 26.27 & 26.62 & 28.97 & 29.04 & 32.87 & 29.06 & 29.88 & 28.68 \\
			\hline
			GraRED-HP$^3(D_1)$ & {29.11} & {27.05} & {27.16} & {29.37} & {30.46} & {32.43} & {28.77} &  {29.51} & {29.86}\\
			\hline
			Restarted HPPP$(D_2)$  & 28.34 & 26.75 & 27.02 & 29.71 &
			30.44 & 32.69 & 28.91 & 29.28 & 30.01\\
			\hline
			GraRED-HP$^3(D_3)$ & 29.51 & {\color{blue}27.09} &27.45 & 30.03 & 30.79 & 33.15 & 30.06 & 30.49 & 30.36\\
			\hline
			Restarted HPPP$(D_3)$ & {\color{red}29.81} & {\color{red}27.13} & {\color{blue}27.49} & {\color{blue}30.04} & {\color{blue}30.80} & 33.15 & {\color{red}30.09} & {\color{blue}30.53} & {\color{blue}30.37}\\
			\hline
	\end{tabular}}
\end{table}
\begin{figure}[htbp]
	\centering
	\includegraphics[width=1\linewidth]{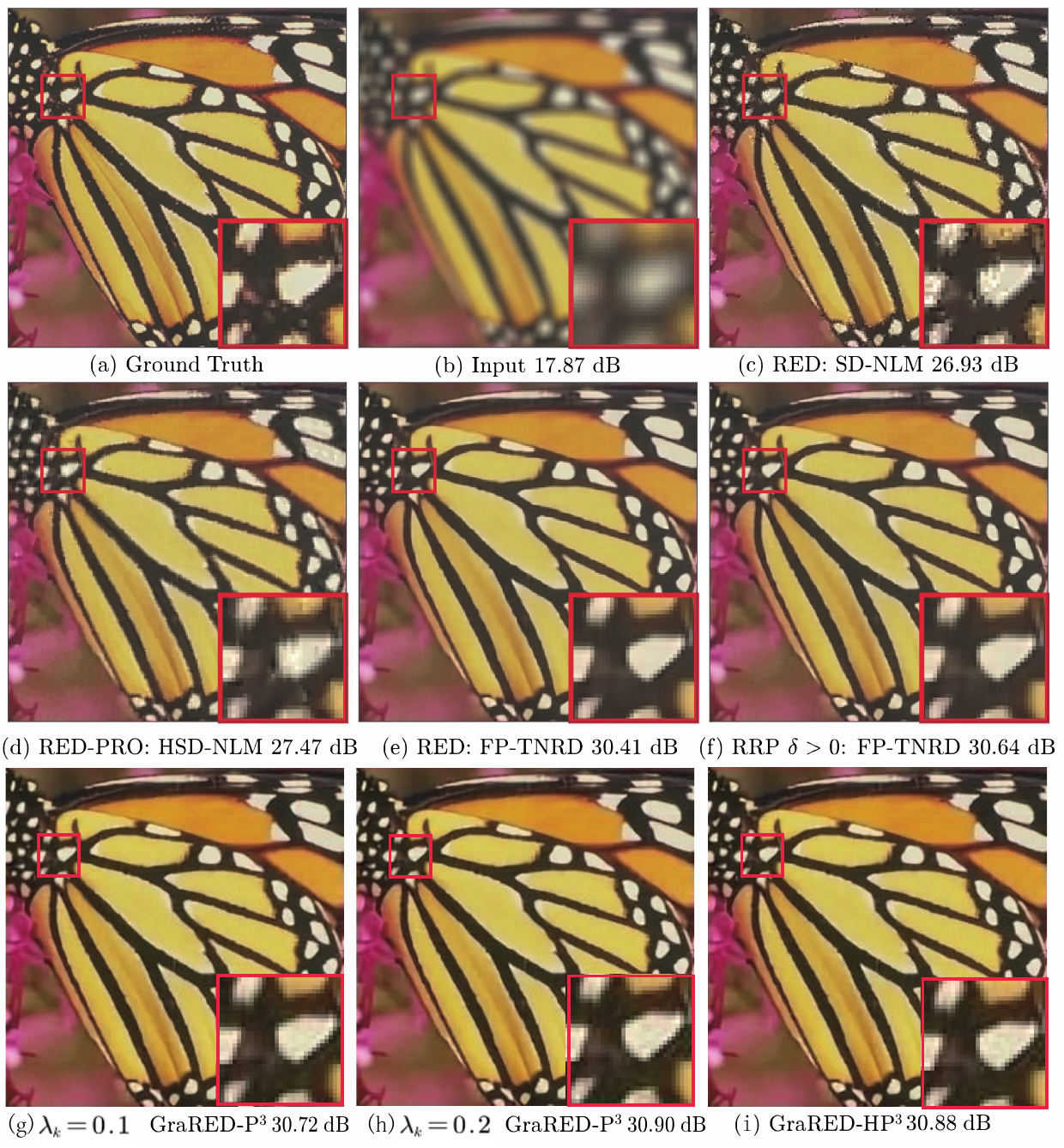}
	\caption{Deblurring results of 
		\textbf{Butterfly} degraded by the uniform kernel.}
	\label{fig: comparison with RED and REDPRO}
\end{figure} 
\subsection{Image inpainting}
In this section, we use the proposed algorithm to solve TV-regularized image inpainting problems and compare their numerical results with CP~\cite{Chambolle2011}, PPP~\cite{Bredies2022}, and HPPP algorithms. The discrete image inpainting model is 
$$
\min_{\bx\in \mathbb{R}^n} \lambda \left\| \mathbf{M}\odot \bx- \mathbf{y} \right\|_F^2+\beta \left\| \nabla \bx \right\|_1,
$$
where $ \mathbf{y} $ is the degraded image, $ \mathbf{M}$ is a mask, and $\lambda, \beta $ are balance coefficients. 

We test 10 common images for evaluation. The first $ \mathbf{M} $ is filled with a Bernoulli random mask, where each pixel is missing with probability $ p=0.5 $,  i.e., $ 50\%  $ of pixels are missed. The second $ \mathbf{M} $ is a character mask where about $ 19\% $ of pixels are missed. All the algorithms start with the degraded image. For classic algorithms, we fix the balance parameter $ \alpha = 0.01 $ and the total number $ N = 400 $ and use the following additional parameters:
\begin{itemize}
\item HPPP: $ \tau = s  = 1/\normm{K}  =  0.57,\bx_a = \mathbf{1} \in \mathbb{R}^{m\times n}, \by_a = 0\cdot \nabla \bx_a, \mu_k = \frac{1}{10(k+2)} $;
\item PPP: $ \tau = s  = 1/\normm{K}  =  0.57, \lambda_k= 1.6$, or $ \lambda_k= 1.2 $;
\item CP: $ \tau = s  = 1/\normm{K}  =  0.57 $.
\end{itemize}
Both GraRED-P$ ^3 $ and  GraRED-HP$ ^3 $ use DnCNN~\cite{ryu2019plug}, other parameters are used:
\begin{itemize}
\item GraRED-P$ ^3 $ : $ \tau = 10, s =0.1, \lambda_k = 0.2, \lambda = 5$;
\item GraRED-HP$ ^3 $ : $ \tau = 10, s =0.1, \mu_k = 0.05/(k+2), \lambda = 5, \bx_a = \by, \by_a = \mathbf{0}$.
\end{itemize}

In Tables~\ref{table: image inpainting PSNR about CP, PPP and HPPP} and~\ref{table: character missing image inpainting PSNR about CP, PPP and HPPP}, we compared the numerical performance of classic algorithms with TV regularization. In Figures~\ref{fig: inpainting random mask} and~\ref{fig: inpainting character mask}, we compare visualization results of \textbf{House} degraded by the Bernoulli random mask and the character mask; the proposed algorithms achieve better visual performance than TV regularization.  
\begin{table}[htbp]
	\centering
	\caption{Numerical results of image inpainting compared with CP, PPP, and HPPP (noise level $ \sigma = 2.55 $, Bernoulli random missing).}
	\label{table: image inpainting PSNR about CP, PPP and HPPP}
	\scalebox{0.8}{
		\begin{tabular}{|c|c|c|c|c|c|c|c|c|c|}
			\hline 
			& Cameraman & House & Peppers & Starfish & Butterfly & Craft & Parrots & Barbara & Boat \\
			\hline
			CP & 23.54 & 28.94 & 24.53 & 23.89 & 23.59 & 23.42 & 23.20  & 23.10 & 26.29 \\
			\hline
			PPP & {23.58} & 29.05 & {24.53 }& 24.32 & 23.76& 23.43 & {23.33}  & 23.13& 26.36\\
			\hline
			HPPP & {23.89} & {29.19} & {24.55} &{24.46} & {24.04} & {23.52} & {23.44 } & {23.35} & {26.45}\\
			\hline
			GraRED-P$ ^3 $ & \textbf{29.92} & \textbf{36.15} & {32.67} & \textbf{31.99} & \textbf{31.95} &\textbf{29.53} & \textbf{30.05} & \textbf{32.25} & \textbf{32.91} \\
			\hline
			GraRED-HP$ ^3 $ & 29.90 & 36.14 & \textbf{32.68} & {31.97} & {31.92} & 29.52 & 30.04  & {32.23} & {32.90} \\
			\hline
	\end{tabular}}
\end{table}
\begin{table}[htbp]
	\centering
	\caption{Numerical results of image inpainting compared with CP, PPP, and HPPP (noise level $ \sigma= 2.55 $, character texture missing).}
	\label{table: character missing image inpainting PSNR about CP, PPP and HPPP}
	\scalebox{0.8}{
		\begin{tabular}{|c|c|c|c|c|c|c|c|c|c|}
			\hline 
			& Cameraman & House & Peppers & Starfish & Butterfly & Craft & Parrots & Barbara & Boat \\
			\hline
			CP & 25.87 & 31.58 & 29.40 & 26.16 & 25.60 & 26.33 & 25.20  & 24.21 & 25.23 \\
			\hline
			PPP & {26.33} & 31.78 & {30.13 }& 26.86 & 26.08 & 26.47 & {25.64}  & 27.32& 28.05\\
			\hline
			HPPP & {26.33} & {32.00} & {30.02} & {26.88} & {26.38} & {26.51} & {25.83 } &{27.34} & {28.09}\\
			\hline
			GraRED-P$ ^3 $ & \textbf{29.18} & {37.46} & {33.79} & {30.89} & \textbf{29.45} &{29.20} & \textbf{26.04}  & \textbf{29.68} & {30.04} \\
			\hline
			GraRED-HP$ ^3 $ & 29.13 & \textbf{37.47 }& \textbf{33.80} & \textbf{30.93} & {29.42} & \textbf{29.20} & 26.02 &{29.60} & \textbf{30.04} \\
			\hline
	\end{tabular}}
\end{table}

\begin{figure}[htbp]
	\centering
	\captionsetup[subfigure]{justification=centering}
	\begin{subfigure}[b]{.22\linewidth}
		\centering
		\begin{tikzpicture}[spy using outlines={rectangle,blue,magnification=5,size=1.2cm, connect spies}]
		\node {\includegraphics[height=3cm]{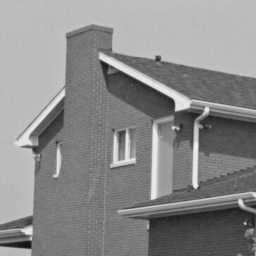}};
		\spy on (0.25,-0.6) in node [left] at (1.5,.9);
		\end{tikzpicture}
		\caption{Clean image  \\ ~}
	\end{subfigure}
	\begin{subfigure}[b]{.22\linewidth}
		\centering
		\begin{tikzpicture}[spy using outlines={rectangle,blue,magnification=5,size=1.2cm, connect spies}]
		\node {\includegraphics[height=3cm]{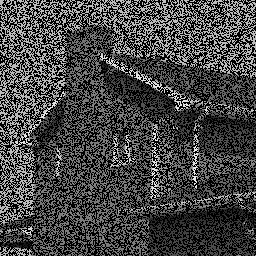}};
		\spy on (0.25,-0.6) in node [left] at (1.5,.9);
		\end{tikzpicture}
		\caption{Degraded image \\ ~}
	\end{subfigure}
	\begin{subfigure}[b]{.22\linewidth}
		\centering
		\begin{tikzpicture}[spy using outlines={rectangle,blue,magnification=5,size=1.2cm, connect spies}]
		\node {\includegraphics[height=3cm]{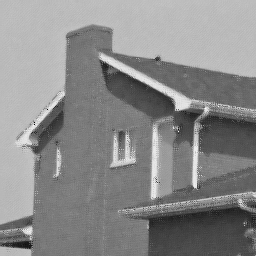}};
		\spy on (0.25,-0.6) in node [left] at (1.5,.9);
		\end{tikzpicture}
		\caption{CP ($28.94$dB) \\ ~}
	\end{subfigure}
	\begin{subfigure}[b]{.22\linewidth}
		\centering
		\begin{tikzpicture}[spy using outlines={rectangle,blue,magnification=5,size=1.2cm, connect spies}]
		\node {\includegraphics[height=3cm]{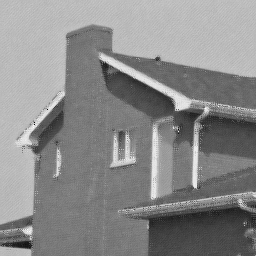}};
		\spy on (0.25,-0.6) in node [left] at  (1.5,.9);
		\end{tikzpicture}
		\caption{$ \lambda_k = 1.6 $: PPP ($29.06$dB)}
	\end{subfigure}
	\begin{subfigure}[b]{.22\linewidth}
		\centering
		\begin{tikzpicture}[spy using outlines={rectangle,blue,magnification=5,size=1.2cm, connect spies}]
		\node {\includegraphics[height=3cm]{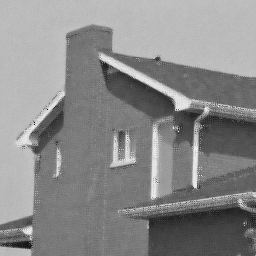}};
		\spy on (0.25,-0.6) in node [left] at (1.5,.9);
		\end{tikzpicture}
		\caption{$ \lambda_k = 1.2 $: PPP ($29.09$dB)}
	\end{subfigure}
	\begin{subfigure}[b]{.22\linewidth}
		\centering
		\begin{tikzpicture}[spy using outlines={rectangle,blue,magnification=5,size=1.2cm, connect spies}]
		\node {\includegraphics[height=3cm]{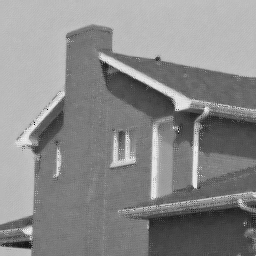}};
		\spy on (0.25,-0.6) in node [left] at  (1.5,.9);
		\end{tikzpicture}
		\caption{HPPP ($29.20$dB) \\ ~}
	\end{subfigure}
	\begin{subfigure}[b]{.22\linewidth}
		\centering
		\begin{tikzpicture}[spy using outlines={rectangle,blue,magnification=5,size=1.2cm, connect spies}]
		\node {\includegraphics[height=3cm]{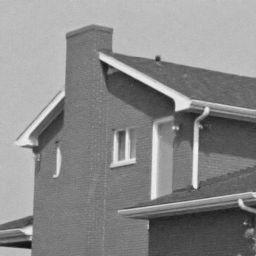}};
		\spy on (0.25,-0.6) in node [left] at  (1.5,.9);
		\end{tikzpicture}
		\caption{GraRED-P$ ^3 $ ($36.15$dB)}
	\end{subfigure}
	\begin{subfigure}[b]{.22\linewidth}
		\centering
		\begin{tikzpicture}[spy using outlines={rectangle,blue,magnification=5,size=1.2cm, connect spies}]
		\node {\includegraphics[height=3cm]{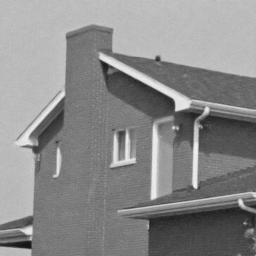}};
		\spy on (0.25,-0.6) in node [left] at  (1.5,.9);
		\end{tikzpicture}
		\caption{GraRED-HP$ ^3 $ ($36.14$dB)}
	\end{subfigure}
\caption{Recovery results of \textbf{House} degraded by the random mask.}
\label{fig: inpainting random mask}
\end{figure}

\begin{figure}[htbp]
    \centering
	\captionsetup[subfigure]{labelformat=parens, labelsep=space, justification=centering}
    \begin{minipage}[b]{.22\linewidth}
        \centering
        \begin{tikzpicture}[spy using outlines={rectangle,blue,magnification=5,size=1.2cm, connect spies}]
        \node {\includegraphics[height=3cm]{./figures/CP_P3_HP3_GraRED_P3_GraRED_HP3/inpainting/02.png}};
        \spy on (0.86,-0.16) in node [left] at (1.5,.9);
        \end{tikzpicture}
        \subcaption{Clean image  \\ ~}
    \end{minipage}
    \begin{minipage}[b]{.22\linewidth}
        \centering
        \begin{tikzpicture}[spy using outlines={rectangle,blue,magnification=5,size=1.2cm, connect spies}]
        \node {\includegraphics[height=3cm]{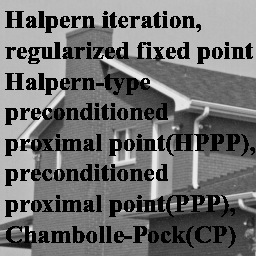}};
        \spy on (0.86,-0.16) in node [left] at (1.5,.9);
        \end{tikzpicture}
        \subcaption{Degraded image \\ ~}
    \end{minipage}
    \begin{minipage}[b]{.22\linewidth}
        \centering
        \begin{tikzpicture}[spy using outlines={rectangle,blue,magnification=5,size=1.2cm, connect spies}]
        \node {\includegraphics[height=3cm]{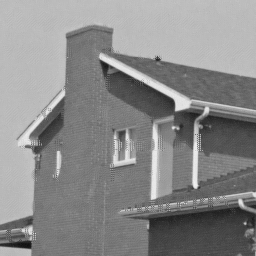}};
        \spy on (0.86,-0.16) in node [left] at (1.5,.9);
        \end{tikzpicture}
        \subcaption{CP ($31.58$dB) \\ ~}
    \end{minipage}
    \begin{minipage}[b]{.22\linewidth}
        \centering
        \begin{tikzpicture}[spy using outlines={rectangle,blue,magnification=5,size=1.2cm, connect spies}]
        \node {\includegraphics[height=3cm]{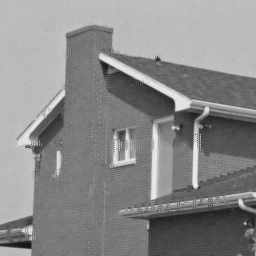}};
        \spy on (0.86,-0.16) in node [left] at  (1.5,.9);
        \end{tikzpicture}
        \subcaption{$ \lambda_k = 1.6 $: PPP ($31.79$dB)}
    \end{minipage}
    \begin{minipage}[b]{.22\linewidth}
        \centering
        \begin{tikzpicture}[spy using outlines={rectangle,blue,magnification=5,size=1.2cm, connect spies}]
        \node {\includegraphics[height=3cm]{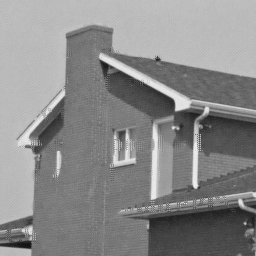}};
        \spy on (0.86,-0.16) in node [left] at (1.5,.9);
        \end{tikzpicture}
        \subcaption{$ \lambda_k = 1.2 $: PPP ($31.74$dB)}
    \end{minipage}
    \begin{minipage}[b]{.22\linewidth}
        \centering
        \begin{tikzpicture}[spy using outlines={rectangle,blue,magnification=5,size=1.2cm, connect spies}]
        \node {\includegraphics[height=3cm]{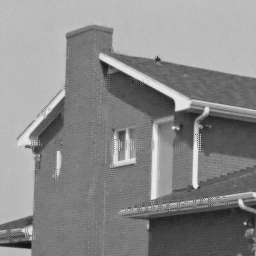}};
        \spy on (0.86,-0.16) in node [left] at  (1.5,.9);
        \end{tikzpicture}
        \subcaption{HPPP ($32.00$dB) \\ ~ }
    \end{minipage}
    \begin{minipage}[b]{.22\linewidth}
        \centering
        \begin{tikzpicture}[spy using outlines={rectangle,blue,magnification=5,size=1.2cm, connect spies}]
        \node {\includegraphics[height=3cm]{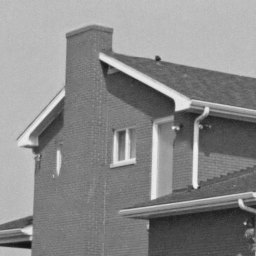}};
        \spy on (0.86,-0.16) in node [left] at  (1.5,.9);
        \end{tikzpicture}
        \subcaption{GraRED-P$ ^3 $ ($37.46$dB)}
    \end{minipage}
    \begin{minipage}[b]{.22\linewidth}
        \centering
        \begin{tikzpicture}[spy using outlines={rectangle,blue,magnification=5,size=1.2cm, connect spies}]
        \node {\includegraphics[height=3cm]{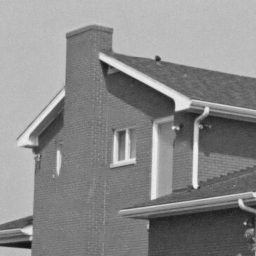}};
        \spy on (0.86,-0.16) in node [left] at  (1.5,.9);
        \end{tikzpicture}
        \subcaption{GraRED-HP$ ^3 $ ($37.47$dB)}
    \end{minipage}
\caption{Recovery results of \textbf{House} degraded by the character mask.}
\label{fig: inpainting character mask}
\end{figure}
\section{Conclusions}
\label{sec:conclusions}
This paper introduces a Halpern-type PPP (HPPP) algorithm, which has the advantage of strong convergence over PPP. In addition, GraRED-HP$^3$ is proposed by integrating HPPP with denoiser priors for IR problems, which is an accelerated PnP-ADMM. Advantages of HPPP are demonstrated through several toy examples, and GraRED-HP$^3$ achieves state-of-the-art performance on IR experiments. In the future, we plan to extend the definition of $ \mathcal{M}$-monotonicity to $ \mathcal{M}$-comonotonicity for nonconvex case~\cite{Bauschke2021}. 
\section*{Acknowledgments}
The authors thank Prof. Defeng Sun for bringing the interesting work~\cite{Sun2024} to our attention and Mr. Feng Xue for his insightful discussions. We sincerely appreciate Associate Editor Michael Elad and the anonymous referees for their thorough review and valuable comments, which have substantially improved this paper.
\appendix
\section{Boundedness}
\label{sec:main}
First, we study the boundedness and asymptotic regularity of $ \{\bu^k\}_{k\in \mathbb{N}} $ generated by~\eqref{eq: HPPP}. To further establish regularity of $ \{\bu^k\}_{k\in \mathbb{N}} $ (Theorem~\ref{thm:bigthm}), we introduce an important lemma from~\cite[Lemma 2.5]{2002Iterative}.
\begin{lemma}\label{lemma for convergence}
	Let $ \{a_k\}_{k\in\mathbb{N}} $be a sequence of non-negative real numbers satisfying
	\begin{equation}
	a_{k+1}\leq (1-\mu_k)a_k+\mu_k \beta_k + \gamma_k,
	\end{equation}
	where $ \{\mu_k\}_{k\in\mathbb{N}}, \{\beta_k\}_{k\in\mathbb{N}}, \{\gamma_k\}_{k\in\mathbb{N}}$ satisfies the following conditions:
	\begin{enumerate}[label={(\roman*)},leftmargin=1cm]
		\item $\{ \mu_k \} $  converges to 0 in $ [0,1] $, and $\sum_{k=0}^\infty \mu_k = +\infty $, or equivalently, $ \prod_{k=0}^\infty  (1-\mu_k)= 0 $;
		\item $\lim\sup_{k\to \infty}\beta_k \leq 0 $;
		\item $\gamma_k \geq 0, \sum_{k=0}^\infty\gamma_k <\infty $.
	\end{enumerate}
	Then $\lim_{k\to \infty}a_k=0$.
\end{lemma}
\begin{lemma}\label{lemma: lemma for regularized fixed point}
	Let $ \mathcal{A}:\mathcal{H}\to 2^\mathcal{H}$ be an operator with $\mathrm{zer} \mathcal{A} \neq \emptyset$, and let $\mathcal{M}$ be an admissible preconditioner such that $\mathcal{M}^{-1}\mathcal{A}$ is $\mathcal{M}$-monotone and $(\mathcal{M}+\mathcal{A})^{-1}$ is $L$-Lipschitz. Let $\{\bu^k \}_{k\in \mathbb{N}}$ be the sequence generated by~\eqref{eq: HPPP}. Then the following
	assertions are hold:
	\begin{enumerate}[label={(\roman*)},leftmargin=1cm]
		\item The sequences $\{\bu^k \}_{k\in\mathbb{N}}, \{\mathcal{T}\bu^k\}_{k\in\mathbb{N}}$ are bounded;
		\item If the sequence $\{ \mu_k \}_{k\in \mathbb{N}}$ satisfies $\lim_{k\to \infty}\frac{\mu_{k+1}-\mu_k}{\mu_k} = 0 $ or $\sum_{k\in \mathbb{N}}\left\vert \mu_{k+1}-\mu_k \right\vert<\infty$, then $\lim_{ k\to \infty} \left\| \mathcal{T}\bu^k-\bu^k \right\|_\mathcal{M}=  0$;
		\item If the conditions of (ii) hold, then $\lim_{ k\to \infty}\left\| \bu^{k+1}-\bu^k \right\|= 0, \lim_{ k\to \infty}\left\| \mathcal{T}\bu^k-\bu^k \right\|= 0$.
	\end{enumerate}
\end{lemma}
\begin{proof}
	Let $\mathcal{M} =\mathcal{C}\mathcal{C}^*$ be a decomposition of $\mathcal{M}$~\cite[Proposition 2.3]{Bredies2022}. Denote 
	\begin{equation}\label{constant C}
		C = L\left\| \mathcal{C} \right\|,
	\end{equation}
	for $\bu', \bu''\in \mathcal{H}$, we have 
		\begin{align}
			\left\| \mathcal{T}\bu'-\mathcal{T}\bu'' \right\|
			&=\left\| (\mathcal{M}+\mathcal{A})^{-1}\mathcal{C}\mathcal{C}^*\bu'-(\mathcal{M}+\mathcal{A})^{-1}\mathcal{C}\mathcal{C}^*\bu'' \right\| \notag\\
			&=\left\| (\mathcal{M}+\mathcal{A})^{-1}\mathcal{C}(\mathcal{C}^*\bu'-\mathcal{C}^*\bu'') \right\| \notag\\
			&\leq L\left\| \mathcal{C} \right\| \left\| \mathcal{C}^*(\bu'-\bu'') \right\| \notag\\
			& = L\left\| \mathcal{C} \right\|\left\| \bu'-\bu''  \right\|_\mathcal{M} \notag\\
			& = C\left\|\bu'-\bu'' \right\|_\mathcal{M},\label{ineq: boundedness by M norm}
			\end{align}
	where the third inequality follows from the $L$-Lipschitz property of $(\mathcal{M}+\mathcal{A})^{-1}$, and the fourth equality follows from the definition of $\left\| \mathbf{u} \right\|_{\mathcal{M}}$, i.e., $\left\| \mathbf{u} \right\|_{\mathcal{M}} = \sqrt{\langle \mathcal{M} \mathbf{u},\mathbf{u}\rangle} =\left\| \mathcal{C}^* \mathbf{u} \right\|$ for any $\mathbf{u}\in \mathcal{H}$.

	First, we prove $ \{ \mathbf{u}^k \}_{k\in \mathbb{N}} $ is bounded. For any $ \bu^*\in \mathrm{Fix}(\mathcal{T})$, we have
	\begin{align*}
	\left\lVert \bu^{k+1} -\bu^*\right\rVert_{\mathcal{M}} &= \left\| \mu_{k+1}(\mathbf{a}-\bu^*)+(1-\mu_{k+1})(\mathcal{T}\bu^k-\bu^*)\right\|_{\mathcal{M}} \\
	& \leq \mu_{k+1}\left\| \mathbf{a}-\bu^* \right\|_\mathcal{M}+(1-\mu_{k+1})\left\| \bu^k-\bu^* \right\|_\mathcal{M} \\
	& \leq \max \{ \left\| \mathbf{a}-\bu^* \right\|_\mathcal{M},\left\| \bu^k-\bu^* \right\|_\mathcal{M}\}\\
	& \leq \dots \leq \max\{\left\| \mathbf{a}-\bu^* \right\|_\mathcal{M},\left\| \bu^0-\bu^* \right\|_\mathcal{M}\}.
	\end{align*}
	Denote 
\begin{equation}\label{constant C1}
	C_1 = \max\{\left\| \mathbf{a}-\bu^* \right\|_\mathcal{M},\left\| \bu^0-\bu^* \right\|_\mathcal{M}\},
\end{equation}
then we have $ \left\| \bu^k-\bu^* \right\|_\mathcal{M} \leq C_1 $ for all $ k\in \mathbb{N}$.

	Furthermore, 
	\begin{displaymath}
	\begin{aligned}
	\left\| \bu^{k+1}-\bu^* \right\| & \leq \mu_{k+1}\left\| \mathbf{a}-\bu^* \right\|+(1-\mu_{k+1})\left\| \mathcal{T}\bu^k-\bu^* \right\| \\
	&\leq \mu_{k+1}\left\| \mathbf{a}-\bu^* \right\|+(1-\mu_{k+1})\cdot C \left\| \bu^k-\bu^* \right\|_\mathcal{M} \\
	& \leq \max\{ \left\| \mathbf{a}-\bu^* \right\|, CC_1\}<+\infty .
	\end{aligned}
	\end{displaymath}
	By~\eqref{ineq: boundedness by M norm}, we have
	\begin{equation}\label{ineq: boundedness of Tuk-a}
	\left\| \mathcal{T}\bu^k-\bu^* \right\|\leq C\left\| \bu^k-\bu^* \right\|_\mathcal{M}\leq CC_1<+\infty.  
	\end{equation}
	So the sequence $\{ \mathcal{T}\bu^k \}_{k\in \mathbb{N}}$ is also bounded.
	
	For (ii), we will prove $\lim_{k\to\infty}\normm{\mathcal{T}\bu^k-\bu^k}_{\mathcal{M}} = 0$. Let $M = \left\| \mathbf{a}-\mathbf{u}^* \right\|_{\mathcal{M}}+C_1 $, by the triangle inequality and the $\mathcal{M}$-non-expansiveness of $\mathcal{T}$, we have
	\begin{align}
	\left\| \mathbf{a}-\mathcal{T}\bu^k \right\|_\mathcal{M}
	&\leq \left\| \mathbf{a}-\mathbf{u}^* \right\|_{\mathcal{M}}+\left\| \mathbf{u}^*-\mathcal{T}\bu^k \right\|_{\mathcal{M}} \notag\\
	&\leq \left\| \mathbf{a}-\mathbf{u}^* \right\|_{\mathcal{M}}+\left\| \mathbf{u}^k-\mathbf{u}^* \right\|_\mathcal{M} \notag\\
	&\leq \left\| \mathbf{a}-\mathbf{u}^* \right\|_{\mathcal{M}}+C_1 = M \label{eq: boundedness of Tuk-a}.
	\end{align}

	Note that by $\mathbf{u}^{k+1} = \mu_{k+1} \mathbf{a}+(1-\mu_{k+1})\mathcal{T}\mathbf{u}^k$, we have
	\begin{align}
	\left\| \bu^{k+1}-\bu^k \right\|_\mathcal{M} & = \left\| (\mu_{k+1}-\mu_k)(\mathbf{a}-\mathcal{T}\bu^{k-1})+(1-\mu_{k+1})(\mathcal{T}\bu^k-\mathcal{T}\bu^{k-1}) \right\|_\mathcal{M}\notag \\
	& \leq (1-\mu_{k+1})\left\| \mathcal{T}u^k-\mathcal{T}u^{k-1} \right\|_\mathcal{M}+\left\vert \mu_{k+1}-\mu_k \right\vert\left\| \mathbf{a}-\mathcal{T}\bu^{k-1} \right\|_\mathcal{M}\notag\\
	&\leq (1-\mu_{k+1})\left\| \bu^k-\bu^{k-1} \right\|_\mathcal{M}+M \left\vert \mu_{k+1}-\mu_k \right\vert,\label{ineq: convergence of M uk+1-uk}
	\end{align}
	where the last inequality follows from Lemma~\ref{lemma: MFNE} and~\eqref{eq: boundedness of Tuk-a}.

	Applying Lemma~\ref{lemma for convergence} together with $\lim_{k\to \infty}\frac{\mu_{k+1}-\mu_k}{\mu_k} = 0 $ or $\sum_{k\in \mathbb{N}}\left\vert \mu_{k+1}-\mu_k \right\vert<\infty$ to~\eqref{ineq: convergence of M uk+1-uk}, we have $ \lim_{k\to \infty}\left\| \bu^{k+1}-\bu^k \right\|_\mathcal{M} = 0$. Since $\mathbf{u}^{k+1}-\mathbf{u}^k = \mu_{k+1}(\mathbf{a}-\mathcal{T}\mathbf{u}^k)$, we have
	\begin{align}
	\left\| \bu^k-\mathcal{T}\bu^k \right\|_\mathcal{M} & = \left\| (\bu^k-\bu^{k+1})+(\bu^{k+1}-\mathcal{T}{\mathbf{u}^k}) \right\|_\mathcal{M} \notag\\
	&\leq \left\| \bu^k-\bu^{k+1} \right\|_\mathcal{M}+\mu_{k+1}\left\| \mathbf{a}-\mathcal{T}\bu^k \right\|_\mathcal{M}.\label{eq:uk-Tuk}
	\end{align}
	Using~\eqref{eq: boundedness of Tuk-a} and $\lim_{k\to \infty}\mu_k = 0$ in~\eqref{eq:uk-Tuk}, we can obtain $\lim_{k\to \infty} \left\| \bu^k-\mathcal{T}\bu^k \right\|_\mathcal{M}= 0$.
	
For (iii), let $M'=\left\| \mathbf{a}-\mathbf{u}^* \right\| +CC_1$, by the triangle inequality and~\eqref{ineq: boundedness of Tuk-a}, we have 
\begin{equation}\label{ineq: boundedness of a-Tuk l2 norm}
	\left\| \mathbf{a}-\mathcal{T}\bu^{k} \right\|\leq \left\| \mathbf{a}-\mathbf{u}^* \right\| +\left\| \mathbf{u}^*- \mathcal{T}\bu^{k}\right\|\leq \left\| \mathbf{a}-\mathbf{u}^* \right\| +CC_1 = M'.
\end{equation}

Similar to the proof of~\eqref{ineq: convergence of M uk+1-uk}, we have
		\begin{align}
			\left\| \bu^{k+1}-\bu^k \right\| &\leq (1-\mu_{k+1})\left\| \mathcal{T}\bu^k-\mathcal{T}\bu^{k-1} \right\|+\left\vert \mu_{k+1}-\mu_k \right\vert \left\| \mathbf{a}-\mathcal{T}\bu^{k-1} \right\| \notag\\
			& \leq C(1-\mu_{k+1})\left\| \bu^k-\bu^{k-1} \right\|_\mathcal{M}+M'\left\vert \mu_{k+1}-\mu_k \right\vert, \label{ineq: convergence of uk+1-uk}
			\end{align}
	where the last inequality follows from~\eqref{ineq: boundedness by M norm}.
	Applying~\eqref{ineq: convergence of M uk+1-uk} and $\lim_{k\to\infty}\mu_k = 0$ to the above inequality~\eqref{ineq: convergence of uk+1-uk}, we have $\lim_{k\to\infty}\left\| \bu^{k+1}-\bu^{k} \right\|= 0$. Similar to the proof of~\eqref{eq:uk-Tuk},
	\begin{displaymath}
	\left\| \bu^k-\mathcal{T}\bu^k \right\|\leq \left\| \bu^k-\bu^{k+1} \right\|+\mu_{k+1}\left\| \mathbf{a}-\mathcal{T}\bu^k \right\|.
	\end{displaymath}
	By $\lim_{k\to\infty}\left\| \bu^{k+1}-\bu^{k} \right\|= 0, \lim_{k\to \infty}\mu_k = 0$ and boundedness of $\{ \mathcal{T} \mathbf{u}^k \}_{k\in \mathbb{N}}$, we finally prove $\lim_{k\to \infty}\left\| \bu^k-\mathcal{T}\bu^k \right\|= 0$.
\end{proof}
\section{$\mathcal{M}$-projection}
\begin{lemma}\label{lemma: generalized VI}
	Let $ \mathcal{A}:\mathcal{H}\to 2^\mathcal{H}$ be an operator with $\mathrm{zer} \mathcal{A} \neq \emptyset$, and let $\mathcal{M}$ be an admissible preconditioner such that $\mathcal{M}^{-1}\mathcal{A}$ is $\mathcal{M}$-monotone and $(\mathcal{M}+\mathcal{A})^{-1}$ is $L$-Lipschitz. Then there exists a unique solution ${\mathbf{u}}^*=\arg\min_{\bu\in \mathrm{Fix}(\mathcal{T})}\left\| \bu-\mathbf{a} \right\|_\mathcal{M}^2$ which solves:
	\begin{equation*}
	\left\langle \bu^*-\mathbf{a}, \bu-\bu^*\right\rangle_ \mathcal{M}\geq 0\quad \forall \bu\in \mathrm{Fix}(\mathcal{T}).
	\end{equation*}
\end{lemma}
\begin{proof}
	Since $l(\bu) = \frac{1}{2} \left\| \bu-\mathbf{a}\right\|_\mathcal{M}^2$ is a proper lower-semicontinuous differentiable convex function, assume that $\mathbf{u}^*$ is the optimal solution of $\min_{\mathbf{u}\in \mathrm{Fix} (\mathcal{T})}l(\mathbf{u})$. Since $\mathrm{Fix}( \mathcal{T}) =\mathrm{zer}\mathcal{A}$ is the convex set, thus $\bu^*+t(\bu-\bu^*)\in \mathrm{Fix} (\mathcal{T})$ for any $ \bu\in \mathrm{Fix}( \mathcal{T} )$ and $t\in (0,1)$. Hence, 
		\begin{align}
			\lim_{t\to 0}\frac{l(\bu^*+t(\bu-\bu^*))-l(\bu^*)}{t} & = \left\langle l'(\bu^*), \bu-\bu^*\right\rangle  \notag\\
			& = \left\langle \mathcal{M}(\bu^*-\mathbf{a}),\bu-\bu^*\right\rangle \notag\\
			& = \left\langle \bu^*-\mathbf{a}, \bu-\bu^*\right\rangle _\mathcal{M}\geq 0.\label{ineq: VI1 for lu}
			\end{align}
	If $\bu^{**}$ is the another solution, then it also satisfies
	\begin{equation}\label{ineq: VI2 for lu}
		\left\langle \bu^{**}-{\mathbf{a}}, \bu-\bu^{**}\right\rangle _\mathcal{M}\geq 0.
	\end{equation}
	Replace $\bu$ with $\bu^{**}, \bu^*$ in~\eqref{ineq: VI1 for lu} and~\eqref{ineq: VI2 for lu}, respectively, then
	\begin{align*}
	& \left\langle \bu^*-\mathbf{a},  \bu^{**}-\bu^*\right\rangle _\mathcal{M}\geq 0,\\
	& \left\langle \bu^{**}-\mathbf{a}, \bu^*-\bu^{**}\right\rangle _\mathcal{M}\geq 0.
	\end{align*}
	Add the above two inequalities, we then obtain $\left\| \bu^*-\bu^{**} \right\|_ \mathcal{M} = 0$. By~\eqref{ineq: boundedness by M norm}, we have 
	\begin{displaymath}
	\left\| \bu^*-\bu^{**} \right\|=\left\| \mathcal{T}\bu^*-\mathcal{T}\bu^{**} \right\|\leq C\left\| \bu^*-\bu^{**} \right\|_\mathcal{M}=0.
	\end{displaymath}
	It follows $\bu^* = \bu^{**}$, which completes the proof of Lemma~\ref{lemma: generalized VI}.
\end{proof}

As mentioned in~\cite{Bauschke2023,BUI2020124315}, we can introduce the following notion of $ \mathcal{M}$-projection.
\begin{definition}[$ \mathcal{M}$-projection]\label{def: M-projection}
	Assume $\mathbf{a}\in \mathcal{H} $ and that there exists a unique point $ \bu^* \in \mathrm{Fix}(\mathcal{T}) $ such that $ \normm{\bu^*-\mathbf{a}}_\mathcal{M}\leq \normm{\bu-\mathbf{a}}_\mathcal{M}$ for any $\bu  \in  \mathrm{Fix}(\mathcal{T})$, then $ \bu^* $ is called the $  \mathcal{M}$-projection of $ \mathbf{a} $ onto $ \mathrm{Fix}(\mathcal{T})$, denoted by $ P_{\mathrm{Fix}(\mathcal{T})} ^\mathcal{M}(\mathbf{a})$.
\end{definition}

The following Lemma~\ref{ineq: generalized VI} extends the variational inequality of the metric projection. 
\begin{lemma}
    \label{ineq: generalized VI}
	Let $ \mathcal{A}:\mathcal{H}\to 2^\mathcal{H}$ be an operator with $\mathrm{zer} \mathcal{A} \neq \emptyset$, and let $\mathcal{M}$ be an admissible preconditioner such that $\mathcal{M}^{-1}\mathcal{A}$ is $\mathcal{M}$-monotone and $(\mathcal{M}+\mathcal{A})^{-1}$ is $L$-Lipschitz. Then the following assertions are equivalent:
	\begin{enumerate}[label={(\roman*)},leftmargin=1cm]
		\item $ \bu^* = P_{\mathrm{Fix}(\mathcal{T})} ^\mathcal{M}(\mathbf{a})$;
		\item $ \left\langle \bu^*-\mathbf{a}, \bu-\bu^*\right\rangle_ \mathcal{M}\geq 0$ for any $ \bu\in \mathrm{Fix}(\mathcal{T}) $.
	\end{enumerate}
\end{lemma}
\begin{proof}
	See Lemma~\ref{lemma: generalized VI} and Definition~\ref{def: M-projection}.
\end{proof}
\bibliographystyle{siamplain}
\bibliography{references.bib}

@Article{Bredies2022,
  author   = {Bredies, Kristian and Chenchene, Enis and Lorenz, Dirk A. and Naldi, Emanuele},
  journal  = {SIAM Journal on Optimization},
  title    = {{Degenerate Preconditioned Proximal Point Algorithms}},
  year     = {2022},
  number   = {3},
  pages    = {2376-2401},
  volume   = {32},
  url      = {https://doi.org/10.1137/21M1448112}
}

@Article{Chambolle2011,
author={Chambolle, Antonin
and Pock, Thomas},
title={{A First-Order Primal-Dual Algorithm for Convex Problems with Applications to Imaging}},
journal={Journal of Mathematical Imaging and Vision},
year={2011},
month={May},
day={01},
volume={40},
number={1},
pages={120-145},
issn={1573-7683},
url={https://doi.org/10.1007/s10851-010-0251-1}
}

@Article{Bredies2015,
  author   = {Bredies, Kristian and Sun, Hongpeng},
  journal  = {SIAM Journal on Numerical Analysis},
  title    = {{Preconditioned Douglas--Rachford Splitting Methods for Convex-concave Saddle-point Problems}},
  year     = {2015},
  number   = {1},
  pages    = {421-444},
  volume   = {53},
  url      = {https://doi.org/10.1137/140965028},
}

@book{BCombettes,
  AUTHOR = {Bauschke, Heinz H. and Combettes, Patrick L.},
  TITLE = {{Convex Analysis and Monotone Operator Theory in
              Hilbert Spaces}},
  EDITION = {Second},
  PUBLISHER = {Springer Cham},
  YEAR = {2017}
}

@article{halpern1967fixed,
  title={Fixed points of nonexpanding maps},
  author={Halpern, Benjamin},
  journal={Bulletin of the American Mathematical Society},
  volume={73},
  number={6},
  pages={957--961},
  year={1967},
  publisher={American Mathematical Society}
}

@article{lions1977approximation,
  title={Approximation de points fixes de contractions},
  author={Lions, Pierre-Louis},
  journal={CR Acad. Sci. Paris Serie, AB},
  volume={284},
  pages={1357--1359},
  year={1977}
}

@article{xu2004viscosity,
  title={Viscosity approximation methods for nonexpansive mappings},
  author={Xu, Hong-Kun},
  journal={Journal of Mathematical Analysis and Applications},
  volume={298},
  number={1},
  pages={279--291},
  year={2004},
  publisher={Elsevier}
}

@article{yamada2001hybrid,
  title={The hybrid steepest descent method for the variational inequality problem over the intersection of fixed point sets of nonexpansive mappings},
  author={Yamada, Isao},
  journal={Inherently parallel algorithms in feasibility and optimization and their applications},
  volume={8},
  pages={473--504},
  year={2001},
  publisher={Citeseer}
}

@Misc{Bauschke2023,
  author        = {Heinz H. Bauschke and Walaa M. Moursi and Shambhavi Singh and Xianfu Wang},
  title         = {{On the Bredies-Chenchene-Lorenz-Naldi algorithm}},
  year          = {2023},
  archiveprefix = {arXiv},
  eprint        = {2307.09747},
  primaryclass  = {math.OC},
}

@Article{2002Iterative,
  author  = {Hong-Kun, X. U.},
  journal = {Journal of the London Mathematical Society},
  title   = {{Iterative Algorithms for Nonlinear Operators}},
  year    = {2002},
  number  = {1},
  pages   = {240-256},
  volume  = {66},
}

@Article{Sabach2017,
  author   = {Sabach, Shoham and Shtern, Shimrit},
  journal  = {SIAM Journal on Optimization},
  title    = {{A First Order Method for Solving Convex Bilevel Optimization Problems}},
  year     = {2017},
  number   = {2},
  pages    = {640-660},
  volume   = {27},
  url      = {https://doi.org/10.1137/16M105592X},
}

@article{lieder2021convergence,
  title={{On the convergence rate of the Halpern-iteration}},
  author={Lieder, Felix},
  journal={Optimization letters},
  volume={15},
  number={2},
  pages={405--418},
  year={2021},
  publisher={Springer}
}

@article{RUDIN1992259,
title = {Nonlinear total variation based noise removal algorithms},
journal = {Physica D: Nonlinear Phenomena},
volume = {60},
number = {1},
pages = {259-268},
year = {1992},
issn = {0167-2789},
url = {https://www.sciencedirect.com/science/article/pii/016727899290242F},
author = {Leonid I. Rudin and Stanley Osher and Emad Fatemi}
}

@Article{He2012,
  author   = {He, Bingsheng and Yuan, Xiaoming},
  journal  = {SIAM Journal on Imaging Sciences},
  title    = {{Convergence Analysis of Primal-Dual Algorithms for a Saddle-Point Problem: From Contraction Perspective}},
  year     = {2012},
  number   = {1},
  pages    = {119-149},
  volume   = {5},
  url      = {https://doi.org/10.1137/100814494},
}

@Article{Eckstein1992,
author={Eckstein, Jonathan
and Bertsekas, Dimitri P.},
title={{On the Douglas---Rachford splitting method and the proximal point algorithm for maximal monotone operators}},
journal={Mathematical Programming},
year={1992},
month={Apr},
day={01},
volume={55},
number={1},
pages={293-318},
issn={1436-4646},
url={https://doi.org/10.1007/BF01581204}
}

@Article{He2015,
author={He, Bingsheng
and Yuan, Xiaoming},
title={{On the convergence rate of Douglas--Rachford operator splitting method}},
journal={Mathematical Programming},
year={2015},
month={Nov},
day={01},
volume={153},
number={2},
pages={715-722},
issn={1436-4646},
url={https://doi.org/10.1007/s10107-014-0805-x}
}

@article{BUI2020124315,
title = {Warped proximal iterations for monotone inclusions},
journal = {Journal of Mathematical Analysis and Applications},
volume = {491},
number = {1},
pages = {124315},
year = {2020},
issn = {0022-247X},
url = {https://www.sciencedirect.com/science/article/pii/S0022247X20304777},
author = {Minh N. Bùi and Patrick L. Combettes}
}

@article{Minty1962MonotoneO,
  title={Monotone (nonlinear) operators in Hilbert space},
  author={George J. Minty},
  journal={Duke Mathematical Journal},
  year={1962},
  volume={29},
  pages={341-346},
  url={https://api.semanticscholar.org/CorpusID:121956938}
}

@Article{Rockafellar1976,
  author   = {Rockafellar, R. Tyrrell},
  journal  = {SIAM Journal on Control and Optimization},
  title    = {{Monotone Operators and the Proximal Point Algorithm}},
  year     = {1976},
  number   = {5},
  pages    = {877-898},
  volume   = {14}
}

@Article{Bredies2017,
author={Bredies, Kristian
and Sun, Hongpeng},
title={{A Proximal Point Analysis of the Preconditioned Alternating Direction Method of Multipliers}},
journal={Journal of Optimization Theory and Applications},
year={2017},
month={Jun},
day={01},
volume={173},
number={3},
pages={878-907},
issn={1573-2878},
url={https://doi.org/10.1007/s10957-017-1112-5}
}

@Article{Wittmann1992,
author={Wittmann, Rainer},
title={Approximation of fixed points of nonexpansive mappings},
journal={Archiv der Mathematik},
year={1992},
month={May},
day={01},
volume={58},
number={5},
pages={486-491},
issn={1420-8938},
url={https://doi.org/10.1007/BF01190119}
}

@Article{Yamada2005,
  author    = {Isao Yamada and Nobuhiko Ogura},
  journal   = {Numerical Functional Analysis and Optimization},
  title     = {{Hybrid Steepest Descent Method for Variational Inequality Problem over the Fixed Point Set of Certain Quasi-nonexpansive Mappings}},
  year      = {2005},
  number    = {7-8},
  pages     = {619-655},
  volume    = {25},
  publisher = {Taylor & Francis},
  url       = {https://doi.org/10.1081/NFA-200045815},
}

@InProceedings{Pock2011,
  author    = {Pock, Thomas and Chambolle, Antonin},
  booktitle = {2011 International Conference on Computer Vision},
  title     = {Diagonal preconditioning for first order primal-dual algorithms in convex optimization},
  year      = {2011},
  pages     = {1762-1769},
  doi       = {10.1109/ICCV.2011.6126441},
}

@Article{Romano2017,
  author   = {Romano, Yaniv and Elad, Michael and Milanfar, Peyman},
  journal  = {SIAM Journal on Imaging Sciences},
  title    = {{The Little Engine That Could: Regularization by Denoising (RED)}},
  year     = {2017},
  number   = {4},
  pages    = {1804-1844},
  volume   = {10},
  url      = {https://doi.org/10.1137/16M1102884},
}

@article{moreau1965proximite,
  title={Proximit{\'e} et dualit{\'e} dans un espace hilbertien},
  author={Moreau, Jean-Jacques},
  journal={Bulletin de la Soci{\'e}t{\'e} math{\'e}matique de France},
  volume={93},
  pages={273--299},
  year={1965}
}

@Article{Gribonval2020,
author={Gribonval, R{\'e}mi
and Nikolova, Mila},
title={{A Characterization of Proximity Operators}},
journal={Journal of Mathematical Imaging and Vision},
year={2020},
month={Jul},
day={01},
volume={62},
number={6},
pages={773-789},
issn={1573-7683},
url={https://doi.org/10.1007/s10851-020-00951-y}
}

@book{beck2017first,
  title={First-order methods in optimization},
  author={Beck, Amir},
  year={2017},
  publisher={SIAM}
}

@inproceedings{ryu2019plug,
  title={{Plug-and-Play Methods Provably Converge with Properly Trained Denoisers}},
  author={Ryu, Ernest and Liu, Jialin and Wang, Sicheng and Chen, Xiaohan and Wang, Zhangyang and Yin, Wotao},
  booktitle={International Conference on Machine Learning},
  pages={5546--5557},
  year={2019},
  organization={PMLR}
}

@article{cohen2021regularization,
  title={{Regularization by Denoising via Fixed-Point Projection (RED-PRO)}},
  author={Cohen, Regev and Elad, Michael and Milanfar, Peyman},
  journal={SIAM Journal on Imaging Sciences},
  volume={14},
  number={3},
  pages={1374--1406},
  year={2021},
  publisher={SIAM}
}

@Article{Zhang2017,
  author  = {Zhang, Kai and Zuo, Wangmeng and Chen, Yunjin and Meng, Deyu and Zhang, Lei},
  journal = {IEEE Transactions on Image Processing},
  title   = {{Beyond a Gaussian Denoiser: Residual Learning of Deep CNN for Image Denoising}},
  year    = {2017},
  number  = {7},
  pages   = {3142-3155},
  volume  = {26},
  doi     = {10.1109/TIP.2017.2662206},
}

@ARTICLE{zhang2021plug,
  author={Zhang, Kai and Li, Yawei and Zuo, Wangmeng and Zhang, Lei and Van Gool, Luc and Timofte, Radu},
  journal={IEEE Transactions on Pattern Analysis and Machine Intelligence}, 
  title={Plug-and-Play Image Restoration With Deep Denoiser Prior}, 
  year={2022},
  volume={44},
  number={10},
  pages={6360-6376},
  doi={10.1109/TPAMI.2021.3088914}}

@InProceedings{Venkatakrishnan2013,
  author    = {Venkatakrishnan, Singanallur V. and Bouman, Charles A. and Wohlberg, Brendt},
  booktitle = {2013 IEEE Global Conference on Signal and Information Processing},
  title     = {{Plug-and-Play priors for model based reconstruction}},
  year      = {2013},
  pages     = {945-948},
  doi       = {10.1109/GlobalSIP.2013.6737048},
}

@InProceedings{hurault2022proximal,
  title = 	 {{Proximal Denoiser for Convergent Plug-and-Play Optimization with Nonconvex Regularization}},
  author =       {Hurault, Samuel and Leclaire, Arthur and Papadakis, Nicolas},
  booktitle = 	 {Proceedings of the 39th International Conference on Machine Learning},
  pages = 	 {9483--9505},
  year = 	 {2022},
  volume = 	 {162},
  series = 	 {Proceedings of Machine Learning Research},
  month = 	 {17--23 Jul},
  publisher =    {PMLR},
  url = 	 {https://proceedings.mlr.press/v162/hurault22a.html}
}

@Article{Bauschke2021,
author={Bauschke, Heinz H.
and Moursi, Walaa M.
and Wang, Xianfu},
title={Generalized monotone operators and their averaged resolvents},
journal={Mathematical Programming},
year={2021},
month={Sep},
day={01},
volume={189},
number={1},
pages={55-74},
issn={1436-4646}
}

@inproceedings{NEURIPS2021_97108695,
 author = {Cohen, Regev and Blau, Yochai and Freedman, Daniel and Rivlin, Ehud},
 booktitle = {Advances in Neural Information Processing Systems},
 pages = {18152--18164},
 publisher = {Curran Associates, Inc.},
 title = {{It Has Potential: Gradient-Driven Denoisers for Convergent Solutions to Inverse Problems}},
 volume = {34},
 year = {2021}
}

@Article{Chen2017,
  author   = {Chen, Yunjin and Pock, Thomas},
  journal  = {IEEE Transactions on Pattern Analysis and Machine Intelligence},
  title    = {{Trainable Nonlinear Reaction Diffusion: A Flexible Framework for Fast and Effective Image Restoration}},
  year     = {2017},
  number   = {6},
  pages    = {1256-1272},
  volume   = {39},
  doi      = {10.1109/TPAMI.2016.2596743},
}

@Article{Xue2022,
author={Xue, Feng},
title={A generalized forward--backward splitting operator: degenerate analysis and applications},
journal={Computational and Applied Mathematics},
year={2023},
volume={42},
number={1},
issn={1807-0302},
url={https://doi.org/10.1007/s40314-022-02143-3}
}

@Article{Buzzard2018,
  author   = {Buzzard, Gregery T. and Chan, Stanley H. and Sreehari, Suhas and Bouman, Charles A.},
  journal  = {SIAM Journal on Imaging Sciences},
  title    = {{Plug-and-Play Unplugged: Optimization-Free Reconstruction Using Consensus Equilibrium}},
  year     = {2018},
  number   = {3},
  pages    = {2001-2020},
  volume   = {11},
  url      = {https://doi.org/10.1137/17M1122451}
}

@InProceedings{Park2022,
  author    = {Park, Jisun and Ryu, Ernest K},
  booktitle = {Proceedings of the 39th International Conference on Machine Learning},
  title     = {{Exact Optimal Accelerated Complexity for Fixed-Point Iterations}},
  year      = {2022},
  month     = {17--23 Jul},
  pages     = {17420--17457},
  publisher = {PMLR},
  series    = {Proceedings of Machine Learning Research},
  volume    = {162}
}

@article{he2024convergence,
  title={{Convergence analysis of the Halpern iteration with adaptive anchoring parameters}},
  author={He, Songnian and Xu, Hong-Kun and Dong, Qiao-Li and Mei, Na},
  journal={Mathematics of Computation},
  volume={93},
  number={345},
  pages={327--345},
  year={2024}
}

@InProceedings{Diakonikolas2020,
  author    = {Diakonikolas, Jelena},
  booktitle = {Proceedings of Thirty Third Conference on Learning Theory},
  title     = {{Halpern Iteration for Near-Optimal and Parameter-Free Monotone Inclusion and Strong Solutions to Variational Inequalities}},
  year      = {2020},
  month     = {09--12 Jul},
  pages     = {1428--1451},
  publisher = {PMLR},
  series    = {Proceedings of Machine Learning Research},
  volume    = {125}
}

@book{engl1996regularization,
  title={Regularization of inverse problems},
  author={Engl, Heinz Werner and Hanke, Martin and Neubauer, Andreas},
  volume={375},
  year={1996},
  publisher={Springer Science \& Business Media}
}

@Article{Tan2024,
  author   = {Tan, Hong Ye and Mukherjee, Subhadip and Tang, Junqi and Sch\"{o}nlieb, Carola-Bibiane},
  journal  = {SIAM Journal on Imaging Sciences},
  title    = {{Provably Convergent Plug-and-Play Quasi-Newton Methods}},
  year     = {2024},
  number   = {2},
  pages    = {785-819},
  volume   = {17},
  url      = {https://doi.org/10.1137/23M157185X},
}

@Article{Qi2021,
  author    = {Huiqiang Qi and Hong-Kun Xu},
  journal   = {Numerical Functional Analysis and Optimization},
  title     = {{Convergence of Halpern's Iteration Method with Applications in Optimization}},
  year      = {2021},
  number    = {15},
  pages     = {1839--1854},
  volume    = {42},
  publisher = {Taylor \& Francis},
}

@article{Sun2024,
author = {Sun, Defeng and Yuan, Yancheng and Zhang, Guojun and Zhao, Xinyuan},
title = {Accelerating Preconditioned ADMM via Degenerate Proximal Point Mappings},
journal = {SIAM Journal on Optimization},
volume = {35},
number = {2},
pages = {1165-1193},
year = {2025},
doi = {10.1137/24M1650053}
}

@ARTICLE{PnP2015,
  author={Sreehari, Suhas and Venkatakrishnan, S. V. and Wohlberg, Brendt and Buzzard, Gregery T. and Drummy, Lawrence F. and Simmons, Jeffrey P. and Bouman, Charles A.},
  journal={IEEE Transactions on Computational Imaging}, 
  title={{Plug-and-Play Priors for Bright Field Electron Tomography and Sparse Interpolation}}, 
  year={2016},
  volume={2},
  number={4},
  pages={408-423},
  doi={10.1109/TCI.2016.2599778}}

@ARTICLE{REDClarifications2019,
  author={Reehorst, Edward T. and Schniter, Philip},
  journal={IEEE Transactions on Computational Imaging}, 
  title={{Regularization by Denoising: Clarifications and New Interpretations}}, 
  year={2019},
  volume={5},
  number={1},
  pages={52-67},
  doi={10.1109/TCI.2018.2880326}}

@misc{bredies2024learningfirmlynonexpansiveoperators,
      title={{Learning Firmly Nonexpansive Operators}}, 
      author={Kristian Bredies and Jonathan Chirinos-Rodriguez and Emanuele Naldi},
      year={2024},
      eprint={2407.14156},
      archivePrefix={arXiv},
      primaryClass={math.OC}
}

@article{MMO2021,
author = {Pesquet, Jean-Christophe and Repetti, Audrey and Terris, Matthieu and Wiaux, Yves},
title = {{Learning Maximally Monotone Operators for Image Recovery}},
journal = {SIAM Journal on Imaging Sciences},
volume = {14},
number = {3},
pages = {1206-1237},
year = {2021},
doi = {10.1137/20M1387961}
}

@InProceedings{Wei2024,
  author    = {Wei, Deliang and Chen, Peng and Li, Fang},
  booktitle = {Proceedings of the 41st International Conference on Machine Learning},
  title     = {{Learning Pseudo-Contractive Denoisers for Inverse Problems}},
  year      = {2024},
  month     = {21--27 Jul},
  pages     = {52500--52524},
  publisher = {PMLR},
  series    = {Proceedings of Machine Learning Research},
  volume    = {235}
}

@INPROCEEDINGS{building_FNE,
  author={Terris, Matthieu and Repetti, Audrey and Pesquet, Jean-Christophe and Wiaux, Yves},
  booktitle={ICASSP 2020 - 2020 IEEE International Conference on Acoustics, Speech and Signal Processing (ICASSP)}, 
  title={{Building Firmly Nonexpansive Convolutional Neural Networks}}, 
  year={2020},
  volume={},
  number={},
  pages={8658-8662},
  doi={10.1109/ICASSP40776.2020.9054731}}

@InProceedings{pmlr-v202-delattre23a,
  title = 	 {Efficient Bound of {L}ipschitz Constant for Convolutional Layers by {G}ram Iteration},
  author =       {Delattre, Blaise and Barth\'{e}lemy, Quentin and Araujo, Alexandre and Allauzen, Alexandre},
  booktitle = 	 {Proceedings of the 40th International Conference on Machine Learning},
  pages = 	 {7513--7532},
  year = 	 {2023},
  volume = 	 {202},
  series = 	 {Proceedings of Machine Learning Research},
  month = 	 {23--29 Jul},
  publisher =    {PMLR}
}

@inproceedings{hurault2022gradient,
title={{Gradient Step Denoiser for convergent Plug-and-Play}},
author={Samuel Hurault and Arthur Leclaire and Nicolas Papadakis},
booktitle={International Conference on Learning Representations},
year={2022}
}

@article{HUNDAL200435,
title = {An alternating projection that does not converge in norm},
journal = {Nonlinear Analysis: Theory, Methods \& Applications},
volume = {57},
number = {1},
pages = {35-61},
year = {2004},
issn = {0362-546X},
doi = {https://doi.org/10.1016/j.na.2003.11.004},
author = {Hein S. Hundal}
}

@article{DR_weak_convergence2020,
author = {B\`{u}i, Minh N. and Combettes, Patrick L.},
title = {{The Douglas--Rachford Algorithm Converges Only Weakly}},
journal = {SIAM Journal on Control and Optimization},
volume = {58},
number = {2},
pages = {1118-1120},
year = {2020},
doi = {10.1137/19M1308451}
}

@article{There-Operator-Splitting2024,
author = {Wu, Zhongming and Huang, Chaoyan and Zeng, Tieyong},
title = {{Extrapolated Plug-and-Play Three-Operator Splitting Methods for Nonconvex Optimization with Applications to Image Restoration}},
journal = {SIAM Journal on Imaging Sciences},
volume = {17},
number = {2},
pages = {1145-1181},
year = {2024},
doi = {10.1137/23M1611166}
}

@inproceedings{wu2024principled,
 author = {Wu, Zihui and Sun, Yu and Chen, Yifan and Zhang, Bingliang and Yue, Yisong and Bouman, Katherine},
 booktitle = {Advances in Neural Information Processing Systems},
 pages = {118389--118427},
 publisher = {Curran Associates, Inc.},
 title = {{Principled Probabilistic Imaging using Diffusion Models as Plug-and-Play Priors}},
 volume = {37},
 year = {2024}
}

@ARTICLE{PnP-CI2023,
  author={Kamilov, Ulugbek S. and Bouman, Charles A. and Buzzard, Gregery T. and Wohlberg, Brendt},
  journal={IEEE Signal Processing Magazine}, 
  title={{Plug-and-Play Methods for Integrating Physical and Learned Models in Computational Imaging: Theory, algorithms, and applications}}, 
  year={2023},
  volume={40},
  number={1},
  pages={85-97},
  doi={10.1109/MSP.2022.3199595}}

@article{Izuchukwu2023,
  author = {Chinedu Izuchukwu and Simeon Reich and Yekini Shehu and Adeolu Taiwo},
  title = {{Strong Convergence of Forward-Reflected-Backward Splitting Methods for Solving Monotone Inclusions with Applications to Image Restoration and Optimal Control}},
  journal = {Journal of Scientific Computing},
  volume = {94},
  number = {3},
  pages = {73},
  year = {2023},
  month = {Feb},
  doi = {10.1007/s10915-023-02132-6},
  issn = {1573-7691}
}

@misc{bc24,
    title={{Extra-Gradient Method with Flexible Anchoring: Strong Convergence and Fast Residual Decay}}, 
    author={Radu Ioan Boţ and Enis Chenchene},
    year={2024},
    archivePrefix={arXiv},
    primaryClass={math.OC},
    url={https://arxiv.org/abs/2410.14369}
}

@misc{chen2024hprlpimplementationhprmethod,
      title={{HPR-LP: An implementation of an HPR method for solving linear programming}}, 
      author={Kaihuang Chen and Defeng Sun and Yancheng Yuan and Guojun Zhang and Xinyuan Zhao},
      year={2024},
      archivePrefix={arXiv},
      primaryClass={math.OC},
      url={https://arxiv.org/abs/2408.12179}, 
}

@misc{lu2024restartedhalpernpdhglinear,
      title={{Restarted Halpern PDHG for Linear Programming}}, 
      author={Haihao Lu and Jinwen Yang},
      year={2024},
      archivePrefix={arXiv},
      primaryClass={math.OC},
      url={https://arxiv.org/abs/2407.16144}, 
}

@ARTICLE{PnP2020forMRI,
  author={Ahmad, Rizwan and Bouman, Charles A. and Buzzard, Gregery T. and Chan, Stanley and Liu, Sizhuo and Reehorst, Edward T. and Schniter, Philip},
  journal={IEEE Signal Processing Magazine}, 
  title={{Plug-and-Play Methods for Magnetic Resonance Imaging: Using Denoisers for Image Recovery}}, 
  year={2020},
  volume={37},
  number={1},
  pages={105-116},
  doi={10.1109/MSP.2019.2949470}
}

@article{TranDinh2024,
  author  = {Tran-Dinh, Quoc},
  title   = {{From Halpern's fixed-point iterations to Nesterov's accelerated interpretations for root-finding problems}},
  journal = {Computational Optimization and Applications},
  volume  = {87},
  number  = {1},
  pages   = {181--218},
  year    = {2024},
  month   = {01},
  doi     = {10.1007/s10589-023-00518-8}
}

@article{fastKM2023,
author = {Bo\c{t}, Radu Ioan and Nguyen, Dang-Khoa},
title = {{Fast Krasnosel’skiĭ–Mann Algorithm with a Convergence Rate of the Fixed Point Iteration of \(\boldsymbol{{ o} \left(\frac{1}{{k}} \right)}\)}},
journal = {SIAM Journal on Numerical Analysis},
volume = {61},
number = {6},
pages = {2813-2843},
year = {2023}
}

@article{BOT2021,
title = {{A strongly convergent Krasnosel’skiǐ–Mann-type algorithm for finding a common fixed point of a countably infinite family of nonexpansive operators in Hilbert spaces}},
journal = {Journal of Computational and Applied Mathematics},
volume = {395},
pages = {113589},
year = {2021},
issn = {0377-0427},
author = {Radu Ioan Boţ and Dennis Meier}
}
\end{document}